\documentclass{statsoc}

\usepackage[T1]{fontenc}
\usepackage[utf8]{inputenc}
\usepackage{graphicx}
\usepackage{bm}
\usepackage{bbm}
\usepackage{fullpage}
\usepackage{setspace}
\usepackage{color}
\usepackage{url}

\usepackage[utf8]{inputenc} 
\usepackage[T1]{fontenc}    
\usepackage{hyperref}       
\usepackage{url}            
\usepackage{booktabs}       

\usepackage{amsfonts,amsmath,amsthm}       

\usepackage{nicefrac}       
\usepackage{microtype}      

\usepackage{bm}
\usepackage{color}
\usepackage{graphicx}
\usepackage{algorithm}
\usepackage{algorithmic}

\newcommand{\bs}{\mathbf{s}}
\newcommand{\bx}{\mathbf{x}}

\newcommand{\bz}{\mathbf{z}}
\newcommand{\bX}{\mathbf{X}}
\newcommand{\bW}{\mathbf{W}}
\newcommand{\bB}{\mathbf{B}}
\newcommand{\bA}{\mathbf{A}}

\newcommand{\bC}{\mathbf{C}}
\newcommand{\bS}{\mathbf{\Sigma}}

\newcommand{\bU}{\mathbf{U}}
\newcommand{\bfeta}{\bm{\eta}}
\newcommand{\bLambda}{\bm{\Lambda}}
\newcommand{\bF}{\mathbf{F}}

\newcommand{\by}{\mathbf{y}}

\newcommand{\bu}{\mathbf{u}}

\newcommand{\bG}{\bm{\Gamma}}
\newcommand{\bzero}{\bm{0}}
\newcommand{\bI}{\mathbf{I}}

\newcommand{\bH}{\mathbf{H}}

\newcommand{\btheta}{\boldsymbol{\theta}}

\newcommand{\dd}{\mathrm{d}}

\newcommand{\Op}{\mathcal{O}}

\newcommand{\Gau}{\mathcal{N}}

\theoremstyle{remark}

\providecommand{\theoremname}{Theorem}
\providecommand{\propositionname}{Proposition}
\providecommand{\lemname}{Lemma}
\newtheorem{thm}{\protect\theoremname}
 \theoremstyle{plain}
 \newtheorem{prop}[thm]{\protect\propositionname}
  \theoremstyle{plain}
  \newtheorem{lem}{\protect\lemname}

\newcommand\rev[1]{{\color{black}#1}}  

\title[Short title]{Auxiliary gradient-based sampling algorithms}

\author[Titsias and Papaspiliopoulos]{Michalis K. Titsias}
\address{Department of Informatics, Athens University of Economics and Business, Athens, Greece}
\author{Omiros Papaspiliopoulos}
\address{ICREA and Department of Economics and Business Universitat Pompeu Fabra, Barcelona 08005}
  
\begin{document}



\begin{abstract}
We introduce a new family of MCMC samplers 
that combine auxiliary variables, Gibbs sampling and  Taylor expansions of the
target density.  
Our approach permits the marginalisation over the auxiliary variables
yielding marginal samplers, or  the augmentation of the auxiliary
variables, yielding auxiliary samplers. 
The well-known Metropolis-adjusted Langevin algorithm
(MALA) and preconditioned Crank-Nicolson Langevin (pCNL) algorithm
are shown to be special cases. 
\rev{We prove that marginal samplers are superior in terms
of asymptotic variance and demonstrate cases where they are slower in computing time compared to
auxiliary samplers.}  
In the context of latent
Gaussian models we propose new auxiliary and marginal
samplers whose implementation requires a single tuning parameter,
which can be found automatically during the transient phase.
Extensive experimentation shows that the increase in 
efficiency  (measured as effective sample size per unit of 
computing time) relative to (optimised implementations of)  pCNL,
elliptical slice sampling and MALA ranges from 10-fold  in binary
classification problems to 25-fold in log-Gaussian Cox processes to
100-fold in Gaussian process regression, \rev{and it is on par with
Riemann manifold Hamiltonian Monte Carlo in an example where the
latter has the same complexity as the aforementioned algorithms.}  We
explain 
this remarkable improvement in terms of the way alternative samplers
try to approximate the eigenvalues of the target. We
introduce a novel MCMC sampling scheme for hyperparameter learning that
builds upon the auxiliary samplers. The MATLAB code for reproducing the
experiments in the article is publicly available and a Supplement to
this article contains additional experiments and implementation details.
\end{abstract}

\section{Introduction}

 Monte Carlo sampling is  the preferable tool for 
uncertainty quantification and statistical learning in a wide range of
scientific problems. We are
 interested in Bayesian learning problems where the
density to be sampled is proporional to $\pi_0(\bx|\btheta) \exp\{f(\bx)\}$, where
$\pi_0(\bx|\btheta)$ is the prior, potentially depending on 
hyperparameters $\btheta$, and $f(\bx)$ is the log-likelihood; the
dependence on any data is suppressed throughout; we highlight the dependence on  $\btheta$ when it is  part of the inference and wish to sample from
$\pi(\bx,\btheta) \propto \pi_0(\bx | \btheta) \exp\{f(\bx)\} p(\btheta)$
where $p(\btheta)$ is a prior on the hyperparameters; when the hyperparameters are fixed we might drop them from the notation and write $\pi_0(\bx)$ instead.  Whereas the main 
algorithmic construction and some theory is generic, the field of
application in this article and all the experiments refer to  latent Gaussian models where $\pi_0(\bx | \btheta) =\Gau(\bx | \bzero, \bC)$.

Major challenges with MCMC in these contexts are: (i)
Likelihood-informed proposals; gradient-based methods based on
discretisations of the Langevin diffusion, e.g. the
preconditioned Metropolis-adjusted Langevin algorithm (pMALA) of \cite{mala} and
the manifold MALA (mMALA) of \cite{Girolami11} are typical
examples. (ii) Prior-informed proposals, i.e., ones that
are in agreement with the dependence structure imposed by the prior; this is particularly important for
priors that impose smoothness and create strong dependencies among the
components of $\bx$, e.g. Gaussian priors;   \cite{neal} first
proposed such a sampler which has evolved to the elliptical slice
sampler of
\cite{murrayEllipt}, and has been extended to the  preconditioned
Crank-Nicolson (pCN) samplers of
\cite{.beskos:2008:bridges+methods+diffusion} who show that failure to be
prior-informed leads to collapse of algorithms in high-dimensional
problems that can otherwise be avoided.  
The
preconditioned Crank-Nicolson Langevin samplers of
\cite{Cotter2013MCMC}, are likelihood and prior informed. (iii) Numerically efficient
computation of the accept-reject mechanism 
within Metropolis-Hastings schemes, which  can be the bottleneck for scaling up MCMC to big data problems,
see e.g. \cite{teh}, \cite{ericLang}, \cite{arnak}.
(iv) The numerically efficient
generation of proposals; see e.g.
\cite{davies}.

\rev{
We introduce a new family of MCMC samplers that are both prior and
likelihood informed and are constructed using a combination of
auxiliary variables and Taylor expansions. The essence of the approach
is to use the proposals used within standard MCMC as auxiliary
variables and sample the augmented target using
Metropolis-within-Gibbs. The resultant algorithms are called auxiliary samplers and the
algorithm that is obtained by integrating the auxiliary variables out is
called a marginal sampler. We prove that the marginal sampler is
better in Peskun ordering than its corresponding auxiliary sampler. We apply the construction in latent Gaussian models and
introduce three new MCMC samplers for such models, two auxiliary and
one marginal; one of the auxiliary samplers
for latent Gaussian models (the scheme based on the auxiliary variable $\bz$ as detailed in
Section 3.3) was first sketched in a discussion in \citep{Titsias2011}. We show that
 pCNL is a special case of the new framework.  We carry out an extensive series of experiments
where we compare our new samplers with pCN,  elliptical slice sampler (Ellipt), pMALA, pCNL. The implementation of all these
algorithms involves an $\Op(n^2)$ cost per iteration, where $n$ is the
dimension of $\bx$, and requires  a single
tuning parameter. For log-Gaussian Cox processes we also compare against
mMALA and the Riemann manifold Hamiltonian Monte Carlo (RMHMC) of
\cite{Girolami11}, since in this context these can also be implemented at
an $\Op(n^2)$ cost, whereas their implementation is typically
$\Op(n^3)$.  We demonstrate that our new samplers are more
  efficient in terms of effective sample size per unit of computing
  time than pMALA, pCN, Ellipt and pCNL by a factor that ranges from
10 to 100 in standard machine learning and spatial statistics problems, and are  comparable
to RMHMC when the latter 
can also be implemented at $\Op(n^2)$ cost. We provide an explanation
on why these gains are achieved in terms of the way the different
algorithms approximate the spectrum of the target by shrinking
differently the eigenvalues of the prior covariance. 

One of the proposed
auxiliary samplers has smaller computational time than the marginal
sampler since it involves a smaller number of matrix-vector
multiplications; actually its acceptance probability can be computed
an order of magnitude faster than that of the marginal
sampler. Nevertheless,  all the experiments reported
here and in a Supplement find the marginal sampler superior to the
auxiliary ones even when computing cost is taken into account. We also
investigate empirically the optimal tuning of the new algorithms and
find the auxiliary ones to require tuning around 50\% acceptance rate and the marginal
around 60\%. 
Utilising one of the auxiliary samplers, we propose a novel sampler for 
$\pi(\bx,\btheta) \propto \Gau(\bx; \bzero, \bC_{\btheta}) \exp\{f(\bx)\} p(\btheta)$
where $\btheta$ are hyperparameters.  We find
impressive efficiency gains relative to alternative 
samplers in the context 
of multi-class Gaussian process classification.
}

The paper is organised as follows. Section \ref{sec:comp} introduces
auxiliary and marginal samplers and proves a comparison result. 
Section \ref{sec:proposedalg} develops the methodology for latent
Gaussian models. Section \ref{sec:hyper}   introduces a sampler for joint inference
of the latent Gaussian process and its hyperparameters. Section 5 contains an extensive series of
numerical experiments.  A Supplement to this article contains
additional theoretical results,  implementation details, pseudocode of the proposed algorithms, and several
simulation experiments. All experiments are completely
reproducible and the code can be found in https://github.com/mtitsias/aGrad.

\section{Auxiliary versus marginal
  samplers}
\label{sec:comp}

We start with a generic result about the use of auxiliary variables
within Metropolis-Hastings schemes. This result establishes that the
marginal samplers we introduce in this article are \emph{statistically}
more efficient than the auxiliary samplers we also introduce;
statistical efficiency here is measured by the asymptotic variance of
ergodic averages computed using the samples generated by the
algorithms, hence the averages computed using marginal samplers are shown
to have lower asymptotic variance than those using auxiliary samplers. However,
we will later establish concrete situations where the auxiliary
samplers are \emph{computationally} more efficient, in particular the
use of auxiliary variables can change the complexity of the computation
of the Metropolis-Hastings ratio. Hence, our broader framework allows
for the exchange of statistical with computational efficiency for
better scalability of the algorithms. 

\subsection{Auxiliary Metropolis-Hastings samplers}

Consider a target density $\pi(\bx)$ and a hierarchical mechanism for
generating proposals $\by$ within a Metropolis-Hastings sampler for
probing $\pi(\bx)$,  by first drawing \rev{$\bu \sim q(\bu|\bx)$} and
then $\by \sim q(\by | \bx,\bu)$. Hence, the overall proposal density
is 
\[
q(\by | \bx) = \int q(\by | \bx,\bu) q(\bu | \bx) \dd \bu\,.
\]
What we will call a marginal scheme proposes in this way and accepts
with probability the minimum of 1 and  the Metropolis-Hastings ratio, 
\[
{\pi(\by) q(\bx | \by) \over \pi(\bx) q(\by | \bx) }\,.
\]
However, we can alternatively augment the state-space with $\bu$ as an
auxiliary variable and design a sampler for $\pi(\bx,\bu) = \pi(\bx)
q(\bu | \bx)$. We can sample from this extended target by
Hastings-within-Gibbs:
\begin{enumerate}
\item Sample $\bu | \bx \sim \pi(\bu | \bx) = q(\bu | \bx)$ 
\item Propose $\by | \bu,\bx \sim q(\by | \bx,\bu)$ and accept the move with probability 
\[
1 \wedge {\pi(\by | \bu) q(\bx | \by,\bu) \over \pi(\bx | \bu) q(\by |
  \bx,\bu)} 
\,.
\] 
\end{enumerate}
Hence, a Metropolis-Hastings step is used to sample the intractable
$\pi(\bx|\bu)$. Let us call this second scheme an auxiliary sampler. 
In this formulation, the marginal and auxiliary samplers use the same
ingredients, $q(\bu|\bx)$ and $q(\by|\bx,\bu)$, and marginally generate
the same proposals $\by$ from a given state $\bx$. However, their
acceptance mechanisms differ.


\subsection{An example: auxiliary MALA}
 We start with a random walk proposal,  $\Gau(\bu | \bx, (\delta/ 2)
\bI)$ to define an augmented target  
\begin{equation*}
\pi(\bx,\bu) \propto  \pi(\bx) q(\bu|\bx) = \pi(\bx) \Gau(\bu | \bx, (\delta/ 2) \bI)
\end{equation*}
and sample from this using an auxiliary sampler, as described above.  In order to sample the intractable
$\pi(\bx|\bu)$ we do a first order approximation to the log target
density, around $\bx$, given by  
$\log \pi(\by) \approx \log \pi(\bx) + \nabla \log \pi(\bx)^T (\by -
\bx)$, and combine it with $q(\bu|\by)$ to obtain the proposal
\begin{align*}
q(\by |\bu, \bx) & \propto \exp\{ \log \pi(\bx) + \nabla \log \pi(\bx)^T (\by - \bx) \}  \Gau(\bu | \by, (\delta/ 2) \bI) \nonumber \\
& \propto  \Gau(\by | \bu  +  (\delta/2) \nabla \log \pi(\bx), (\delta/ 2) \bI)\,.
\label{eq:firstorder}
\end{align*}  
The corresponding marginal sampler is the
Metropolis-adjusted Langevin algorithm (MALA)  since  
$q(\by | \bx) = \Gau(\by | \bx  +  (\delta/2) \nabla \log \pi(\bx),
\delta \bI)$. 

\subsection{Marginal samplers are Peskun better than auxiliary samplers}
\label{sec:peskun}
Recall that a Markov transition kernel $P_2$ dominates in Peskun
ordering another $P_1$ if $P_1(\bx,A) \leq P_2(\bx,A)$ for all measurable
$A$ such that $\bx \notin A$, and for all $\bx$.  An important
implication of this ordering is that if $P_1$ and $P_2$ are ergodic
kernels and the corresponding ergodic averages admit a central limit
theorem for a given function, then the asymptotic variance of those
averages for the function are larger for $P_1$ relative to $P_2$.
This comparison was introduced in  \cite{peskun}, extended  in
\cite{tierney} and then in
\cite{leisen}. 
\begin{prop} The marginal sampler is better in Peskun ordering than the auxiliary
  sampler. 
\end{prop}

This result can be proved by noting that the auxiliary sampler we consider is a special case
of the  framework described in Proposition 2 of \cite{storvik}; the general framework in \cite{storvik} considers samplers where two auxiliary variables are used, but the use of a single auxiliary variable is more appropriate for the specific samplers we introduce in 
Section \ref{sec:proposedalg} for latent Gaussian models. Note also a   recent result in the literature on so-called
  pseudomarginal methods, where auxiliary variables are used to provide
  unbiased estimators of $\pi(\bx)$ when this is intractable or
  expensive to compute. Theorem 7 of \cite{andrieu-pm} establishes that the
  corresponding auxiliary sampler is less statistically efficient (in
  the sense described above) than the marginal sampler. \cite{andrieu-pm} also consider alternative auxiliary variable schemes. 


\section{Gradient-based samplers for latent Gaussian models \label{sec:proposedalg}}

\subsection{Setting and objectives}

We now focus on the latent Gaussian models where the target density is written in the form  
\begin{equation}
\pi( \bx )  \propto
\exp\{f(\bx) \} \Gau( \bx | \bzero, \bC),
\label{eq:latentGauss} 
\end{equation} 
where $\exp\{f(\bx) \}$ is the likelihood and $\Gau( \bx | \bzero, \bC)$ is the Gaussian prior which for simplicity we assume 
to have zero mean. We are interested in general-purpose, black-box
algorithms that would perform
reasonably well in a wide range of applications.  
\rev{With
respect to the dimension of the target, $n$, the algorithms we consider can
be implemented (including the transient phase where step sizes are
tuned) at $\Op(n^2)$ cost, given an initial $\Op(n^3)$ offline
pre-computation, provided that $f(\bx)$ and $\nabla f(\bx)$ are computed at $\Op(n^2)$
cost.
}
 It is practically important that no matrix
decompositions are needed even for the transient phase of the
algorithms where algorithmic parameters, such as step size, are tuned
to achieve theoretically-motivated optimal acceptance rates.  It is
also practically important that algorithms use a small number of
tuning parameters, and all the algorithms we consider here use a
single one. We make no assumptions about special properties of $\bC$, such
as Toeplitz or
banded. Section \ref{sec:exp} contains several model structures that
fit into this category and range from multi-class classification to
point processes.  A Supplement  includes details on
optimal implementation and details on the computational cost
of each algorithm. 

\subsection{\rev{Review of popular MCMC samplers for latent Gaussian models \label{sec:related}}}

\rev{
Here, we summarise the main popular MCMC samplers used in the
context of latent Gaussian models. The most basic gradient-based approach is the
preconditioned MALA (pMALA) that generates proposals according to 
\begin{equation*}
q(\by | \bx)  =  \Gau(\by | \bx + \frac{\delta}{2} \bS\nabla \log \pi(\bx),  \delta \bS ), 
\label{eq:precondMALA}
\end{equation*} 
where $\bS$ is a preconditioning matrix and $\delta$
is a step size parameter that has to be tuned. 
This proposal is obtained as a
first-order Euler discretization of the Langevin stochastic
differential equation (SDE), 
\begin{equation*}
\dd \bX_t = {1 \over 2} \bS \nabla \log \pi(\bX_t) \dd t + \bS^{1/2} \dd \bW_t,
\end{equation*}
where $\bW$ is Brownian motion, which is a stochastic process
reversible with respect to $\pi(\bx)$. 
For good sampling we might
prefer to choose the preconditioning in a state-dependent way to
capture the local covariance structure. This is for example the idea
behind  certain schemes used in \cite{Girolami11} that are referred there as  
simplified manifold MALA (mMALA) proposals. However, such algorithms
would require matrix decompositions at each iteration. For latent
Gaussian models the common state-independent choice is $\bS=\bC$,
while often the 
choice $\bS=\bI$ doesn't work at all since it ignores the correlation structure introduced by
 the Gaussian prior. Hence, in this paper we only consider pMALA with
 $\bS=\bC$, which leads to the proposal
\begin{align*} 
q(\by|\bx) & =  \mathcal{N}\left (\by | (1 - \frac{\delta}{2}) \bx + \frac{\delta}{2} \bC \nabla f(\bx), \delta \bC\right)\,.
\label{eq:precondMALALG}
\end{align*}
The generation of proposals cost is $\Op(n^2)$, including the initial tuning of the algorithm for choosing a
$\delta$ that yields approximately 50\% acceptance probability
following \cite{scaleSS}. The Metropolis-Hastings ratio, which 
involves the Gaussian prior, requires an $\Op(n^2)$ computing cost. 
  
\cite{Cotter2013MCMC} review proposals based on more advanced Crank-Nicolson
discretisations of the Langevin SDE, and  focus on the so-called Crank-Nicolson
Langevin (CNL) proposal
\begin{align*}
  q(\by | \bx) & = \Gau(\by | (2 \bC + \delta \bI)^{-1} (2 \bC -
    \delta \bI) \bx + 2\delta(2\bC + \delta \bI)^{-1} \bC \nabla
      f(\bx), \\ 
&  8 \delta (2\bC + \delta \bI)^{-1} \bC (2\bC + \delta
      \bI)^{-1})
\end{align*}
and the preconditioned  Crank-Nicolson
Langevin (pCNL) proposal 
\begin{equation}
  \label{eq:pCNL}
 q(\by | \bx) = \Gau\left(\by |    {2 
     \over 2 +\delta} \bx + {\delta \over 2 + \delta} \bC \nabla
      f(\bx),   {\delta(\delta + 4)  \over (2 + \delta)^{2}} \bC  \right).
\end{equation}
We have parameterised pCNL so that its step size $\delta$ is
comparable to that in the samplers we introduce in this article, hence
the discrepancy with the notation used in \cite{Cotter2013MCMC}. 
We consider only pCNL in this paper since it has been found to
be more efficient in practice. The Metropolis-Hastings ratio for pCNL
becomes 
\begin{equation}
  \label{eq:pCNLMH}
  \exp\{f(\by) - f(\bx) + k(\bx,\by) - k(\by,\bx) \}
\end{equation}
where 
\[
k(\bx,\by) = {2+\delta \over 4 +\delta} \bx^T  \nabla f(\by) - {\delta \over
  2(\delta+4) } \nabla f(\by)^T \bC \nabla f(\by)\,.
\]
Note that the prior density is cancelled in the ratio, which is
a consequence of a reversibility property this proposal mechanism
enjoys and is discussed in Section \ref{sec:connections}. Notice also that the mean of the pMALA  proposal is essentially
the same as the the mean of pCNL since both can be re-written in the form  
$\beta \bx + (1 - \beta) \bC \nabla f(\bx)$ where $\beta \in
[0,1]$. Note that for this
reparametrization to hold
the step $\delta$ for pMALA must be in the range $[0,1/2]$. 
The generation of proposals and the
computation of the Metropolis-Hastings ratio require an  $\Op(n^2)$
cost.
When the gradient term is dropped we obtain the Crank-Nicolson and the
preconditioned Crank-Nicolson (pCN) algorithms, which are only prior
informed. The more interesting pCN corresponds to the proposal  
\begin{equation}
  \label{eq:pCN}
 q(\by | \bx)  = 
    \Gau\left(\by | {2  \over 2 +\delta} \bx,  {\delta(\delta + 4)  \over (2 + \delta)^{2}} \bC  \right)
\end{equation}
and it was originally proposed by \cite{neal} and more
recently in a function space context by
\cite{.beskos:2008:bridges+methods+diffusion}. 
Section \ref{sec:connections} discusses that  pCN proposal
is reversible with respect to the prior, which leads
to an appealing simplification of the Metropolis-Hastings ratio that
becomes simply $\exp\{f(\by)-f(\bx)\}$. Therefore, pCN involves a
$\Op(n^2)$ cost for generating proposals but only an $\Op(n)$ for
deciding on their acceptance. 

Finally, in this article we consider the  elliptical slice sampler (Ellipt)
proposed by \citep{murrayEllipt}, which is  popular in the machine learning community. This scheme combines 
the pCN proposal with a slice sampling procedure that is constrained
on an ellipse that avoids rejections, hence its proposal mechanism does not have any step sizes that need tuning. 
However, the rejection-free property comes with the cost that multiple log-likelihood evaluations 
are required per iteration. 
}
\subsection{New auxiliary and marginal samplers}

The starting point of an auxiliary sampler is the augmentation with the auxiliary variable $\bu$ drawn 
from $q(\bu | \bx) = \Gau(\bu | \bx, (\delta/2) \bI)$ so that 
\[
\pi( \bx, \bu)  \propto
\exp\{f(\bx) \} \Gau( \bx | \bzero, \bC)   \Gau(\bu | \bx, (\delta/2) \bI ).
\]
The auxiliary sampler requires an approximation to $\pi( \bx | \bu)$ 
to be used as a proposal. Since the product $\Gau( \bx | \bzero, \bC)
\Gau(\bu | \bx, (\delta/2) \bI)$, yields a Gaussian density for
$\bx$, we do a first order Taylor expansion of $f(\bx)$ 
around the current value $\bx$ to come up with a Gaussian proposal
density that is both prior and likelihood informed: 
\begin{align*}
 q(\by | \bx,\bu)  \propto & \exp\{ f(\bx) + \nabla f(\bx)^T
 (\by-\bx) \} \Gau(\by | \bzero, \bC) \Gau(\bu | \by, (\delta/2) \bI)
 \nonumber \\ 
& \propto \Gau \left(\by  |  \frac{2}{\delta} \bA (
    \bu + {\delta \over 2} \nabla f(\bx) ), \bA \right), 
\label{eq:aux}
\end{align*}
\rev{where 
\begin{equation}
\bA =  (\bC^{-1} + {2 \over \delta} \bI)^{-1} = {\delta \over 2}  (
\bC + {\delta \over 2}\bI)^{-1}\bC.
\label{ea:def-A}
\end{equation}
$\bA$ can be defined in any of the alternative ways, which are of course equivalent when $\bC$ is invertible. However, $\bA$ can be defined as the second expression when $\bC$ is not invertible (in which case the first expression is invalid). The second expression is also computationally preferrable when   $\bC$ is close to being singular. 
}
Thus, the auxiliary sampler iterates:
 \begin{enumerate}
\item  $\bu \sim  \Gau\left (\bu | \bx, (\delta / 2) \bI \right ) $ 
\item  Propose $\by  \sim q(\by | \bx,\bu) $ and accept it
    according to the  Metropolis-Hastings 
   ratio  
\begin{align}
&   \exp\left \{f(\by) -
  f(\bx) +j(\bx,\by, \bu) - j(\by,\bx, \bu) \right \}, \\
& j(\bx,\by,\bu) = \left( \bx -  \frac{2}{\delta}\bA  (
   \bu + {\delta \over 4} \nabla f(\by))  \right)^T \nabla f(\by)\,.
 \end{align}

\end{enumerate}
The matrix $\bA$ involved in the
algorithm has the same eigenspace as $\bC$. Hence, given a spectral
decomposition of $\bC$ at the onset,  generating from the proposal 
$q(\by|\bu, \bx)$ has computational cost $\Op(n^2)$. 
Also, the Metropolis-Hastings ratio has cost  $\Op(n^2)$. 

The corresponding marginal scheme is obtained by simply marginalising 
out the auxiliary variable $\bu$ to obtain the marginal proposal 
\begin{align}
q(\by|\bx) & = \int  \Gau (\by |  {2 \over \delta}\bA    (
    \bu + (\delta / 2) \nabla f(\bx) ), \bA )   \Gau\left (\bu | \bx, (\delta / 2) \bI \right )  d \bu \nonumber \\
& =  \Gau \left(\by |  { 2 \over \delta}\bA  \left ( \bx   +  {\delta \over 2} \nabla
f(\bx) \right),   {2 \over \delta }\bA^2  + \bA \right)\,, \nonumber
\end{align}
where we have used that $\bA$ is symmetric. One way to 
generate from the proposal $q(\by|\bx)$ is to be based on the
mixture representation above, which corresponds to the same two-step procedure used in the auxiliary scheme and it has 
computational cost $\Op(n^2)$. The corresponding
Metropolis-Hastings ratio simplifies to
\begin{align*}
& \exp\{f(\by) - f(\bx) + h(\bx,\by) - h(\by,\bx)\}, \label{eq:margina-ar} \\    
& h(\bx,\by)=  \left ( \bx - \frac{2}{\delta}\bA  ( \by +  {\delta \over 4} 
\nabla f(\by) ) \right)^T  \left( {2 \over \delta} \bA + \bI \right)^{-1} \nabla f(\by)\, \nonumber 
\end{align*}
and it is also computed at an $\Op(n^2)$ cost.

A reparameterization of the auxiliary variable $\bu $ yields an alternative
auxiliary sampler in this context, that corresponds to the same
marginal sampler. We define a  new  auxiliary variable $\bz$ 
\begin{equation*}
\bz \equiv  \bu + (\delta / 2) \nabla f(\bx) \sim  \Gau(\bz | \bx  + (\delta /
2) \nabla f(\bx), (\delta/2) \bI ),
\end{equation*}
so that the initial proposal distribution $q(\by | \bx,\bu)$ in the
former auxiliary sampler now 
becomes  $q(\by | \bz) = \Gau\left(\by |  (2/\delta) \bA \bz, \bA \right)$. 
This has the interesting property that the proposed $\by$ becomes conditionally independent from the current 
$\bx$ given the auxiliary variable $\bz$; we 
explicitly exploit this when we design samplers for learning
hyperparameters in Section \ref{sec:hyper}.  
Subsequently, with this formulation, we can work with an alternative expanded target,
\[
\pi( \bx, \bz)  \propto
\exp\{f(\bx) \} \Gau( \bx | \bzero, \bC)   \Gau(\bz | \bx  + (\delta /
2) \nabla f(\bx), (\delta/2) \bI )\,,
\]
and consider an auxiliary sampler that iterates between the steps:

 \begin{enumerate}
\item  $\bz \sim \Gau\left (\bz | \bx +  (\delta / 2) \nabla f(\bx) , (\delta / 2) \bI  \right ) $ 
\item  Propose $\by  \sim q(\by|\bz)$ 
 and   accept it according to the Metropolis-Hastings ratio 
\begin{align*}
&   \exp\left \{f(\by) -
  f(\bx) +g(\bz,\by) - g(\bz,\bx) \right \}, \\ 
& g(\bz,\by) = \left(\bz- \by - (\delta/ 4) \nabla f(\by)\right)^T \nabla f(\by)\,.
 \end{align*}
\end{enumerate}
The simplification in the acceptance ratio follows by noting that 
\begin{equation*}
\Gau\left(\by |  \frac{2}{\delta} \bA
    \bz, \bA \right) \propto \Gau(\by|\bzero,\bC) \Gau(\bz|\by,
  (\delta/2)\bI)\,.
\label{eq:posteriorZ}
\end{equation*}
The first step of the algorithm is likelihood-informed while the second step is 
prior-informed and these two-types of information are linked by the auxiliary variable 
$\bz$. 
The marginal sampler that corresponds to this scheme is
precisely the one obtained earlier.  
\rev{However, the Metropolis-Hastings ratio is computable at $\Op(n)$ cost
when $f(\bx)$ and $\nabla f(\bx)$ are also computable at $\Op(n)$
cost. This 
is in contrast with the marginal sampler
and the auxiliary sampler based on $\bu$.} Thus, the auxiliary sampler
based on $\bz$ allows a tradeoff between computational and statistical
efficiency.  


\subsection{\rev{Connection to Crank-Nicolson schemes and a
  covariance shrinkage perspective on algorithmic performance}}
 \label{sec:connections}
 
We can easily construct \emph{preconditioned} auxiliary and marginal
samplers by taking $q(\bu | \bx) = \Gau(\bu | \bx, (\delta/2) \bS)$. 
The corresponding  marginal 
 proposal is
\begin{align}
\label{eq:marginalGen} 
q(\by|\bx) & =  \Gau \left (\by | { 2 \over \delta} \bB   \bS^{-1}\left( \bx   +  {\delta \over 2} \bS \nabla
f(\bx) \right) ,   {2 \over \delta }\bB \bS^{-1} \bB + \bB \right)\,,
\end{align}
where $\bB = ( \frac{2}{\delta} \bS^{-1} + \bC^{-1})^{-1}$. Now if we set 
$\bS = \bC$ this proposal becomes that of pCNL \eqref{eq:pCNL} in
Section \ref{sec:related}.
Interestingly, it can
be checked that for no
choice of $\bS$  CNL is a special case of the marginal
sampler. 
The experiments in Section \ref{sec:exp} demonstrate that the marginal
sampler leads to a significant increase in efficiency, that can
be as large as 100-fold over pCNL. Here we 
provide some insights on the success of the marginal algorithm \rev{and
some further connections to recent algorithms that have been devised
to improve upon pCN.}  

\rev{We first observe that when $\nabla
f=\bzero$, the marginal sampler proposal density is reversible with
respect to the prior. This follows from the Lemma below, the proof of
which is given in a Supplement to this article. Notice that part (a)
is stated in terms of laws without assuming the existence of Lebesgue
densities, to even cover cases where the prior covariance is
singular; in practice it is also interesting to have proposal
mechanisms that do not require inversion of $\bC$, which even when
invertible might be very ill-conditioned for smooth latent Gaussian processes.
\begin{lem}
\label{lem:reverse}

\begin{enumerate}
\item  Suppose that $\bF$ is symmetric and commutes with $\bC$. Then,
  the transition 
  kernel $\Gau(\dd \by|\bF \bx, (\bI-\bF^2) \bC)$ is 
 reversible with respect to the prior $\Gau(\dd \bx|\bzero,\bC)$. 
\item Suppose that  $\bC$ is invertible with $\bC = \bG^2$ and $\bF$ is
  symmetric.  Then  the transition density $\Gau(\by|\bG \bF
  \bG^{-1} \bx, \bG (\bI-\bF^2) \bG)$ is 
 reversible with respect to $\Gau(\bx|\bzero,\bC)$.
\end{enumerate}
\end{lem}
The choice $\bF=\{2/(2+\delta)\} \bI$ in Lemma 1 (a) yields pCN; the choice $\bF =  \{\bC + (\delta/2) 
  \bI\}^{-1}\bC = (2/\delta) \bA$ yields the marginal sampler proposal
  when $\nabla f = \bzero$, see a Supplement to this article for the
  calculations that establish this correspondence.  Therefore, one
  perspective on the proposed marginal sampler is as a principled
  extension of pCNL that changes the preconditioning matrix from a
  multiple of identity to one informed by $\bC$ while requiring a
  single tuning parameter $\delta$.  This perspective links the
  marginal sampler to some recent work in the inverse problem
  literature that improves upon pCN, such as \cite{kody1,kody2}, which we review below.
 }

The proposal covariances used in the marginal sampler
and pCNL have the same eigenvectors as that of $\bC$, but they have
different eigenvalues. In particular, an eigenvalue $\gamma$ of $\bC$
is mapped to 
\[
p(\gamma) = \left(1 - {4 \over (\delta+2)^2}\right)\gamma \quad
\mbox{and} \quad m(\gamma) = {\delta (\delta+4\gamma) \over (\delta
  + 2\gamma)^2} \gamma
\]
in pCNL and the marginal samplers respectively, where $\delta$ is the
step size used in each algorithm. Whereas pCNL shrinks
multiplicatively all eigenvalues of $\bC$ by the same factor, the
marginal sampler shrinks adaptively with the interesting property that
$m'(0) = 1$, hence small eigenvalues are left practically unchanged,
and $m(\gamma) \to \delta$ as $\gamma\to \infty$. Suppose now that the
latent Gaussian field is observed 
with additive Gaussian error with variance $\sigma^2$, as in the
experiments of Section \ref{sec:inflik}. 
Then, the posterior covariance has the same eigenvectors as $\bC$ but an eigenvalue $\gamma$
of $\bC$ is mapped to $t(\gamma) = \gamma \sigma^2 / (\gamma+\sigma^2)$. Hence, the posterior is shrunk a lot where the prior is
weak, with  $\gamma \gg \sigma^2$ mapped to approximately $\sigma^2$,
but $\gamma \ll \sigma^2$
being left practically unchanged. In fact, by taking
$\delta = \sigma^2$, we get that $m(\gamma)/t(\gamma)$ ranges from 1
to approximately 1.12 for the whole range of possible $\gamma,\sigma^2$
values. Thus, we should expect the
marginal sampler when tuned to achieve a good acceptance probability
(which we empirically have found to be between 50-60\%, see Section
\ref{sec:exp})  to choose a  
$\delta$ close to $\sigma^2$  and by
doing so, to capture the appropriate range of values for all eigenvalues of
the target covariance; this is corroborated by the experiments of
Section \ref{sec:exp}.  On the other hand, pCNL needs to set
$\delta$ as approximately $\sigma^2/\gamma$ in order to match a specific
eigenvalue $\gamma$;  
the experiments of Section
\ref{sec:exp} show that when tuned optimally pCNL chooses a step
size which is $\sigma^2/\gamma_{\max}$, where $\gamma_{\max}$ is the
largest eigenvalue of $\bC$, 
which is of course the one that has experienced the
largest change relative to the prior;  however, this step size is too
small for  the modes of the system for
which the posterior is similar to the prior, hence the posterior
distribution of those is explored 
poorly. 

\rev{An intuition along the above lines for the weakness of pCN
  proposal to capture the spectrum of the posterior is developed in
  Section 3 of \cite{kody1} in the context of Bayesian inverse
  problems, which points out that better algorithms should shrink by
  different amounts different eigenvalues. Such algorithms were
  originally proposed in \cite{kody1}  and then  explored much more
  thoroughly  in \cite{kody2}; these works  explore the 
  framework of Lemma 1 (b) and work with autoregressive
  matrices $\bF$ different from multiples of the identity. The
  construction requires that $\bC$ is invertible and a transformation
  of the latent field to standard Gaussian (or white noise when
  working with infinite-dimensional Gaussians). Then,  $\bF$ is such
  that in directions of the posterior strongly informed by the likelihood  a different step size
is used than in those where the prior dominates. This machinery requires a
preliminary exploration of the 
posterior to estimate a matrix the
eigendecomposition of which, together with some eigenvalue
thresholding, determine the so-called likelihood informed subspace,
i.e., those directions where the likelihood dominates, and its
complement. The resultant algorithms involve a large number of
tuning parameters and although the authors have made a serious effort
to automatize their choice, they are not "plug-and-play" in the way
that all other algorithms we have mentioned so far are.  This is why  we
do not consider them further in this 
article. } 





\section{Hyperparameter learning for latent Gaussian models \label{sec:hyper}}

An important issue in latent Gaussian 
models is concerned with the inference over hyperparameters of the covariance
matrix $\bC$, hence we need to design MCMC samplers that sample also from
the corresponding posterior distribution of those, together with the
latent field.  This problem is challenging 
because of the high dependence between these hyperparameters 
and the state vector $\bx$. To deal with this, several approaches have been presented in the literature  \cite{KnorrHeld02,ncp,murrayCovhyper,FilipponeML13,hensman2015mcmc,sahu}. 
Next, we discuss a new sampling move tailored to 
the auxiliary sampler based on the latent variable $\bz$. This is
basically a proof of concept implementation.
Let $\btheta$ denote the hyperparameters of the covariance,
$\bC_{\btheta}$, and $p(\btheta)$ their prior density. In the rest of this section we index matrices that
depend on the unknown hyperparameters $\btheta$, to emphasise the
dependence and the fact that they would have to be updated whenever
$\btheta$ changes during sampling. 
We consider the expanded target 
\[
\pi( \bx, \bz, \btheta)  \propto
\exp\{f(\bx) \} \Gau( \bx | \bzero, \bC_{\btheta})   \Gau(\bz | \bx  + (\delta /
2) \nabla f(\bx), (\delta/2) \bI ) p(\btheta) \,.
\]
To deal with the dependence between $\btheta$ and $\bx$, we  
construct a joint move $(\bx,\btheta) \rightarrow (\by,\btheta')$, 
conditional on the auxiliary variable $\bz$, with proposal 
\begin{equation*}
q(\by,\btheta'|\bx,\btheta, \bz) = q(\by | \btheta', \bz)  q(\btheta'|\btheta),
\end{equation*}
where 
\begin{equation*}
q(\by | \btheta', \bz) = \Gau (\by |  \frac{2}{\delta} \bA_{\btheta'} \bz, \bA_{\btheta'} )
= \frac{1}{\mathcal{Z}(\bz,\btheta')} \Gau(\by | \bzero, \bC_{\btheta'}) \Gau(\bz | \by, (\delta/2) \bI ). 
\end{equation*}  
Here, $\bA_{\btheta} = (\bC_{\btheta}^{-1} + (2/\delta) \bI)^{-1}$ while 
$\mathcal{Z}(\bz,\btheta') = \mathcal{N}(\bz | \bzero, \bC_{\btheta'}
+ (\delta / 2) \bI)$ is a marginal likelihood term.
The corresponding Metropolis-Hastings acceptance ratio simplifies to 
\begin{equation}
\exp\left \{f(\by) -  f(\bx) +g(\bz,\by) - g(\bz,\bx) \right \}  \times \frac{\mathcal{Z}(\bz,\btheta') 
p(\btheta') q(\btheta|\btheta')}{\mathcal{Z}(\bz,\btheta) p(\btheta)q(\btheta'|\btheta)}\,. \nonumber 
\end{equation} 
This ratio has an interesting product form. The first 
term, corresponding to $\exp\left \{f(\by)-f(\bx) +g(\bz,\by) -
  g(\bz,\bx) \right \}$, comes from the auxiliary sampler and does
not depend on $\btheta$.  In contrast the second term does not depend on $\bx$
and  is a  Metropolis-Hastings ratio for
targeting the posterior on $\btheta$ that would arise if the
observations were $\bz$ and were generated as a noisy observation of
the latent field, having isotropic Gaussian noise with variance
$\delta/2$. Due to this noise, the conditional density proportional to
$\mathcal{N}(\bz | \bzero, \bC_{\btheta} +
(\delta / 2) \bI) p(\btheta)$ will be flatter  than the one that arises
when we do Gibbs sampling on $\pi(\bx,\btheta)$, i.e.,
$\mathcal{N}(\bx | \bzero, \bC_{\btheta}) p(\btheta)$, hence we expect to
get better mixing. Notice also that it is possible to generalize the above move  so that the 
proposal on the hyperparameters becomes dependent on the auxiliary variable $\bz$, i.e.\ 
to take the form $q(\btheta'|\btheta, \bz)$. This, for instance, could allow 
to use a standard MALA proposal to target $\mathcal{Z}(\bz,\btheta) p(\btheta)$. 

The above move shares some similarities with the approach  
in \cite{murrayCovhyper} who also use auxiliary variables to inflate the covariance $\bC_{\btheta}$ with
an isotropic or diagonal covariance. However, the two main differences
with that work is that our approach is likelihood-informed 
through the re-sampling step of the auxiliary variable $\bz \sim \Gau(\bz | \bx  + (\delta / 2) \nabla f(\bx), (\delta/2) \bI )$
which depends on $\nabla f(\bx)$, and is based on joint sampling as
opposed to a reparametrization. 

\rev{The update of $\btheta$  necessitates the decomposition of
$\bC_{\btheta}$ and requires an $\Op(n^3)$ computation. Therefore,
every such update is expensive. To compensate 
for that we can do many
latent process
updates for each $\btheta$
update.  }

\section{Experiments \label{sec:exp}} 


We test the proposed samplers  in several 
latent Gaussian models applied to simulated and real data. 
The purpose of the experiments is  
to investigate the performance of the three proposed samplers detailed in Section \ref{sec:proposedalg}: 
i) the marginal sampler (mGrad), 
ii) the sampler based on auxiliary variable $\bu$ (aGrad-u) 
and iii) the sampler based on $\bz$ (aGrad-z). Recall, that all three
algorithms generate the same proposals but are based on different
rules for their acceptance. 
The three new schemes are compared against the methods
reviewed earlier in Section 3.2:  pMALA, pCN, 
pCNL, Ellipt.  
All algorithms have been implemented in MATLAB in an unified and
highly optimized code outlined in a Supplement and available online. 
In one examples we also use mMALA and RMHMC  based on the code provided by the authors \citep{Girolami11}, since their 
implementation in this setting can also be done at $\Op(n^2)$ cost. 
   In all the experiments the step size of pCN was tuned to 
achieve acceptance rate in the range $20\%$ to $30\%$, while for all the remaining 
algorithms that use gradient information from the likelihood the step size was tuned 
to achieve an acceptance rate in the range $50\%$ to $60\%$. This adaptation phase takes 
place only during burn-in iterations, while at collection of samples stage the step 
size is kept fixed. Notice that Ellipt does not require tuning a step size since it 
accepts all samples. The pilot tuning of $\delta$ can be done in a
computationally efficient way for all algorithms, see a Supplement to
this article for details. See the Supplement also for more figures,
experiments and details.


\subsection{Gaussian process regression and small noise limit\label{sec:inflik}}
 
The aim of this experiment is to compare the methods under different levels of the likelihood informativeness.  
For simplicity, we shall consider a simple Gaussian process regression setting where the 
likelihood is Gaussian so that 
\begin{equation}
f(\bx) = - \frac{n}{2} \log(2 \pi \sigma^2) - \frac{1}{2 \sigma^2} \sum_{i=1}^n (y_i - x_i)^2  
\label{eq:fxreg}
\end{equation}    
and where each $y_i$ is a scalar real-valued observation. We take the
prior covariance 
matrix $\bC$ to be
 defined via the squared exponential kernel function $c(\bs_i,\bs_j) = \sigma_x^2 \exp\{ - \frac{1}{2 \ell^2} ||\bs_i  - \bs_j||^2 \}$ where $\bs_i$ is a input vector associated with the latent function value $x_i$.
Under the above model, the informativeness of the likelihood depends
on the noise variance $\sigma^2$. 
We generated three 
regression datasets having $n=1000$ and noise variances $\sigma^2 =1,
0.1 \ \text{and} \ 0.01$. \rev{Autocorrelation summaries of
MCMC output are plotted in Figure \ref{fig:RegressInformLikel}. Autocorrelation lines are  plotted against 
CPU running time and they have been averaged over ten repeats  (using different random seeds) 
of the experiments, where the CPU time has incorporated the burn-in
iterations that are then removed for estimating the
autocorrelations}.      
\begin{figure}
\centering
\begin{tabular}{ccc}
\includegraphics[width=35mm,height=32mm]
{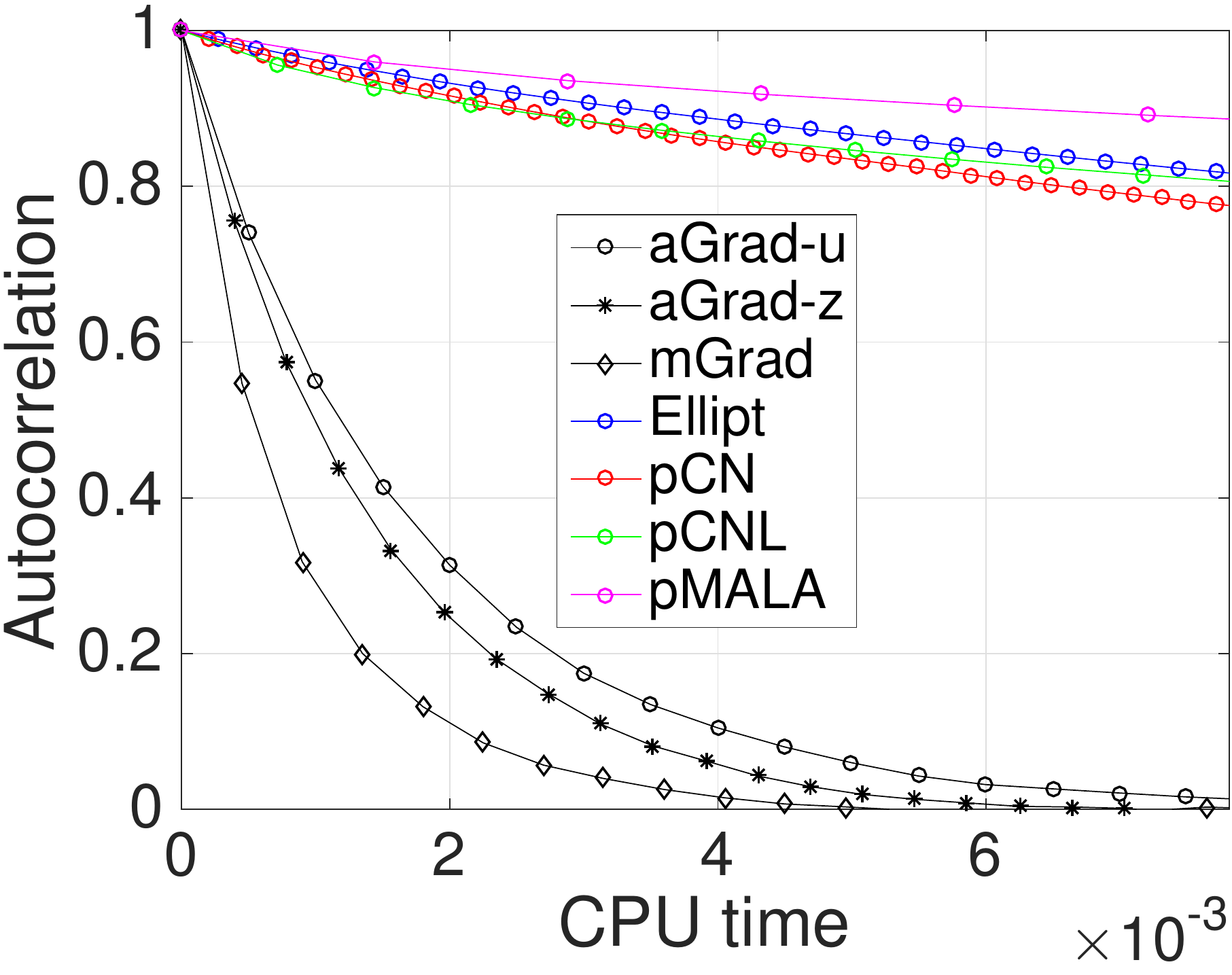} &
\includegraphics[width=35mm,height=32mm]
{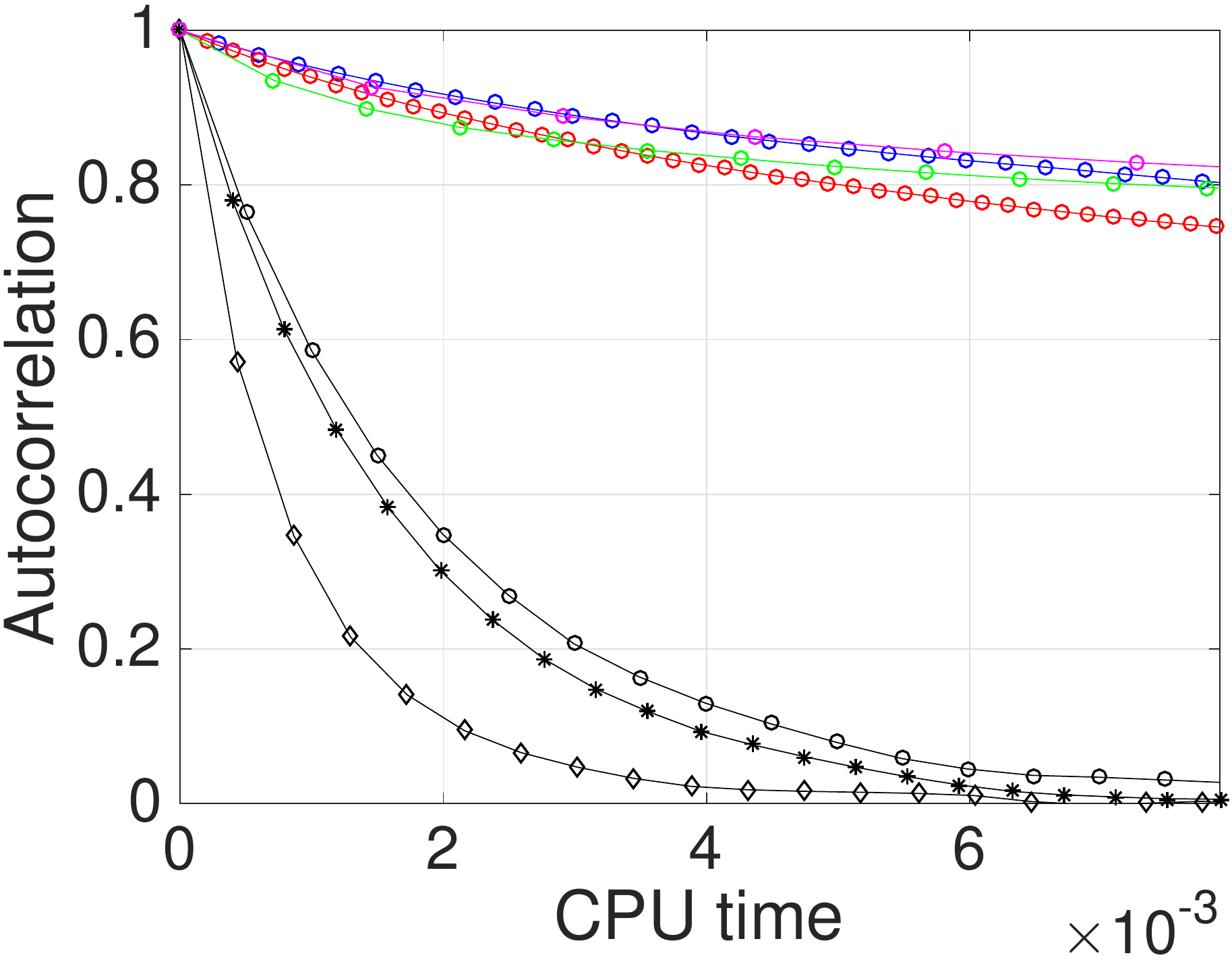} &
\includegraphics[width=35mm,height=32mm]
{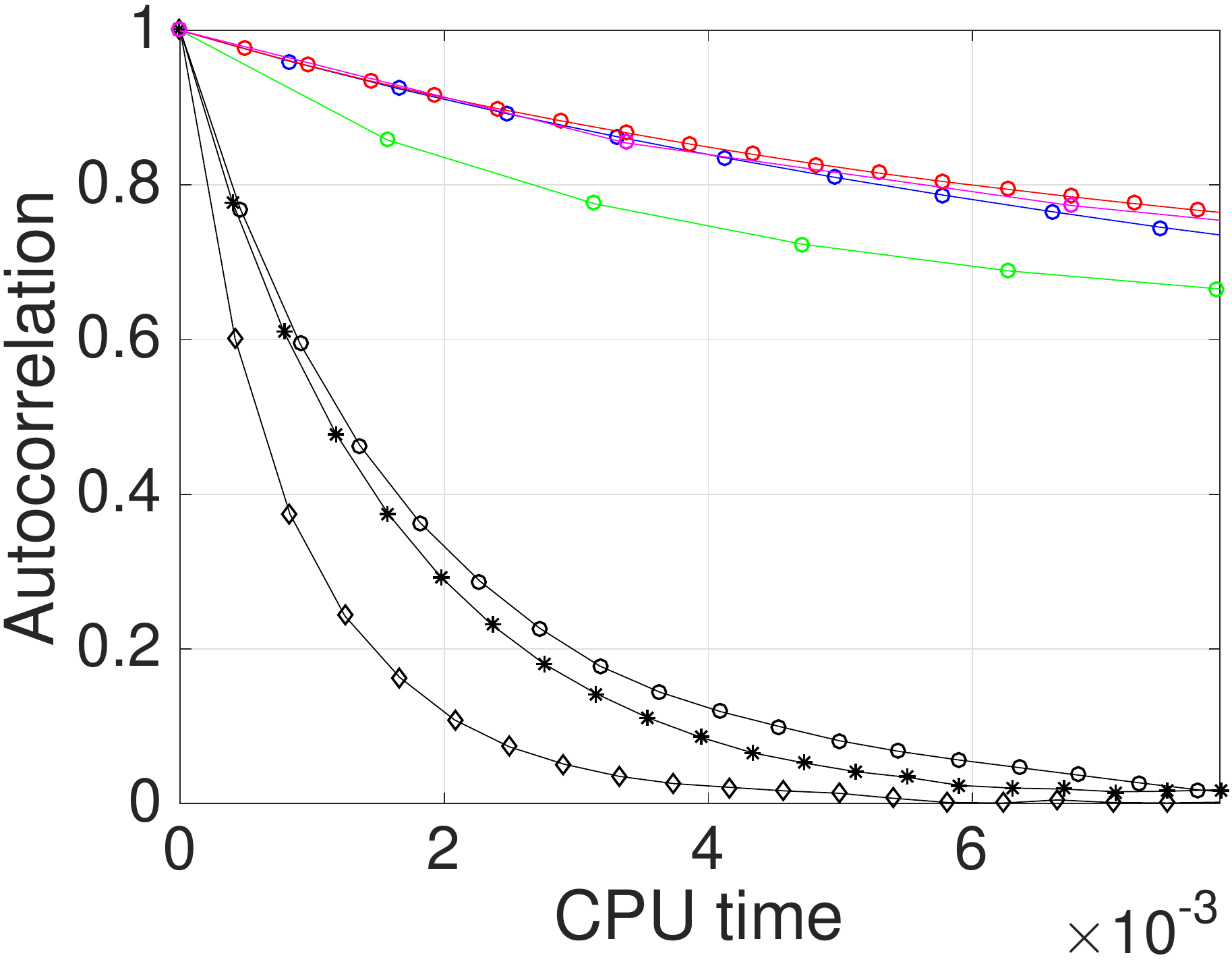} 
\end{tabular}
\caption{\rev{Estimated autocorrelation of the function $f(\bx)$ for
    all algorithms against CPU time; the lines are averages over 
  ten runs of the algorithms.}} 
\label{fig:RegressInformLikel}
\end{figure}
Note that in the limit $\sigma^2 \to 0$ the prior and
the posterior are mutually singular, hence prior-informed only
proposals will collapse. 
For each experiment the noise variance $\sigma^2$ was kept fixed to its  ground-truth value that generated 
the data and the covariance matrix $\bC$ was also fixed and precomputed using suitable values of the kernel 
hyperparameters $(\sigma_x^2, \ell^2)$.  To infer the $1000$-dimensional vector $\bx$ 
we run all alternative sampling schemes for $10^4$ burn-in iterations (starting from a state vector $\bx$ equal to zero) 
and then we collected $5 \times 10^3$ samples. For the third dataset where $\sigma^2=0.01$, 
pCN, pCNL, Ellipt, and pMALA converge more slowly and therefore  we had to increase the number of burn-in 
iterations to $3 \times 10^4$.   Table \ref{table:inflik3} provide 
quantitative results such as effective sample size (ESS), running times and an overall 
efficiency score for the very challenging case where $\sigma^2=0.01$;
the Supplement contains such results for the other two
simulations. The overall efficiency, given by the last column of the
Tables labelled as Min ESS/s, is described as the  
ratio between minimum ESS (among all dimensions of $\bx$) over running time.         
Regarding the performance of the three new samplers we can observe the
following. Firstly,  all samplers converge very fast in all cases and
generate highly effective mixing chains with  
large effective sample sizes (ESS). The highest ESS is always achieved
by mGrad, followed by aGrad-u and aGrad-z.  However, if we also take
into account computational time, the aGrad-z scheme becomes more competitive and 
it can outperform aGrad-u. Secondly, 
both speed of convergence and sampling performance remains stable as the informativeness 
of the likelihood increases. This is achieved through learning the step size $\delta$ 
which adapts to the 
noise of the likelihood - see the argumentation in Section
\ref{sec:connections}. 
Remarkably, in all cases for mGrad $\delta$ tends to be slightly larger than $\sigma^2$, 
while for aGrad-u and aGrad-z $\delta$ is smaller (about a half) than $\sigma^2$.  
Regarding the other methods,  their performance (see Min ESS/s) is at least one order of magnitude 
worse than aGrad-z.   
Notice the step size $\delta$ learned 
pCN, pCNL and pMALA  is not within the same range of the noise
variance $\sigma^2$, but instead it is much smaller - it is actually
about $\sigma^2/ \gamma_{\max}$ where $\gamma_{\max}$ is the largest
eigenvalue of $\bC$, as predicted by the arguments in Section
\ref{sec:connections}. 
 
     
\begin{table}
\caption{\label{table:inflik3}Comparison of sampling methods in the  regression dataset with $\sigma^2 =0.01$. \rev{All numbers are averages across
ten repeats where also one-standard deviation is given for the Min ESS/s score.}}
\centering
\fbox{%
\begin{tabular}{*{5}{c}}
\em Method &\em Time(s) &\em Step $\delta$  &\em ESS (Min, Med, Max)  &\em Min ESS/s (s.d.) \\ 
\hline
aGrad-z  &   5.5  &  0.005  &  (306.7, 430.9, 539.3)  &  55.48 (5.96)\\ 
aGrad-u  &  6.9  &  0.006  &  (345.8, 459.3, 584.8)  &  50.27 (4.02) \\ 
mGrad  &  5.8  &  0.011  &  (856.0, 1095.6, 1301.1)  &  147.67 (9.97)\\ 
pMALA  &  49.3  &  < 0.001  &  (8.4, 45.1, 155.5)  &  0.17 (0.07)\\ 
Ellipt  &  11.0  &   &  (12.9, 50.3, 144.8)  &  1.18 (0.34)\\ 
pCN  &  6.6  &  < 0.001  &  (8.0, 35.4, 107.7)  &  1.21 (0.42)\\ 
pCNL  &  23.9  &  < 0.001  &  (8.7, 54.4, 208.6)  &  0.36 (0.13)\\ 
\end{tabular}}
\end{table}


\subsection{Log-Gaussian Cox Process  \label{sec:logGaussianCox}} 

We consider a log-Gaussian 
Cox model 
and  
apply the sampling schemes to the same model and dataset used in \citep{Girolami11}. 
We also consider a down-sampled version of this dataset in order to investigate how 
the methods behave under mesh refinement. 
More precisely,  we assume that an area of interest $[0,1]^2$ is discretized 
into a $64 \times 64$ regular grid, and each latent variable 
$x_{ij}$ is associated with the grid cell $(i,j),i,j,=1,\ldots,64$. 
The dimension of $\bx$ is $n=4096$. We assume that a data vector of counts is 
generated independently given a latent intensity process $\lambda(\cdot)$
so that each $y_{ij}$ is Poisson distributed with mean $\lambda_{ij}
= m \exp(x_{ij} + v)$, where $m = 1/4096$ is the area of the grid cell
and $v$ is a constant mean value.  
The $4096$-dimensional latent vector $\bx$ is drawn from a
Gaussian process with zero-mean and covariance function 
$$k((i,j),(i',j')) = \sigma_x^2 \exp( -\sqrt{(i-i')^2 + (j -
j')^2}/64 \beta)\,.$$  The likelihood is Poisson form, hence
$$
f(\bx) = \sum_{i,j}^{64} \left( y_{ij} (x_{ij} + v) - m
\exp(x_{ij} + v) \right)\,.
$$ 
We used the dataset from \citep{Girolami11} that was simulated by the above 
model by setting the hyperparameters to $\beta = 1/33$, $\sigma_x^2 = 1.91$ and 
$v = \log(126) - \sigma_x^2/2$. The objective was to infer the latent vector
$\bx$ while keeping the hyperparameters $(\beta,\sigma^2,v)$ fixed to their ground truth values.  
Following the experimental protocol in \citep{Girolami11} we run the algorithms for
$2000$ burn-in iterations and then we collected $5000$ posterior samples. 
Furthermore, for this experiment we also included in the comparison 
the two main algorithms in \citep{Girolami11} namely the manifold MALA (mMALA)
and the Riemann Manifold Hamiltonian Monte Carlo (RMHMC) by using the software  
provided by the authors. 
Table \ref{table:logGaussianCox} provides quantitative scores (see the
Supplement to this article for some trace plots).
We can observe that the new samplers  have significantly better performance 
than all competitors but RMHMC. The latter has slightly less Min ESS/s than mGrad.   
However, notice that for this log-Gaussian Cox model, and due to the analytic properties of the log-likelihood, RMHMC and mMALA 
use a constant metric tensor (i.e.\ the preconditioning matrix $\bS$) that does not depend on $\bx$ \citep{Girolami11}. 
This is an ideal scenario for RMHMC and mMALA which allows the computational cost
to be $O(n^2)$, i.e.\ the same cost as for all other methods in Table \ref{table:logGaussianCox}. 
For more complex models
the metric tensor will depend on $\bx$, which makes RMHMC and mMALA extremely slow    
due to the $O(n^3)$ computational cost per sampling iteration. 
As in the previous experiment  in terms of ESS mGrad is best followed
by aGrad-u and then aGrad-z, although in terms of ESS per unit of time
the fastest aGrad-z scheme gets closer to aGrad-u. Also, notice again the
striking pattern of the learned step sizes for the new samplers 
(third column in Table \ref{table:logGaussianCox}) that positively correlate with ESS.      
In the Supplement we show an experiment with down-sampled sparser
$32\times 32$ grid; results show that $\delta$ in the proposed
gradient samplers is found to be 4 times smaller than on the finer grid,
reflecting the larger information of the data per latent variable in
the sparser dataset. The sampling efficiency of the proposed
algorithms remains unchanged under mesh coarseness, only computing
times decrease.


\begin{table}
\caption{\label{table:logGaussianCox}Comparison of sampling methods in the log-Gaussian Cox model dataset in 
the original version where $n=4096$. \rev{All numbers are averages across
ten repeats where also one-standard deviation is given for the Min ESS/s score.}}
\centering
\fbox{%
\begin{tabular}{*{5}{c}}
\em Method &\em Time(s) &\em Step $\delta$  &\em ESS (Min, Med, Max)  &\em Min ESS/s (s.d.) \\ 
\hline
aGrad-z  &   89.6  &  0.962  &  (36.1, 181.2, 507.5)  &  0.40 (0.10)\\ 
aGrad-u  &  134.6  &  2.814  &  (95.8, 469.3, 1092.4)  &  0.71 (0.14) \\ 
mGrad  &  132.0  &  5.887  &  (177.8, 801.1, 1628.6)  &  1.35 (0.15)\\ 
pMALA  &  218.7  &  0.006  &  (3.5, 12.9, 59.3)  &  0.02 (0.00)\\ 
Ellipt  &  51.5  &   &  (4.2, 17.0, 66.0)  &  0.08 (0.01)\\ 
pCN  &  47.4  &  0.012  &  (3.3, 12.0, 57.6)  &  0.07 (0.00)\\ 
pCNL  &  87.7  &  0.006  &  (3.4, 12.6, 60.2)  &  0.04 (0.00)\\ 
mMALA  &   334.1  &  0.070  &  (21.4, 84.7, 179.4)  &  0.06 (0.02) \\ 
RHMC  &  1493.7  &  0.100  &  (1825.7, 4452.2, 5000.0)  &  1.22 (0.09) \\ 
\end{tabular}}
\end{table}

\subsection{Binary Gaussian process classification \label{sec:binaryclass}}

In this section we consider binary  Gaussian process classification
where given a set of training examples $\{y_i, \bs_i\}_{i=1}^n$ 
we assume a logistic regression likelihood such that
\begin{equation}
f(\bx) = \sum_{i=1}^n y_i \log \sigma(x_i) + (1 - y_i) \log (1 - \sigma(x_i)),
\label{eq:fxbinclass}
\end{equation} 
where $\sigma(x) = 1/(1 + \exp(-x))$ is the logistic function. Here, the vector $\bx$ has been drawn from a 
Gaussian process prior $\bx \sim \mathcal{N}(\bx|\bzero, \bC)$ with the covariance matrix defined by the squared exponential covariance function $c(\bs_i,\bs_j) = \sigma_x^2 \exp\{ - \frac{1}{2 \ell^2} ||\bs_i - \bs_j||^2 \}$.

We considered five standard binary classification datasets with a
number of examples ranging from $n=250$ to $n=1000$. In this section
we show results for the ``Heart'' dataset for which the latent
dimension is $n=270$  and the input dimensionality is $D=13$; see
Table \ref{table:heart}.  In the
Supplement we show results for the other four (and some trace plots
for the ``Heart'' dataset), which qualitatively
communicate the same information as those included here.  We shall follow our experimental protocol used so far 
to compare the sampling methods purely on their ability to sample $\bx$. Thus, we set the covariance 
function hyperparameters $(\sigma_x^2, \ell^2)$ to fixed values chosen to be a 
representative posterior sample obtained after a long run of aGrad-z that carried out full Bayesian inference 
(see Section \ref{sec:hyper}). All sampling methods run for $5000$ burn-in iterations 
and subsequently $5000$ posterior samples were collected. 

The observations we made in all previous experiments hold for the current results as well. For instance, 
the new samplers remain much superior than the competitors. 
However, notice that the improvement over the competitors can be
smaller  in binary classification compared to the previous model classes. This is because the  logistic regression likelihood function can be 
much less informative about the latent $\bx$ compared to standard or Poisson regression.  

\begin{table}
\caption{\label{table:heart}Comparison of sampling methods in Heart dataset. The size of the latent vector
$\bx$ is $n = 270$ and the input dimensionality is $D = 13$. \rev{All numbers are averages across
ten repeats where also one-standard deviation is given for the  Min ESS/s score.}}
\centering
\fbox{%
\begin{tabular}{*{5}{c}}
\em Method &\em Time(s) &\em Step $\delta$  &\em ESS (Min, Med, Max)  &\em Min ESS/s (s.d.) \\ 
\hline
aGrad-z  &   1.7  &  1.440  &  (60.0, 207.2, 419.9)  &  35.29 (11.28)\\ 
aGrad-u  &  1.9  &  2.378  &  (89.1, 327.2, 641.9)  &  48.20 (16.82) \\ 
mGrad  &  1.6  &  5.043  &  (188.9, 593.2, 1096.6)  &  115.39 (8.98)\\ 
pMALA  &  2.1  &  0.049  &  (21.6, 68.7, 151.3)  &  10.36 (1.74)\\ 
Ellipt  &  2.5  &   &  (12.9, 43.9, 101.5)  &  5.17 (0.89)\\ 
pCN  &  1.0  &  0.041  &  (6.9, 28.6, 83.6)  &  6.96 (1.18)\\ 
pCNL  &  1.5  &  0.048  &  (22.2, 69.0, 149.1)  &  14.82 (2.57)\\ 
\end{tabular}}
\end{table}

\subsection{Sampling hyperparameters for binary and multi-class Gaussian process classification
\label{sec:multiclass}}

Here, we wish to infer both hyperparameters of the covariance matrix $\bC$ 
and the state vector $\bx$. We show results for a  very challenging problem  
in multi-class classification involving the popular handwritten recognition task of 
MNIST digits. A Supplement to this article includes an experiment with
binary classification. For all the examples we 
use a squared exponential kernel function 
$c(\bs_i,\bs_j) = \sigma_x^2 \exp\{ - \frac{1}{2 \ell^2} ||\bs_i - \bs_j||^2 \}$ and 
we represent the hyperparameters in the log space so that $\btheta = (\log \sigma_x,  
\log \ell)$. A flat Gaussian prior is assigned to $\btheta$.  
We run only a subset of methods compared in the previous sections. 
Specifically, as a representative of the proposed methods we will consider 
aGrad-z for which we have developed the novel move for jointly sampling 
the latent field $\bx$ and the hyperparameters $\btheta$; see in Section \ref{sec:hyper}. 
We shall consider two versions of aGrad-z: the first based on the standard 
Gibbs approach (aGrad-z-gibbs) that alternates between sampling the latent field and the 
hyperparameters and the second one based on the novel joint sampling move (aGrad-z-joint). 
We also consider pCNL, which is typically the best algorithm among the related pCN and pMALA schemes,
and Ellipt which is widely used in the machine learning community. 
Both pCNL and Ellipt are used to sample the latent field $\bx$ within 
a Gibbs procedure and to clarify this next we denote them as pCNL-gibbs and Ellipt-gibbs. 
For all methods
the proposal for the hyperparameters was chosen to be a Gaussian 
\[
q(\btheta'|\btheta) =  \Gau(\btheta'|\btheta , \kappa \bI),
\] 
where the step size $\kappa$ was tuned to 
achieve acceptance rate in the range $20\%$ to $30\%$. In is important to 
notice that the only difference between aGrad-z-gibbs,  pCNL-gibbs  and Ellipt-gibbs
is how the latent field is sampled, while the step for sampling the hyperparameters
is common to all of them. Also notice that a single 
iteration for aGrad-z-joint consists of two steps: the first for sampling $\bx$ alone 
and the second for joint sampling $(\bx,\btheta)$. These two steps are necessary 
at least for the burn-in phase since the acceptance histories of the first step allow to tune 
$\delta$ while the ones from the second step allow to tune $\kappa$.

In multi-class classification we assume observed pairs $\{y_i,\bs_i\}_{i=1}^n$ 
where $y_i \in \{1,\ldots,K\}$ denotes the class label and $K$ is the number of classes. 
For each $k$-th class we assume a separate latent field drawn independently from a Gaussian 
process so that  $\bx_k \sim \mathcal{N}(\bx_k|\bzero, \bC_{\btheta_k})$ 
where $\btheta_k$ denotes the kernel hyperparameters. The probability of 
an observed $y_i$, given its input vector $\bs_i$, is modelled 
through the multinomial logistic function (most commonly known as softmax), i.e.\
$p(y_i|\bx_i) = \exp\{x_{y_i i}\}/\sum_{k=1}^K \exp\{x_{k i}\}$, which results in a 
log-likelihood function having the form 
\begin{equation}
f(\bx_1,\ldots,\bx_K) = \sum_{i=1}^n \left( x_{y_i i} - \log \sum_{k=1}^K \exp\{x_{k i}\} \right).
\label{eq:multiclasslik}
\end{equation} 
Notice that each $i$-th log-likelihood term couples together $K$ latent values associated with the 
different classes. This adds strong correlations across the a priori independent latent fields, 
which makes posterior inference very challenging. The inference problem involves
learning all latent vectors $(\bx_1,\ldots,\bx_K)$ and the kernel hyperparameters 
$(\btheta_1,\ldots,\btheta_K)$. 

We consider a subset of $n=1000$ instances from the MNIST dataset
which contains $K=10$ classes associated with hand-written digits represented as 
$D=784$ inputs vectors. Thus, each latent vector $\bx_k$ is $1000$-dimensional 
and the overall latent field consists of $K \times n =  10^4$ latent values. 
We also need to infer $20$ kernel hyperparameters, i.e.\ a lengthscale $\ell_{k}^2$ 
and a variance $\sigma_{xk}^2$ for each class. All sampling algorithms were applied 
so that the full latent field is sampled in a single step, i.e.\ all latent  
$(\bx_1,\ldots,\bx_K)$ are concatenated into a single vector that 
follows a $10^4$-dimensional and block-diagonal Gaussian prior. 
We run all algorithms for $10^5$ iterations by initializing them to exactly the same configuration, i.e.\ 
to a zero vector for the full latent field and sensible kernel  hyperparameters based on the inputs vectors. 
All runs were very time-consuming as, for instance, the fastest Ellipt-gibbs completes the $10^5$ iterations 
in roughly $20$ hours while  aGrad-z-joint requires around $40$ hours. Figure \ref{fig:mnist1} 
plots the log-likelihoods across all sampling iterations. Clearly, 
aGrad-z-joint and aGrad-z-gibbs exhibit fast convergence with 
aGrad-z-joint being slightly faster. In contrast, Ellipt-gibbs 
and pCNL-gibbs face severe difficulties to converge to the posterior 
mode found by aGrad-z schemes. For instance, the latter algorithms progress 
very slowly and they are unable to rapidly increase the log-likelihood. 
Recall that the only difference between aGrad-z-gibbs,  pCNL-gibbs  and Ellipt-gibbs
is how the latent field is sampled, therefore the slow convergence of Ellipt-gibbs and 
pCNL-gibbs should be due to their ineffective mechanism for sampling the latent field.

\begin{figure}
\centering
\begin{tabular}{c}
\includegraphics[scale=0.65]
{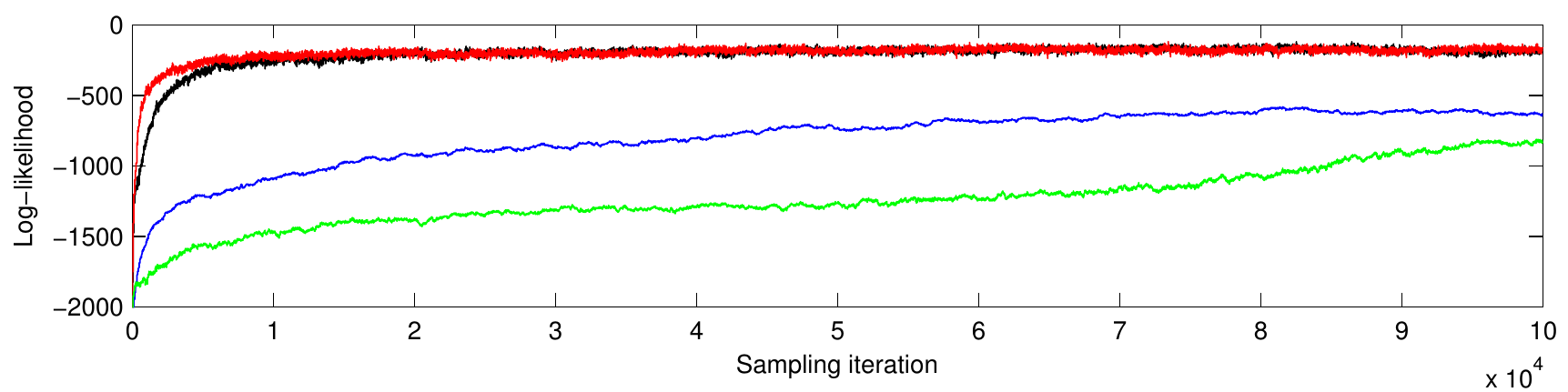} \\
\end{tabular}
\caption{\label{fig:mnist1}The evolution of the log-likelihood values across all iterations in MNIST 
dataset for the sampling schemes aGrad-z-gibbs (black line), aGrad-z-joint (red line),  
Ellipt-gibbs (blue line) and pCNL-gibbs (green line).}
\end{figure}

\section{Discussion}
\label{sec:disc}

We have shown how to turn simple random walk proposals into
gradient-based samplers by the use of auxiliary variables and Taylor
approximations. From a methodological side, we provide a framework
for producing new algorithms that yields as special cases the standard
gradient-based samplers MALA and pCNL. However, the algorithms that we
promote in the article and empirically are shown to have a far
superior performance are different instances of the generic
framework. Our work generates several topics that call for further
research. Here we highlight the following four. One is the scaling of the proposed algorithms. A great asset
of the new samplers is that they only require the tuning of a single
scalar step size, denoted as $\delta$ in the article;  we have
empirically found that setting it to achieve about 50\% to 60\% acceptance rate leads to
good mixing across a wide range of examples \rev{(see Supplement for a plot that empirically  
illustrates that this range of acceptance rates  maximizes ESS)}, and have shown (in the
Supplement) how to
adaptively do the tuning without affecting the complexity of the
algorithm.   Actually, the arguments in Section \ref{sec:connections}
suggest taking $\delta$ as $1/\lambda$, where $\lambda$ is a measure
of precision in the likelihood. It would be really interesting to
formalise these arguments  and even in  the context of Gaussian
targets obtain an optimal scaling theory in terms of  squared jumping
distance. The existing results of \cite{sherlock} are relevant but not
directly applicable. A second is the theoretical investigation
of the relative efficiency among the auxiliary and the marginal
samplers as $n$ or other parameters of the target vary. We have
rigorously shown that the asymptotic variance of ergodic averages
computed with auxiliary samplers is larger than that for the marginal
sampler. We have empirically observed that the efficiency of the
aGrad-u is better than that of aGrad-z, hence conforming to the
intuition that larger computational cost is associated with better
Monte Carlo efficiency. We have also observed empirically that the relative
performance among the samplers does not change drastically as $n$ or
other parameters vary. It is interesting to obtain solid theoretical
results on these comparisons. Thirdly, in this article the dimension
of the auxiliary variables, $\bu$ or $\bz$, matched that of the state
vector $\bx$. This however does not need to be so. Both the theory on
the comparison between marginal and auxiliary samplers, but also the
generic methodology as described in Section \ref{sec:comp}, work also
when dimensions of the auxiliary variable and $\bx$ are
different. We anticipate that even more challenging learning problems
than those considered in this paper can benefit from this
flexibility. Finally, the well developed Markov chain theory for the
convergence of the two-component Gibbs sampler can be used to try and
characterise the properties of the idealised Gibbs sampling algorithm that
can sample exactly $\pi(\bx|\bu)$. Questions regarding the choice of
step size $\delta$, the preconditioner $\bS$ or the augmentation of
$\bu$ or $\bz$, could be answered by
obtaining characterisations of the convergence rate of the idealised
algorithm. 

\section*{Acknowledgements}

We would like to thank both referees for helpful feedback. 




   
\bibliographystyle{plain}
\bibliography{refs}

\section*{Appendix/Supplementary material}

\section*{\rev{Reversibility properties of the marginal sampler}}

\rev{
We start with proving Lemma 1 in the main paper.

\begin{lem}
\label{lem:reverse}

\begin{enumerate}
\item  Suppose that $\bF$ is symmetric and commutes with $\bC$. Then,
  the transition 
  density $\Gau(\by|\bF \bx, (\bI-\bF^2) \bC)$ is 
 reversible with respect to $\Gau(\bx|\bzero,\bC)$. 
\item Suppose that  $\bC$ is invertible with $\bC = \bG^2$ and $\bF$ is
  symmetric.  Then  the transition density $\Gau(\by|\bG \bF
  \bG^{-1} \bx, \bG (\bI-\bF^2) \bG)$ is 
 reversible with respect to $\Gau(\bx|\bzero,\bC)$.
\end{enumerate}
\end{lem}

\begin{proof}
(a): We need to show that the joint law defined by the prior  $\Gau(\dd
\bx|\bzero,\bC)$ and the kernel  $\Gau(\dd \by|\bF \bx, (\bI-\bF^2)
\bC)$, say $Q(\dd \bx, \dd \by)$ is symmetric, in the sense that
$Q(\dd \bx, \dd \by) = Q(\dd \by, \dd \bx)$. With the kernel defined,
the joint law defined is jointly Gaussian with mean $\bzero$ and covariance
\[
\left ( \begin{array}[c]{cc} \bC & \bC \bF^T \\ \bF \bC &
    \bC\end{array} \right) = \left ( \begin{array}[c]{cc} \bC & \bC \bF \\ \bC \bF &
    \bC\end{array} \right)
\]
where the equality is due to the assumed symmetry and comutability of
$\bF$. This Gaussian law is symmetric. \newline (b) Result (a) implies
that $\Gau(\dd \bz ; \bF \bu, \bI-\bF^2)$ is reversible with respect
to $\Gau(\dd \bu; \bzero, \bI)$ when $\bF$ is symmetric; let
$\tilde{\pi}$ and $\tilde{q}$ denote the target and transition density
of these standardized variables. Consider
the transformation $\bx = \bG \bu, \by = \bG \bz$, which is invertible
by assumption. Then, 
\begin{align*}
\pi(\bx) q (\by | \bx) & = |\bG|^{-2} \pi(\bG^{-1} \bu) q(\bG^{-1} \bz |
\bG^{-1}\bu) =  \tilde{\pi}(\bu) \tilde{q}(\bz | \bu) =
\tilde{\pi}(\bz) \tilde{q}(\bu | \bz) \\ &  = \pi(\by) q (\bx | \by),
\end{align*}
where first and last equalities follow from change of variables and
the middle one by the established reversibility of the standardized process.
\end{proof}

We now show that the marginal sampler proposal is reversible with
respect to the prior when $\nabla f = \bzero$. Recall that the
marginal sampler proposal in that case becomes 
\[
q(\by | \bx) = \Gau\left (\by | ({\delta \over 2} \bI + \bC)^{-1} \bC
  \bx, {\delta \over 2}  ({\delta \over 2} \bI + \bC)^{-1} (\bI +
  ({\delta \over 2} \bI + \bC)^{-1} \bC) \bC \right )\,, 
\]
by using the alternative expression for $\bA$ in Equation (4) of the
main article, and after collecting terms. We then work as follows,
using Schur's complement: 
\begin{align*}
\bI - \left({\delta \over 2} \bI + \bC \right)^{-1} \bC & = \bI -
\bC^{1/2}\left({\delta \over 2} \bI + \bC \right)^{-1}  \bC^{1/2} \\
& = \left( \bI + \bC^{1/2} {2 \over \delta} \bC^{1/2} \right)^{-1} =
{\delta \over 2} \left({\delta \over 2} \bI + \bC \right)^{-1}
\end{align*}
therefore the variance of the transition density can be written as 
\[
\left\{ \bI -  \left({\delta \over 2} \bI + \bC \right)^{-2}
  \bC^2\right\} \bC = (\bI - \bF^2) \bC, 
\]
for $\bF$ the autoregressive matrix, $\bF = \left({\delta \over 2} \bI + \bC \right)^{-1}
  \bC$.
}

\section*{Second order schemes
\label{sec:secondorder}}

Our framework can be easily extended to a second-order Taylor
expansion of  $f(\bx)$. For instance the auxiliary sampler based on
$\bu$ and the corresponding marginal sampler are obtained as follows: 
\begin{align*}
q(\by |\bu, \bx)
& \propto  \Gau\left (\by |\frac{2}{\delta} \bA_{\bx} \left( \bu -
   \frac{\delta}{2} (\nabla f(\bx) -  
    \bH_{\bx} \bx ) \right), \bA_{\bx} \right)
\label{eq:secondorder}
\end{align*}    
\begin{equation*}
q(\by | \bx)  =  \Gau\left(\by | \frac{2}{\delta}  \bA_{\bx} \left( \bx -
   \frac{\delta}{2} (\nabla f(\bx) -  
    \bH_{\bx} \bx ) \right),  
 \frac{2}{\delta} \bA_{\bx}^2 + \bA_{\bx} \right)\,,
\label{eq:secondorderMarg}
\end{equation*}  
where $\bH_{\bx} = \nabla \nabla f(\bx)$, and $\bA_{\bx} = ((2/\delta)
\bI  + \bC^{-1} - \bH_{\bx})^{-1}$. It is straightforward 
to show that this defines a Gaussian
autoregression that is reversible with respect to $\pi(\bx)$ when $f$
is a quadratic function, since the second-order expansion is now exact.  We do
not pursue further the construction of second-order schemes 
since they are
computationally prohibitive in high dimensions due to the need for
matrix decompositions at each iteration.  For such cases, approximations to the Hessian 
would be needed as those discussed in \cite{ZhangS11}.

\section*{Experiments}

In all the experiments the ESS has been estimated as proposed in \cite{geyer} and implemented in the code provided by the authors of \cite{Girolami11}. 

\subsection*{Additional experiments with Gaussian process regression}

The three regression datasets along with graphical summaries of MCMC
output are plotted in Figure \ref{fig:RegressInformLikel}. Also 
in the main article we provide ESS summaries for the experiment with
$\sigma^2=0.01$. Here we provide for simulations with $\sigma^2=1$ and
$\sigma^2=0.1$, see Table \ref{table:inflik1} and Table
\ref{table:inflik2} respectively.

\begin{figure}
\centering
\begin{tabular}{ccc}
\includegraphics[width=35mm,height=32mm]
{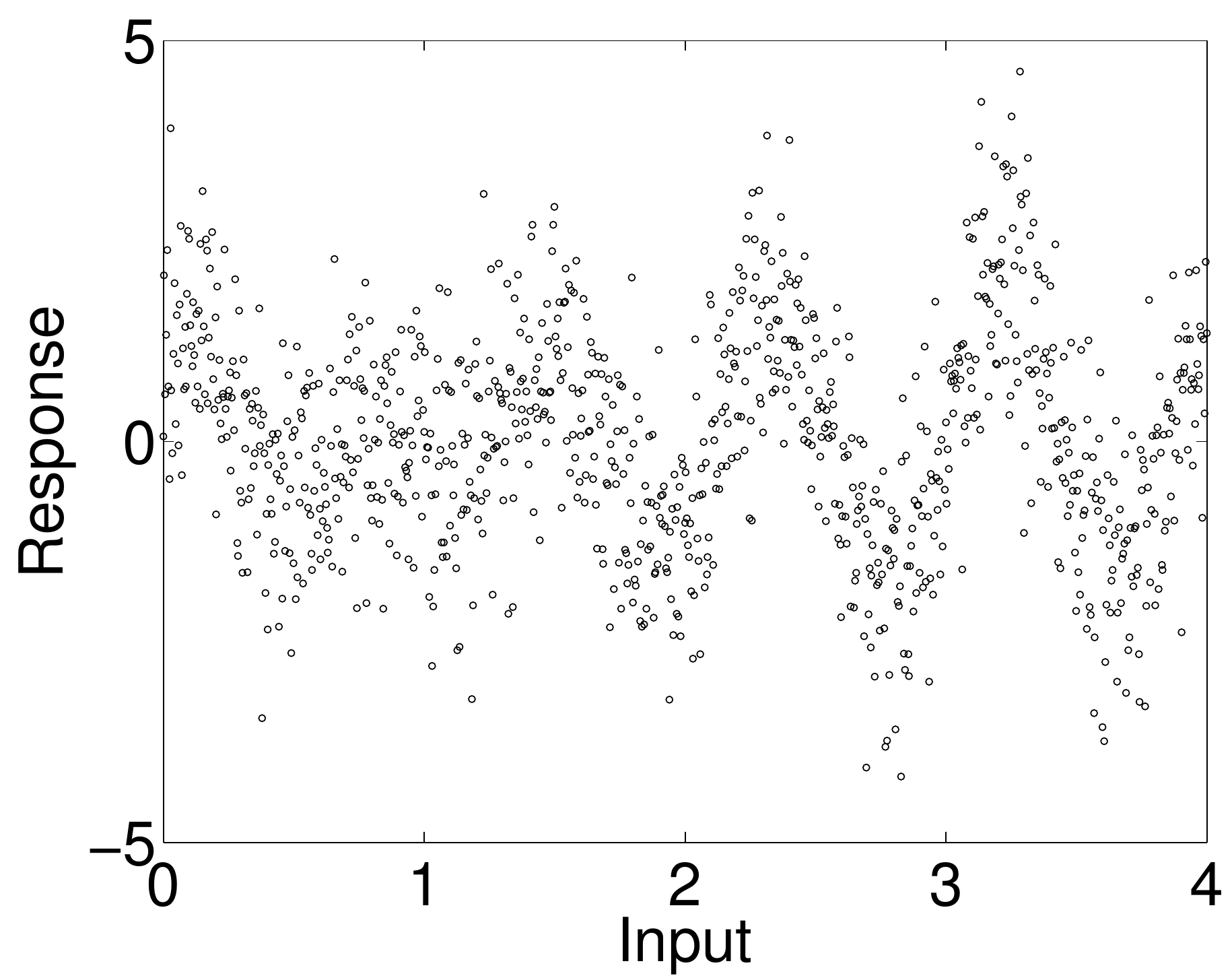} &
\includegraphics[width=35mm,height=32mm]
{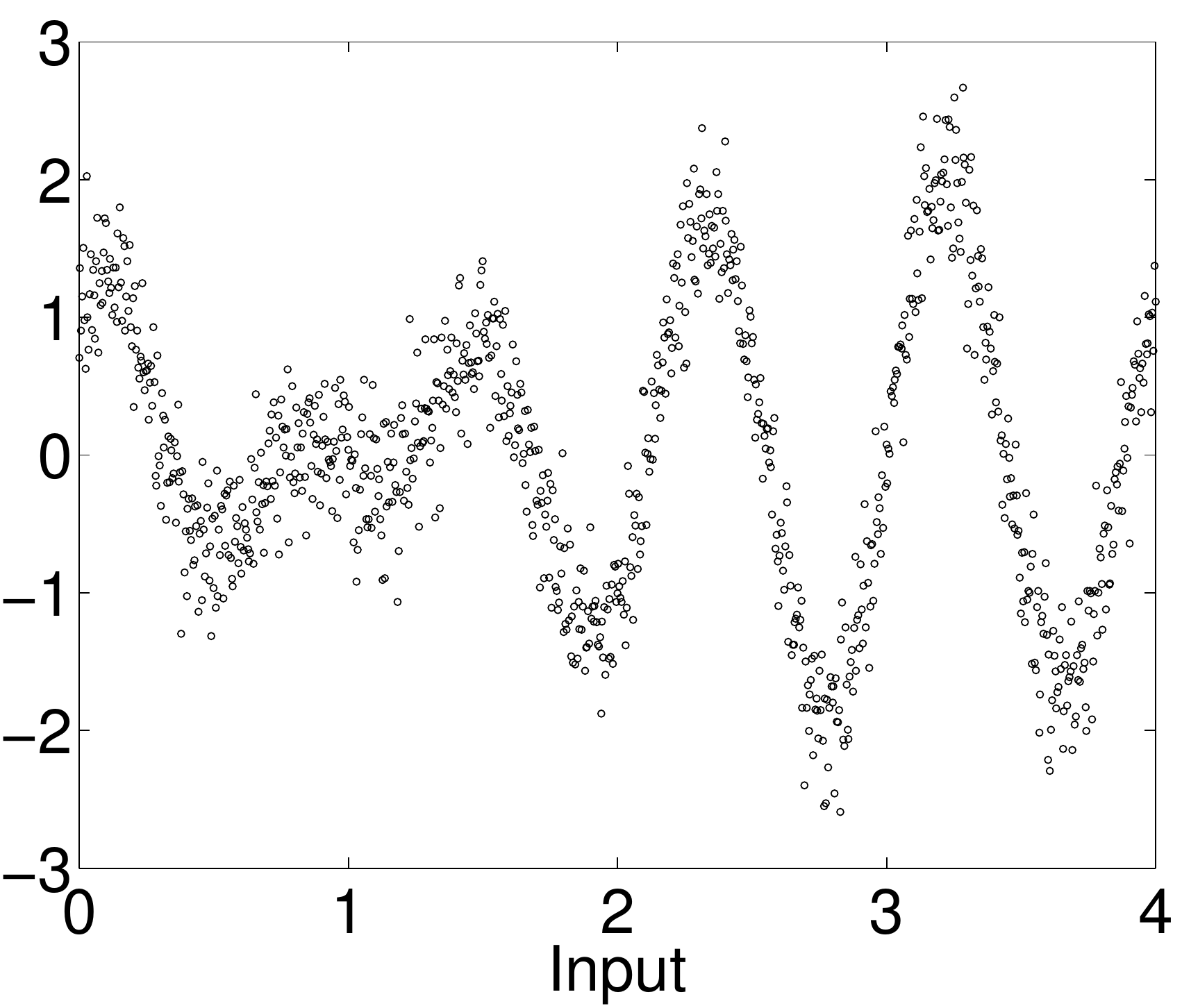} &
\includegraphics[width=35mm,height=32mm]
{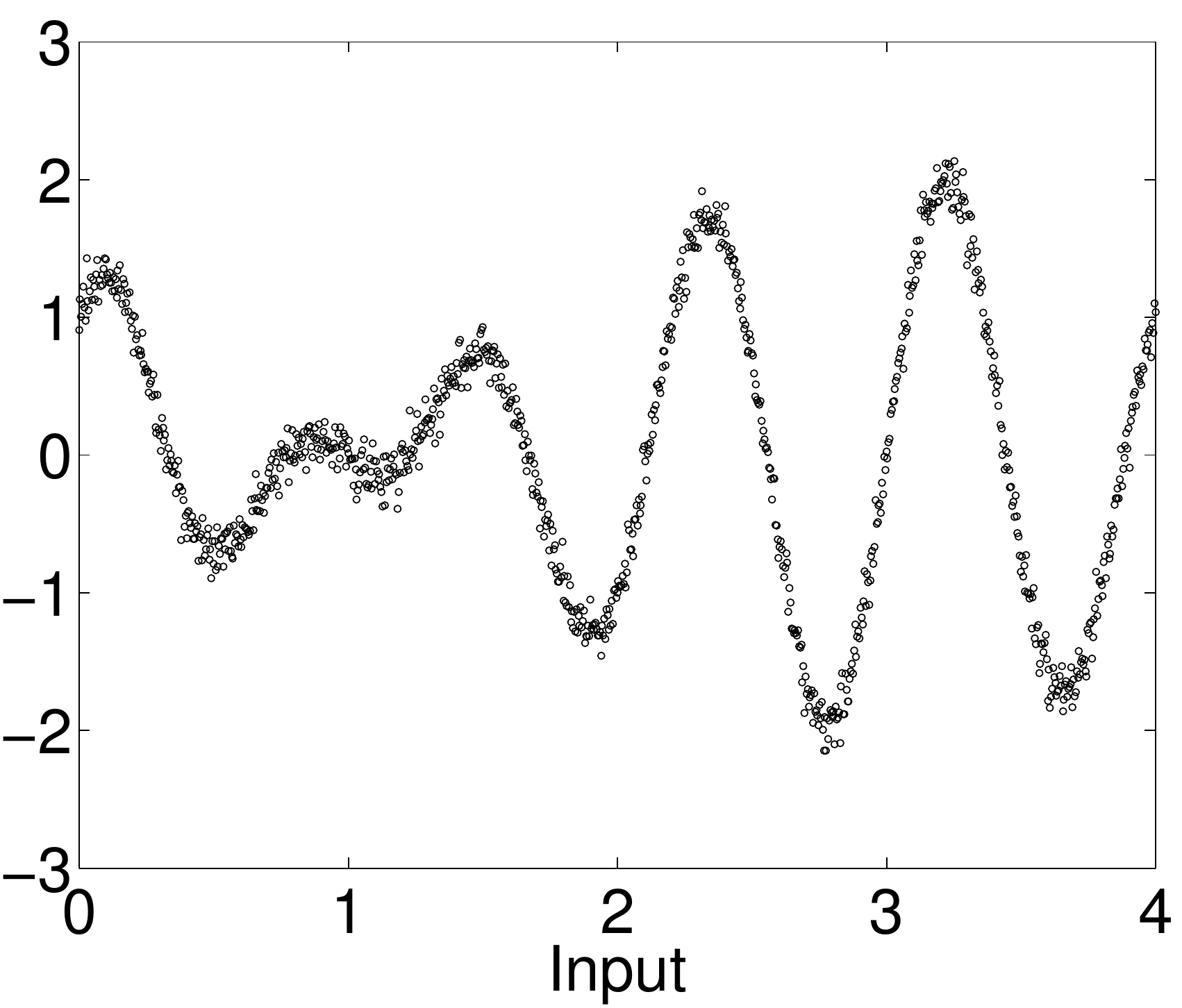} \\
\includegraphics[width=35mm,height=32mm]
{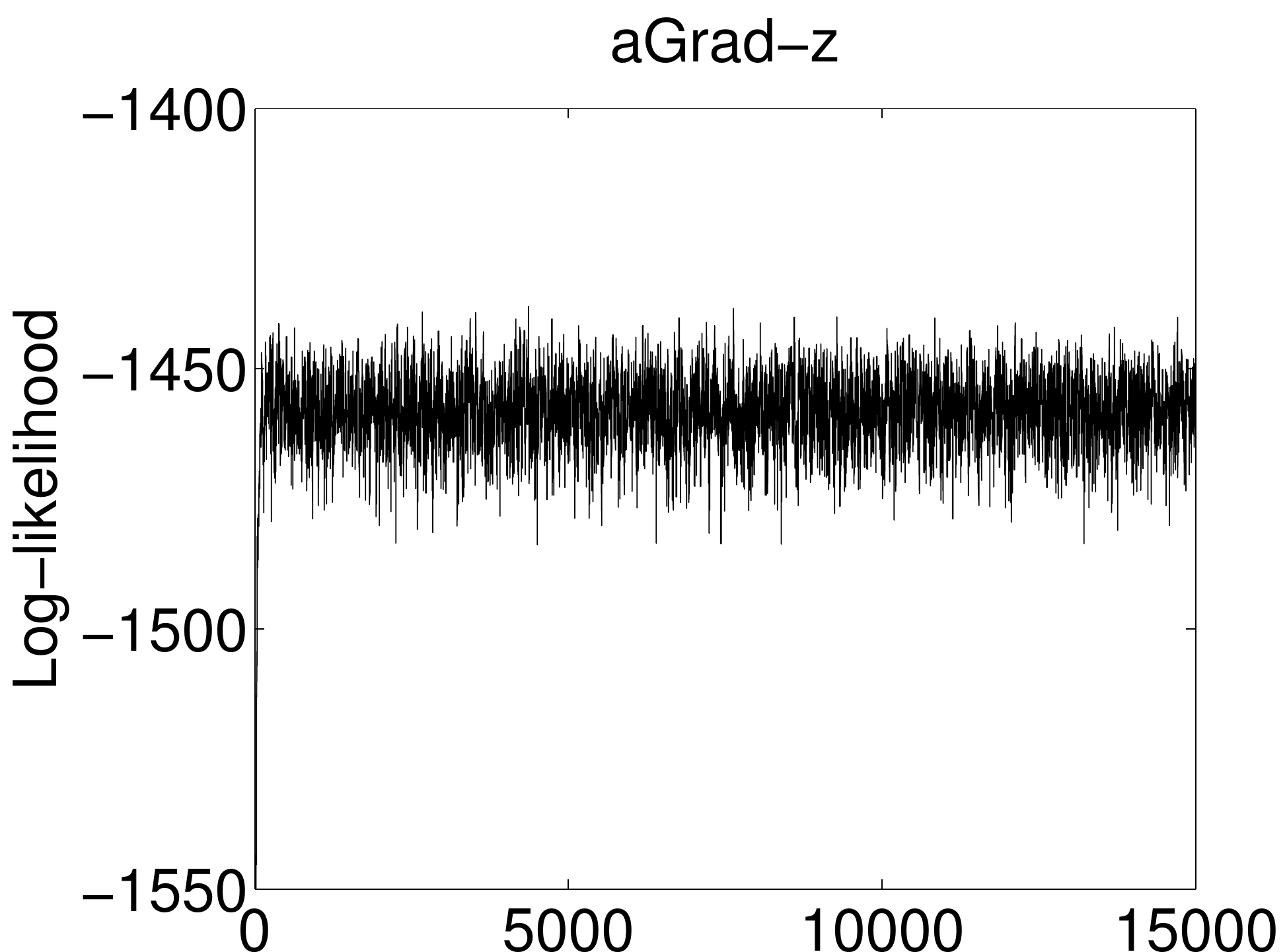} &
\includegraphics[width=35mm,height=32mm]
{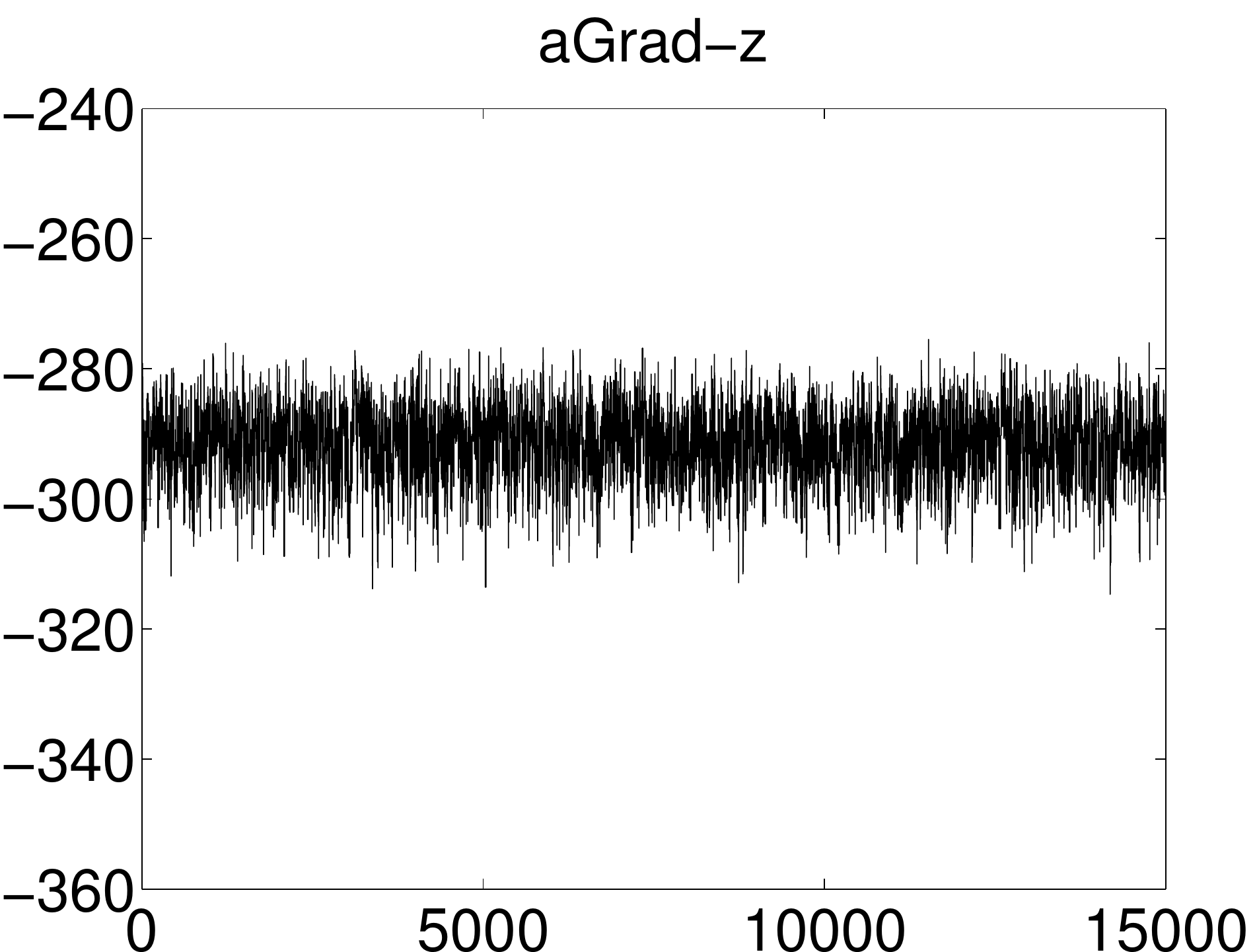} &
\includegraphics[width=35mm,height=32mm]
{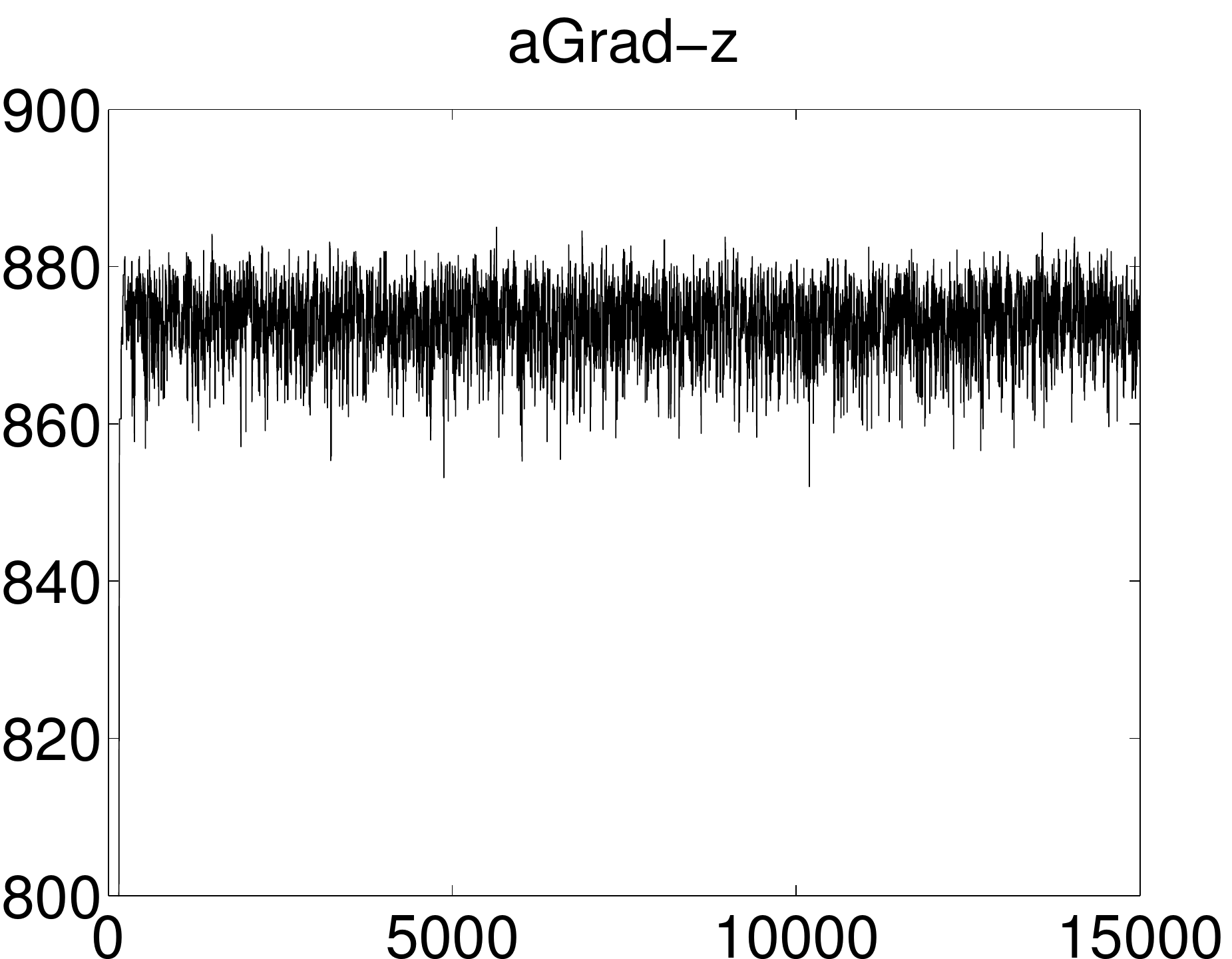} \\
\includegraphics[width=35mm,height=32mm]
{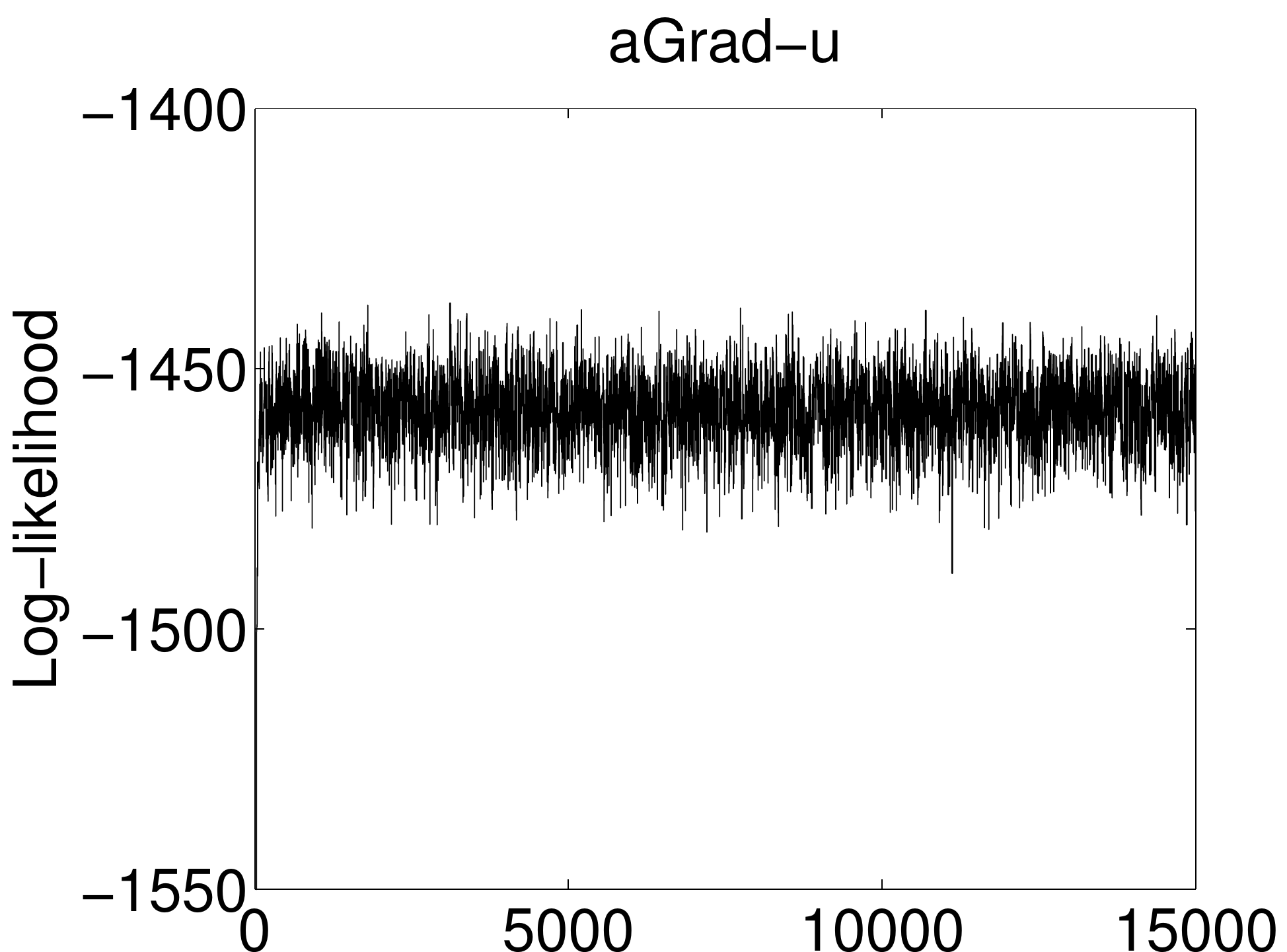} &
\includegraphics[width=35mm,height=32mm]
{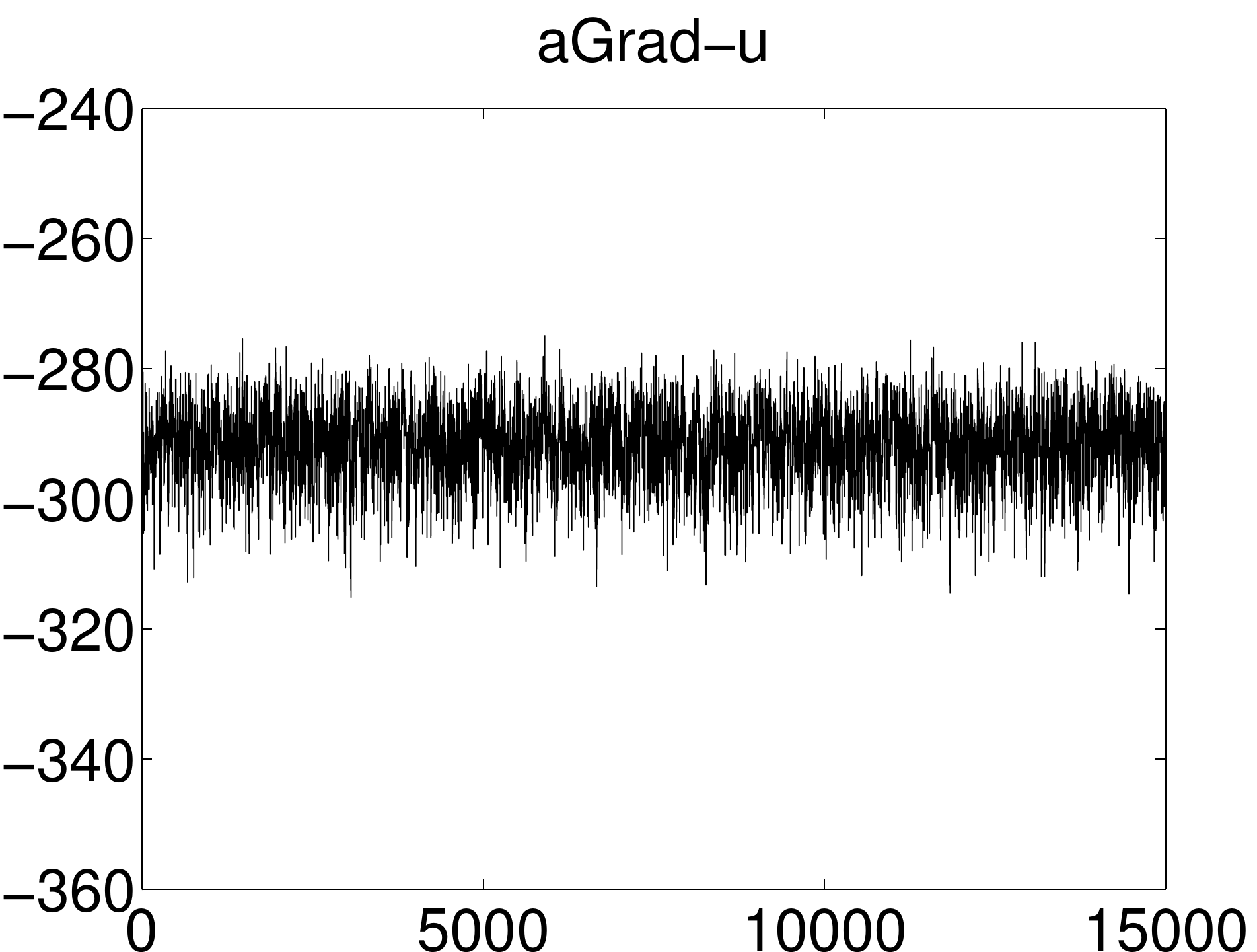} &
\includegraphics[width=35mm,height=32mm]
{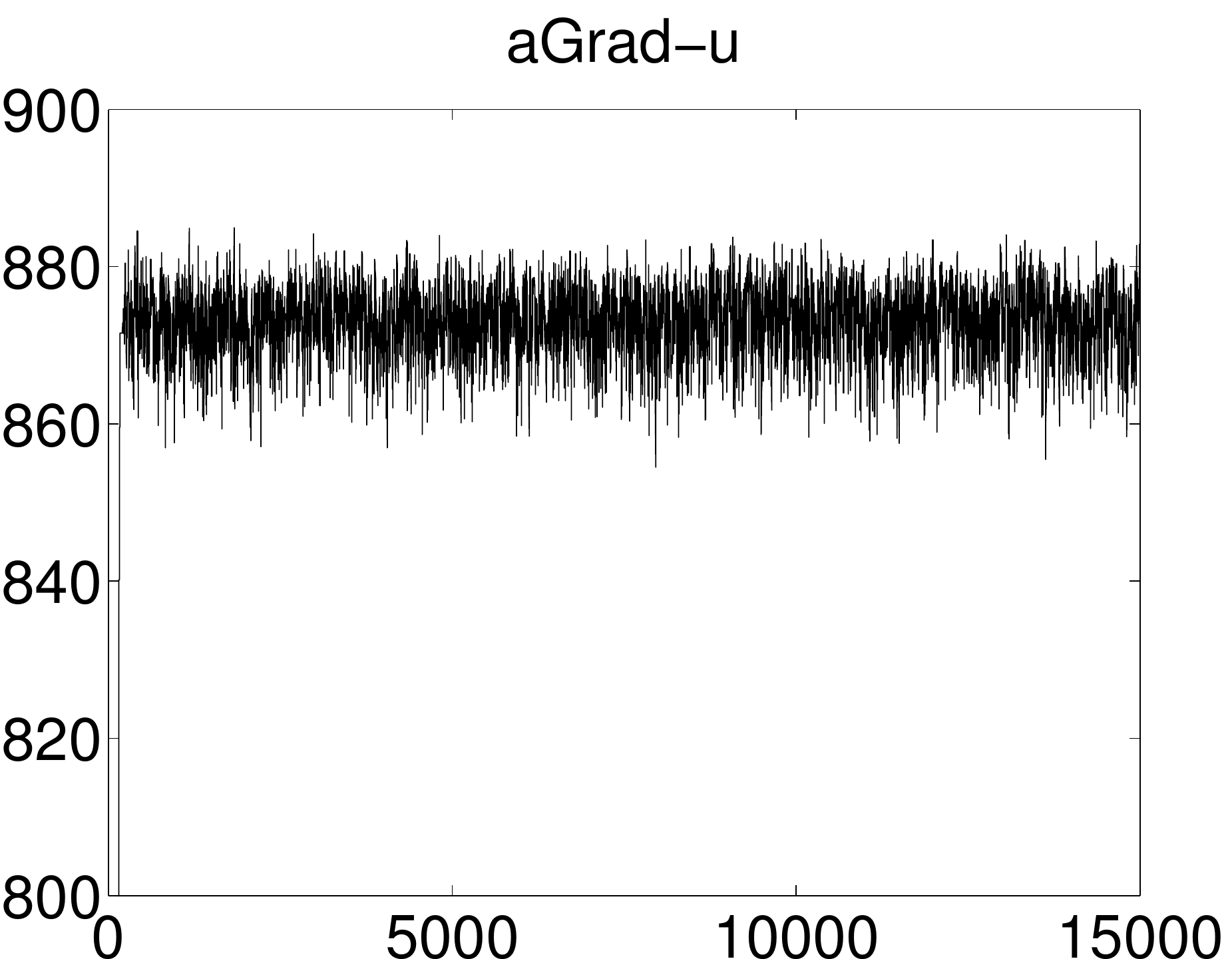} \\
\includegraphics[width=35mm,height=32mm]
{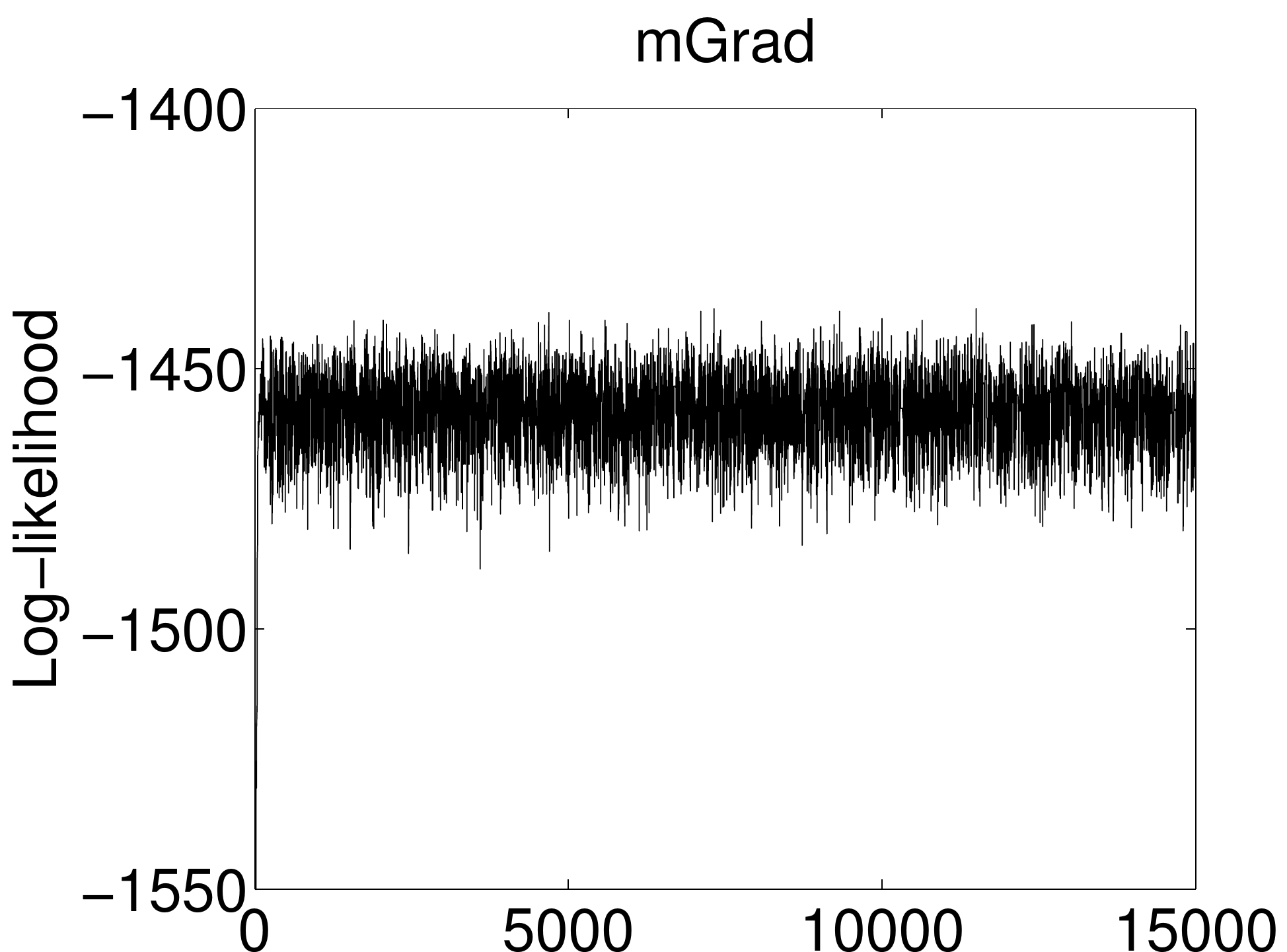} &
\includegraphics[width=35mm,height=32mm]
{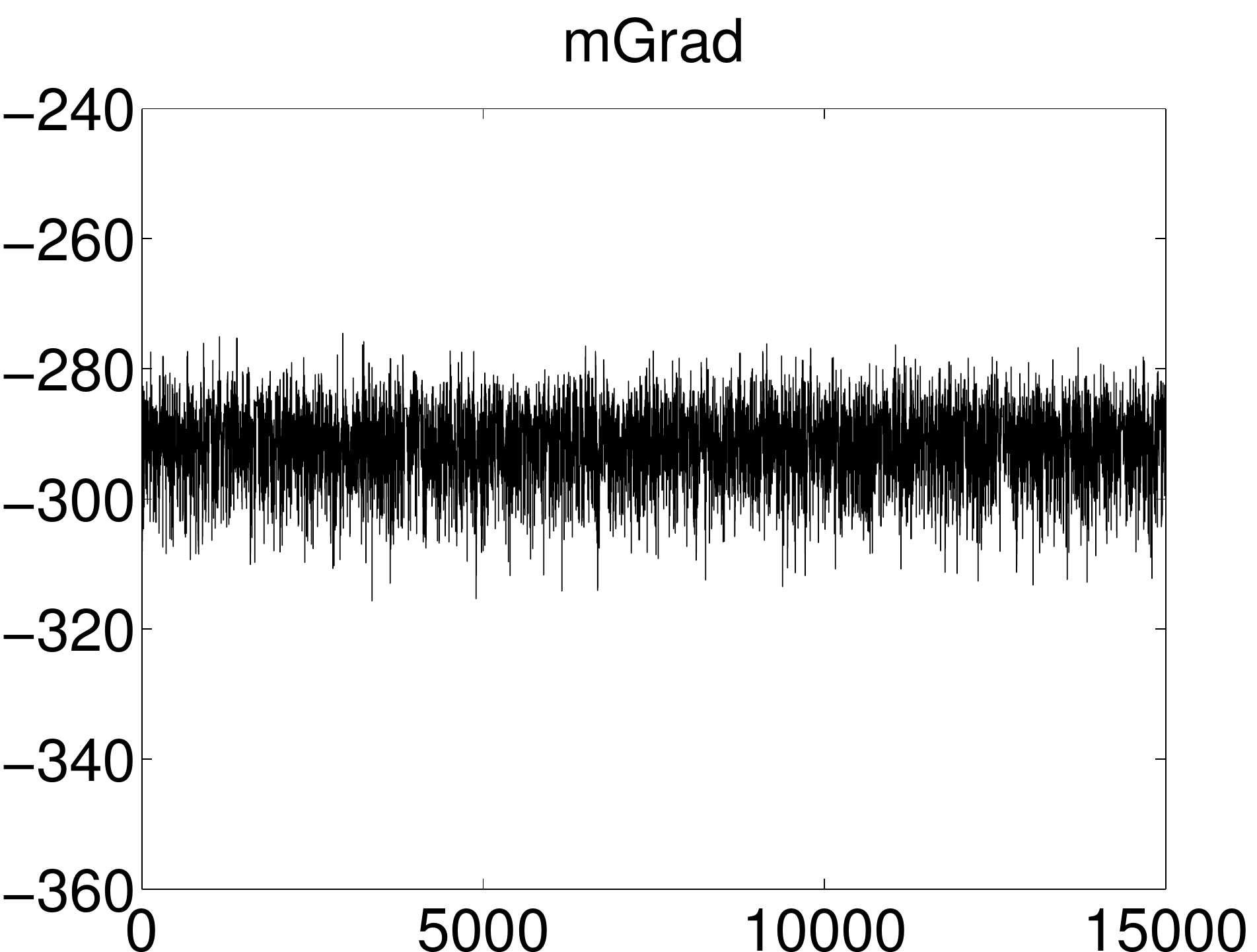} &
\includegraphics[width=35mm,height=32mm]
{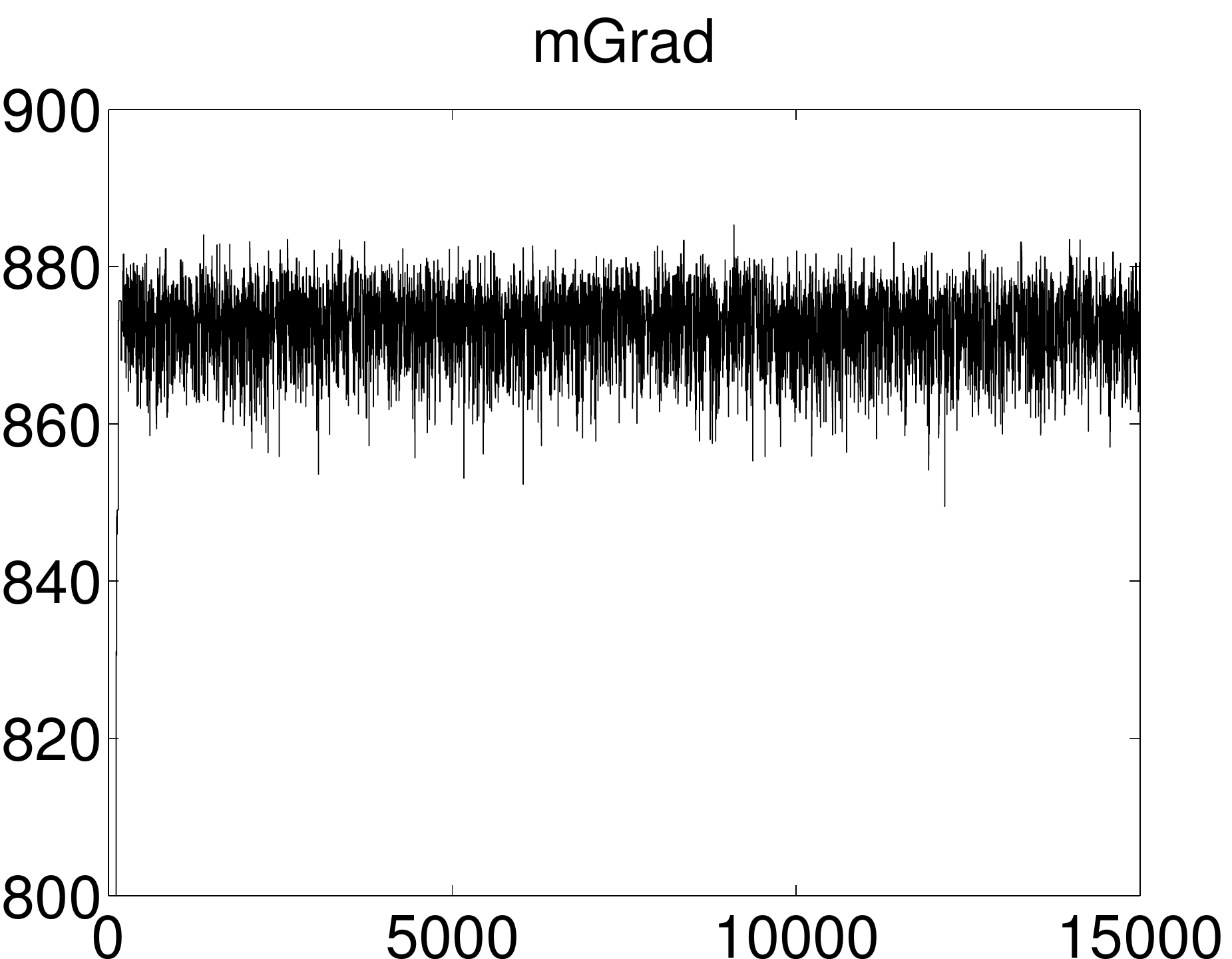} \\
\includegraphics[width=35mm,height=32mm]
{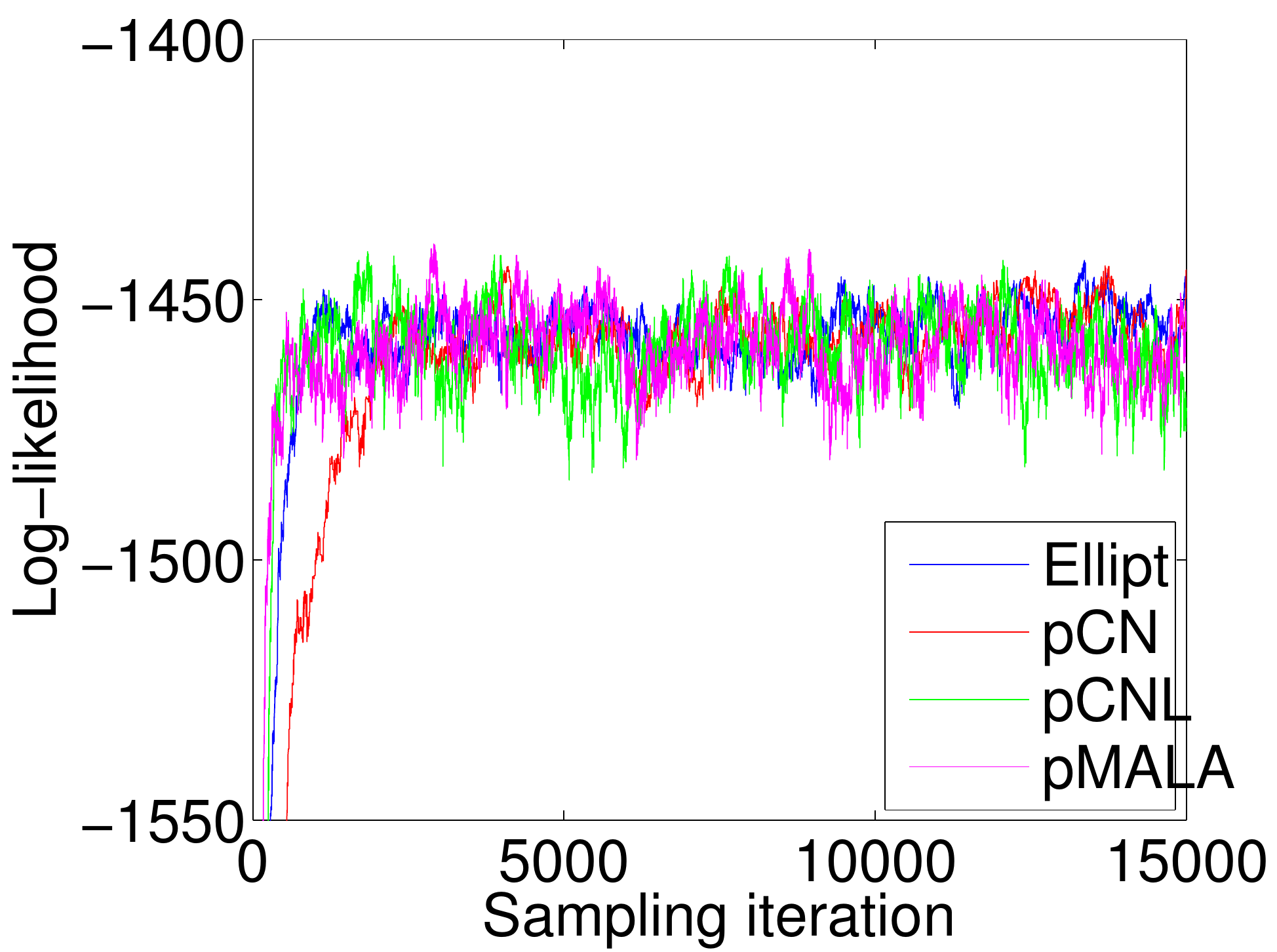} &
\includegraphics[width=35mm,height=32mm]
{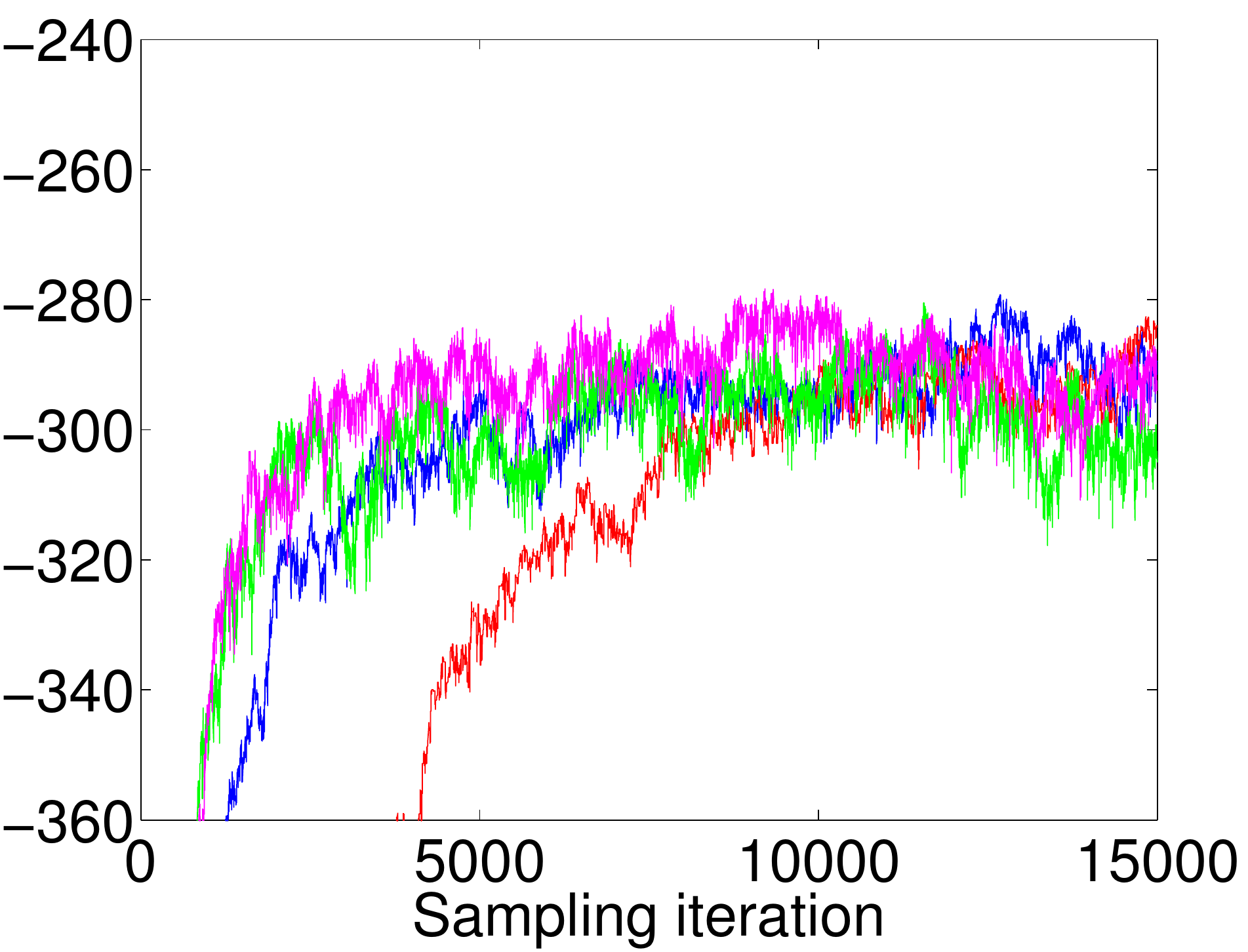} &
\includegraphics[width=35mm,height=32mm]
{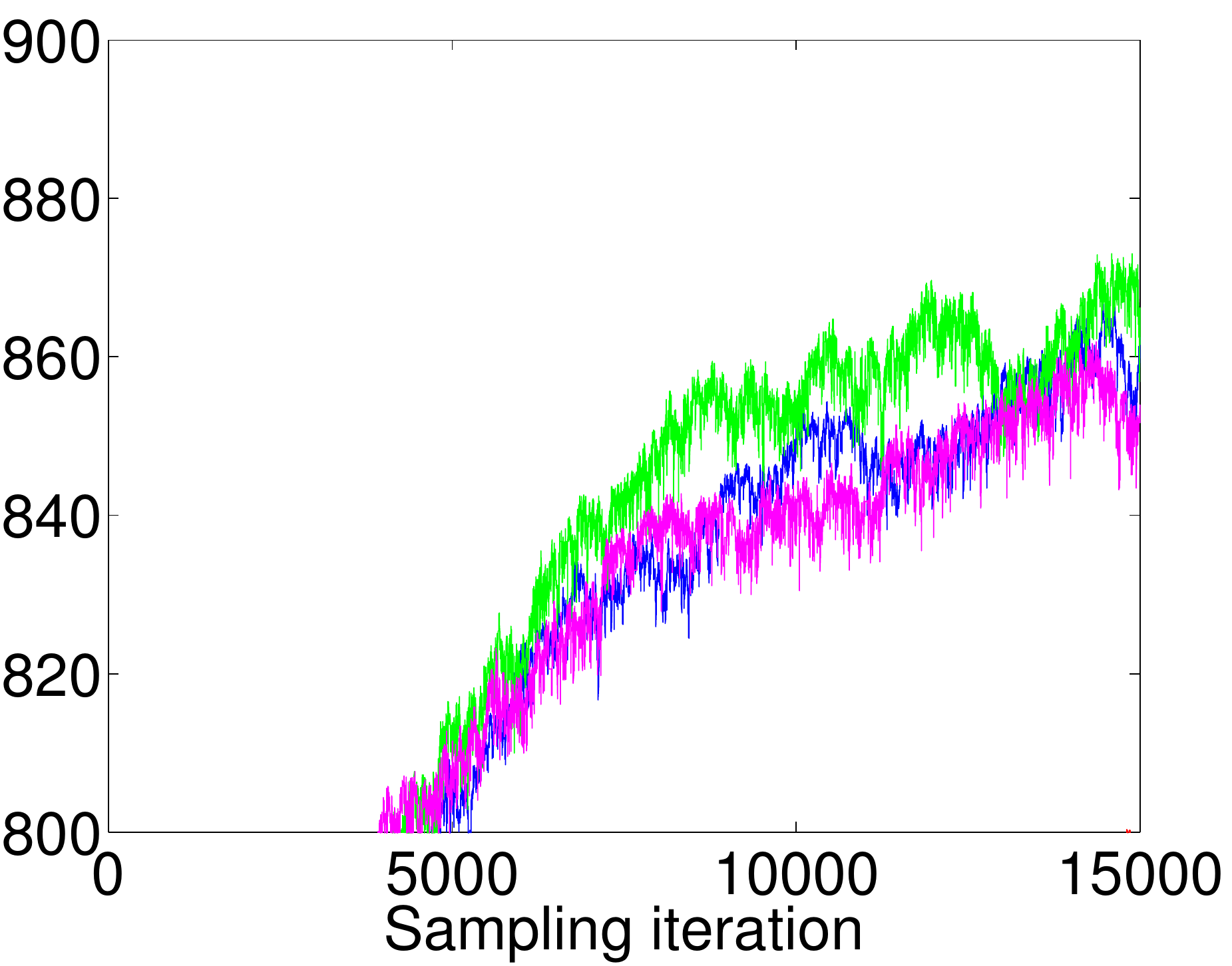} 
\end{tabular}
\caption{First row shows the three datasets of varying observation
  noise. The next rows 
show the
evolution of $f(\bx)$  across iterations for all algorithms. Notice that in the plot at  last row and column the  
log-likelihood values are shown for the first 15000 iterations, while the algorithms 
were actually run longer to reach convergence (see main text).} 
\label{fig:RegressInformLikel}
\end{figure}

\begin{table}
\caption{\label{table:inflik1}Comparison of sampling methods in the  regression dataset with $\sigma^2 =1$.}
\centering
\fbox{%
\begin{tabular}{*{5}{c}}
\em Method &\em Time(s) &\em Step  $\delta$  &\em ESS (Min, Med, Max)  &\em Min ESS/s (s.d.) \\ 
\hline
aGrad-z  &   5.5  &  0.589  &  (391.7, 517.8, 629.0)  &  71.05 (6.41)\\ 
aGrad-u  &  6.8  &  0.673  &  (456.0, 575.2, 696.9)  &  66.66 (6.28) \\ 
mGrad  &  6.0  &  1.337  &  (987.4, 1287.5, 1498.4)  &  167.68 (26.83)\\ 
pMALA  &  21.1  &  0.008  &  (21.0, 99.1, 274.3)  &  1.00 (0.28)\\ 
Ellipt  &  4.1  &   &  (16.7, 66.6, 149.7)  &  4.03 (1.35)\\ 
pCN  &  2.7  &  0.006  &  (14.3, 52.3, 130.1)  &  5.21 (2.40)\\ 
pCNL  &  10.5  &  0.010  &  (31.9, 125.5, 294.5)  &  3.03 (0.94)\\ 
\end{tabular}}
\end{table}

\begin{table}
\caption{\label{table:inflik2}Comparison of sampling methods in the  regression dataset with $\sigma^2 =0.1$.}
\centering
\fbox{%
\begin{tabular}{*{5}{c}}
\em Method &\em Time(s) &\em Step $\delta$  &\em ESS (Min, Med, Max)  &\em Min ESS/s (s.d.) \\ 
\hline
aGrad-z  &   5.5  &  0.053  &  (326.4, 459.9, 575.4)  &  59.05 (7.92)\\ 
aGrad-u  &  6.8  &  0.060  &  (360.4, 500.8, 620.1)  &  52.78 (6.57) \\ 
mGrad  &  5.8  &  0.121  &  (973.6, 1196.1, 1408.1)  &  168.11 (10.69)\\ 
pMALA  &  21.2  &  0.001  &  (12.7, 58.1, 164.4)  &  0.60 (0.26)\\ 
Ellipt  &  4.3  &   &  (9.9, 42.4, 138.6)  &  2.33 (0.81)\\ 
pCN  &  2.7  &  0.001  &  (8.7, 37.5, 108.5)  &  3.19 (1.25)\\ 
pCNL  &  10.4  &  0.001  &  (12.7, 57.4, 208.5)  &  1.21 (0.59)\\ 
\end{tabular}}
\end{table}

\begin{figure}
\centering
{\includegraphics[scale=0.6]
{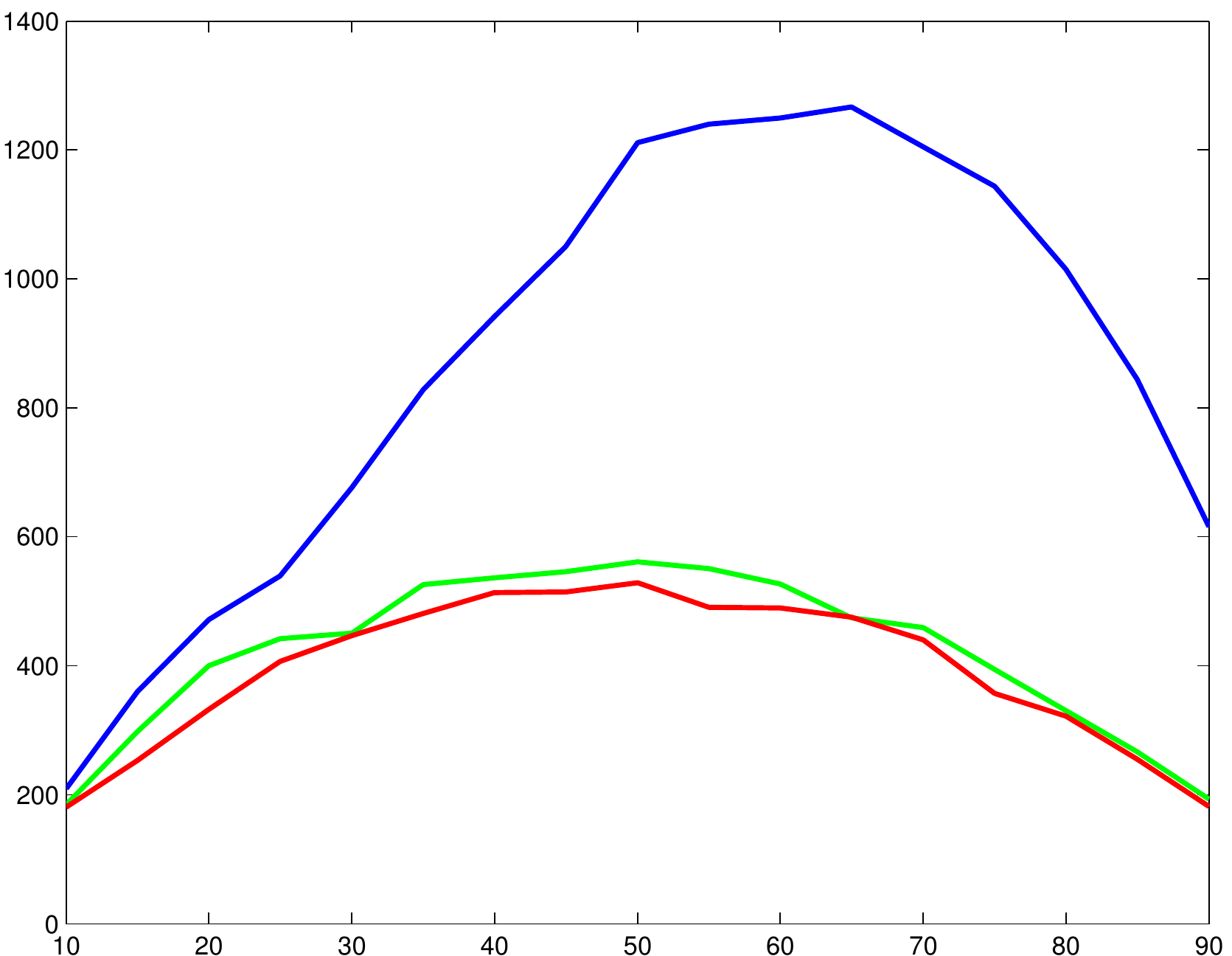}}
\caption{\rev{Effective sample size (ESS) shown in the vertical axis  versus acceptance rate shown in the horizontal axis  for the three proposed algorithms (aGrad-z is shown in red, aGrad-u with green and mGrad with blue). 
To produce this plot we have repeatedly  tuned all algorithms to achieve acceptance rates in a grid of values between 10\%  and 90\%.  Then, each time we estimated the ESS 
based on the first component of the latent vector $\bx$.  All lines are averages across  all three regression datasets and from ten simulation repeats associated with ten different random seeds.}  
\label{fig:essVsaccept}}
\end{figure}

\rev{Also,  Figure \ref{fig:essVsaccept}  provides an empirical investigation of the optimal acceptance rate for the three proposed algorithms using these regression datasets. This and similar 
investigations suggest  that tuning the step size $\delta$ so that to achieve an acceptance rate around 50\% to 60\%, as done in all our experiments, is effective.}

\subsection*{Additional experiments for log-Gaussian Cox process}

\begin{figure}
\centering
\begin{tabular}{ccc}
\includegraphics[scale=0.2]
{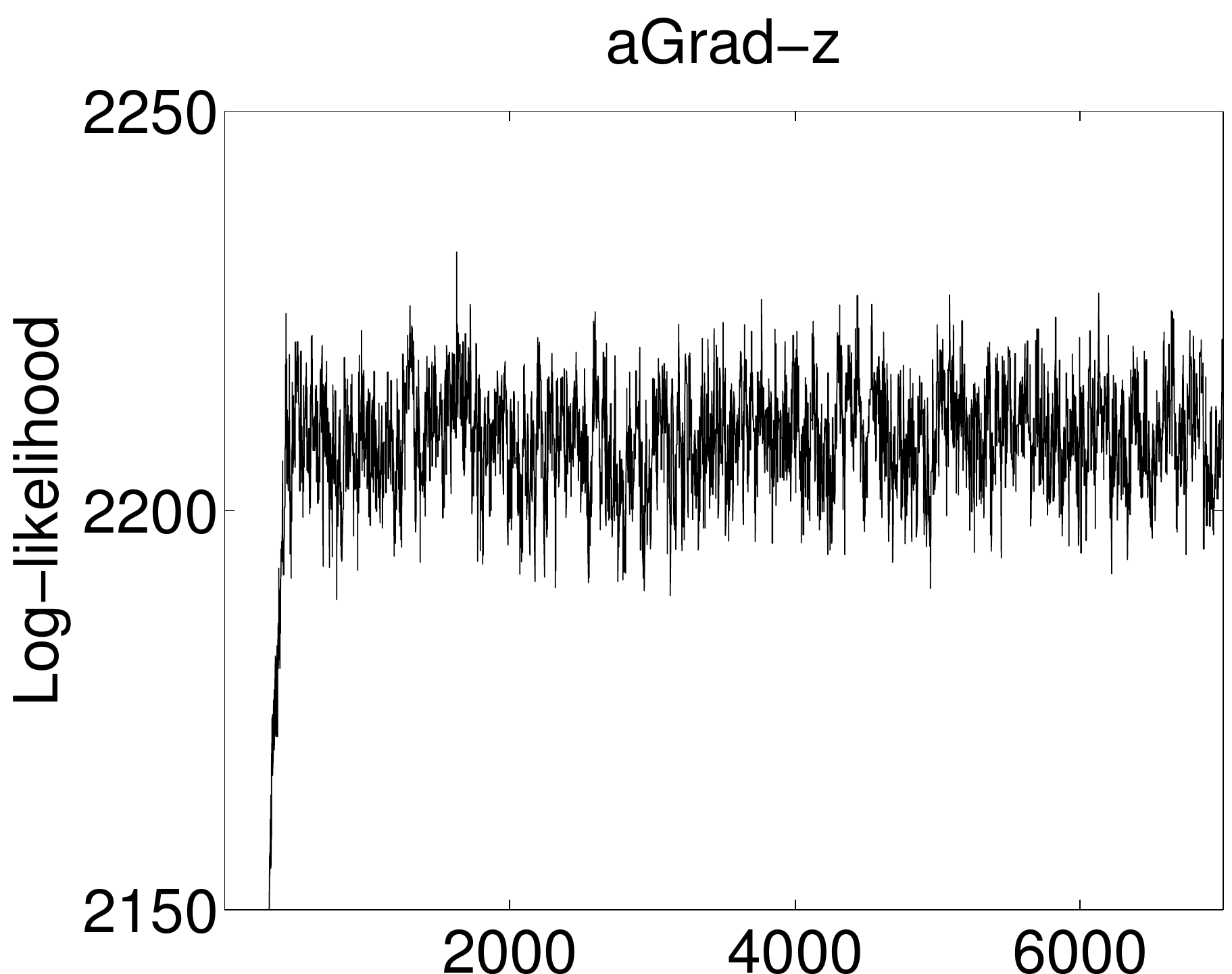} &
\includegraphics[scale=0.2]
{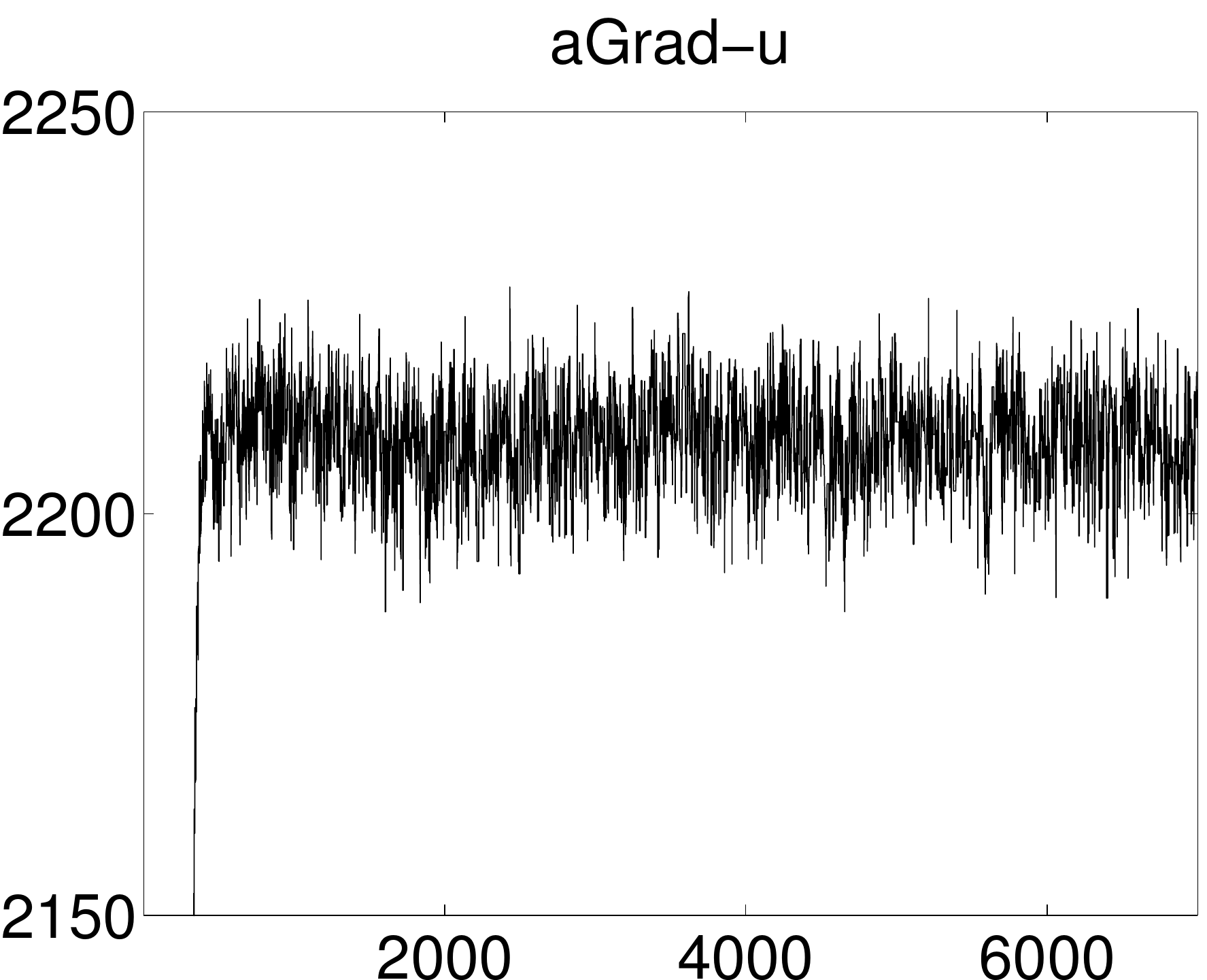} &
\includegraphics[scale=0.2]
{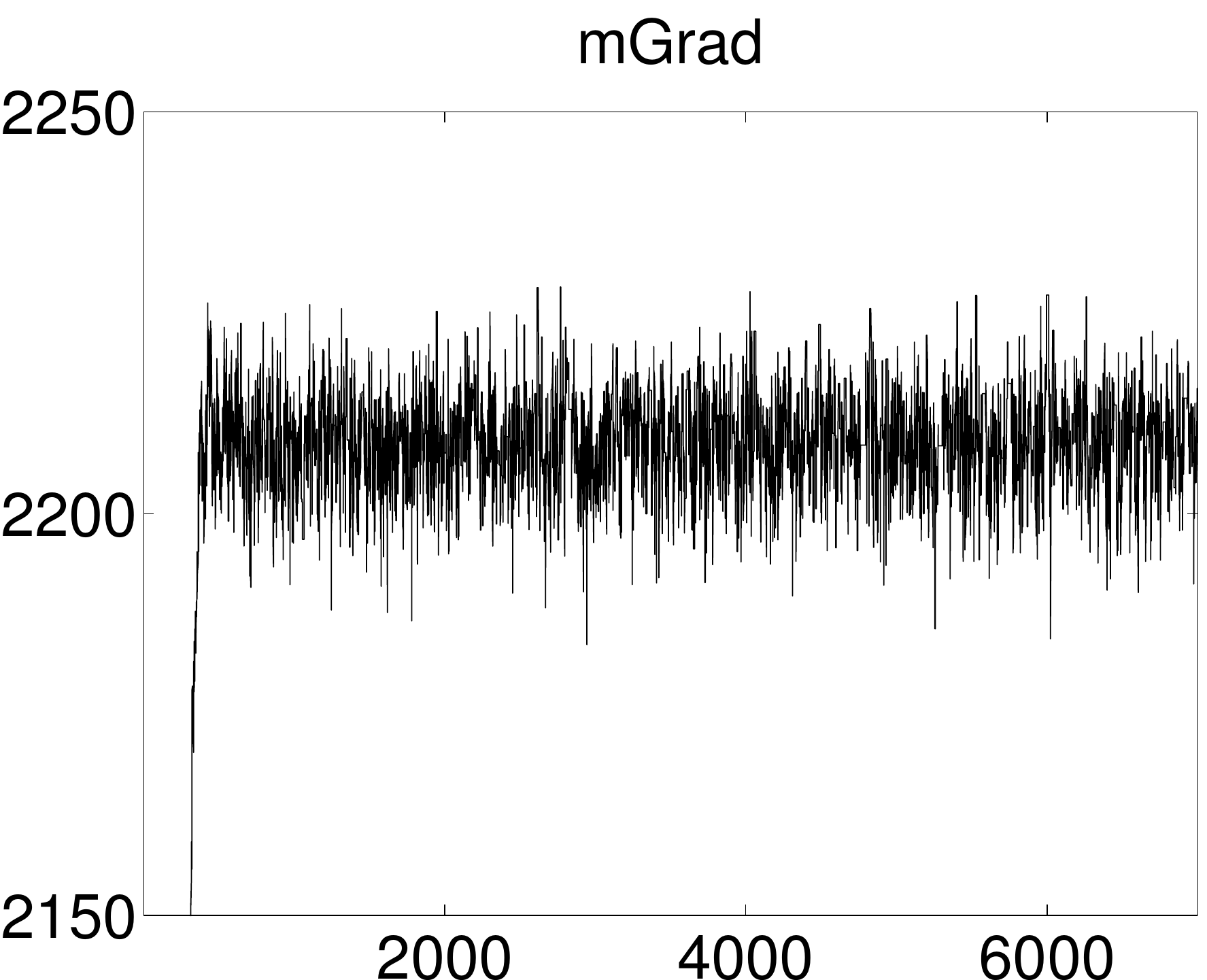} \\
\includegraphics[scale=0.2]
{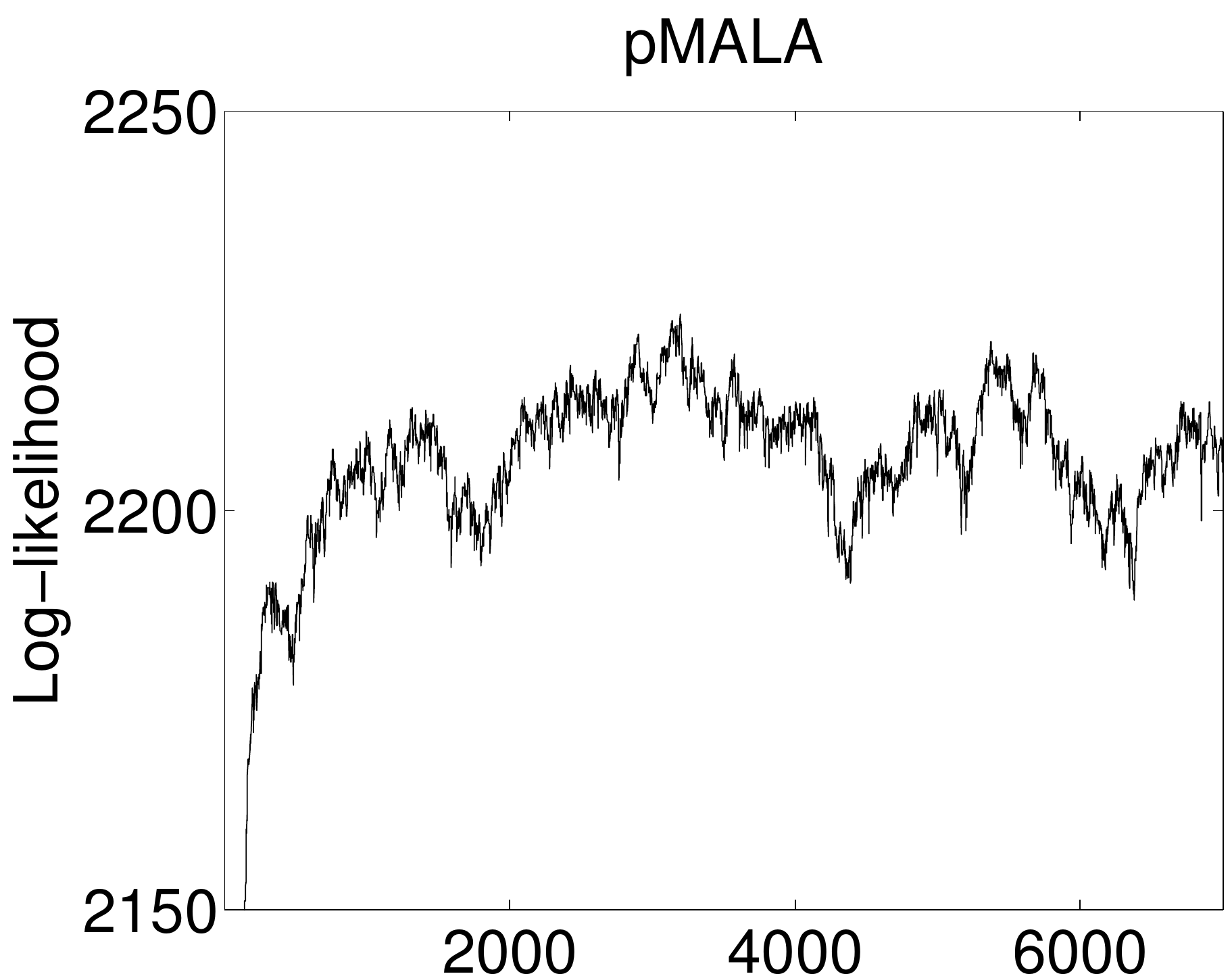} &
\includegraphics[scale=0.2]
{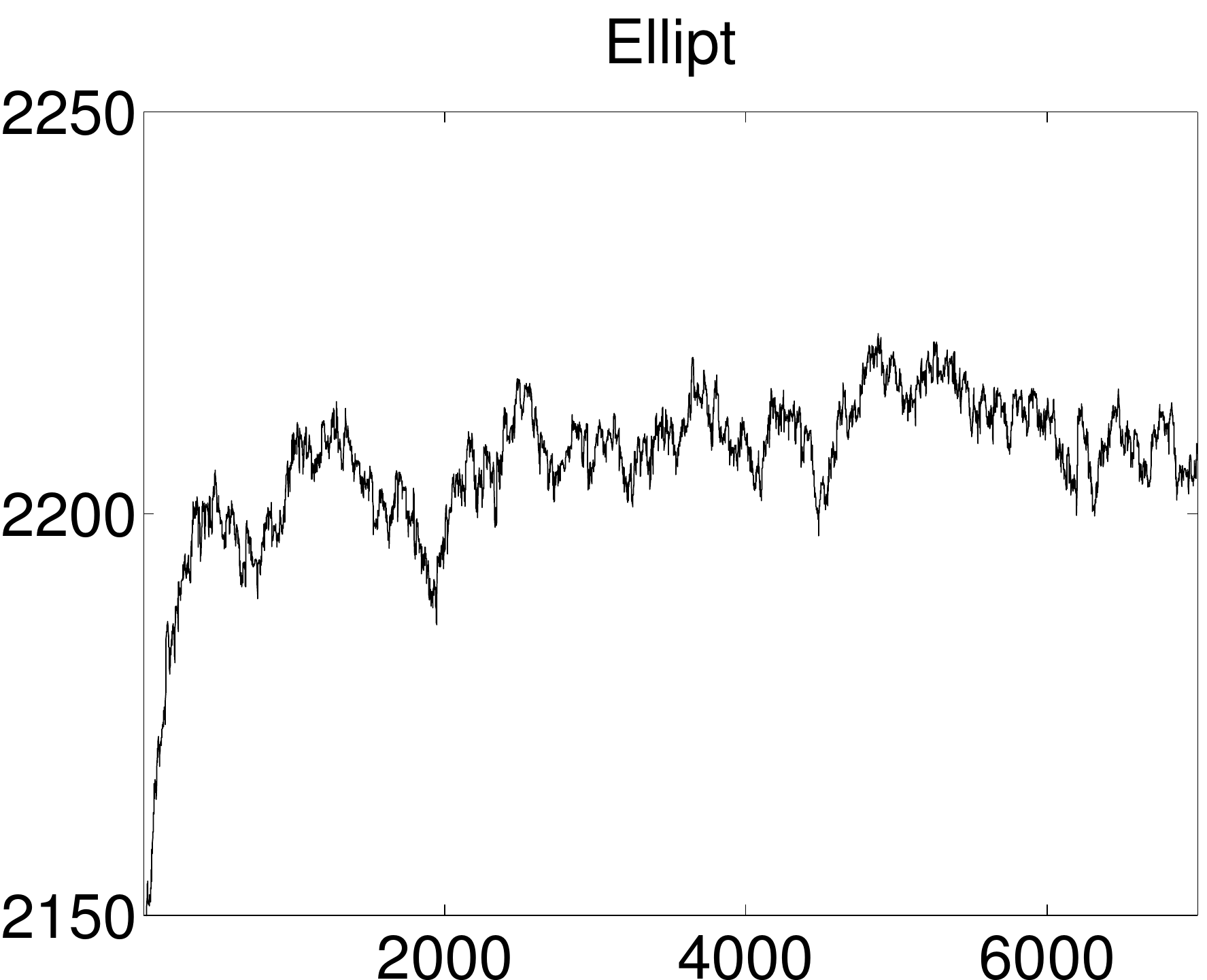} &
\includegraphics[scale=0.2]
{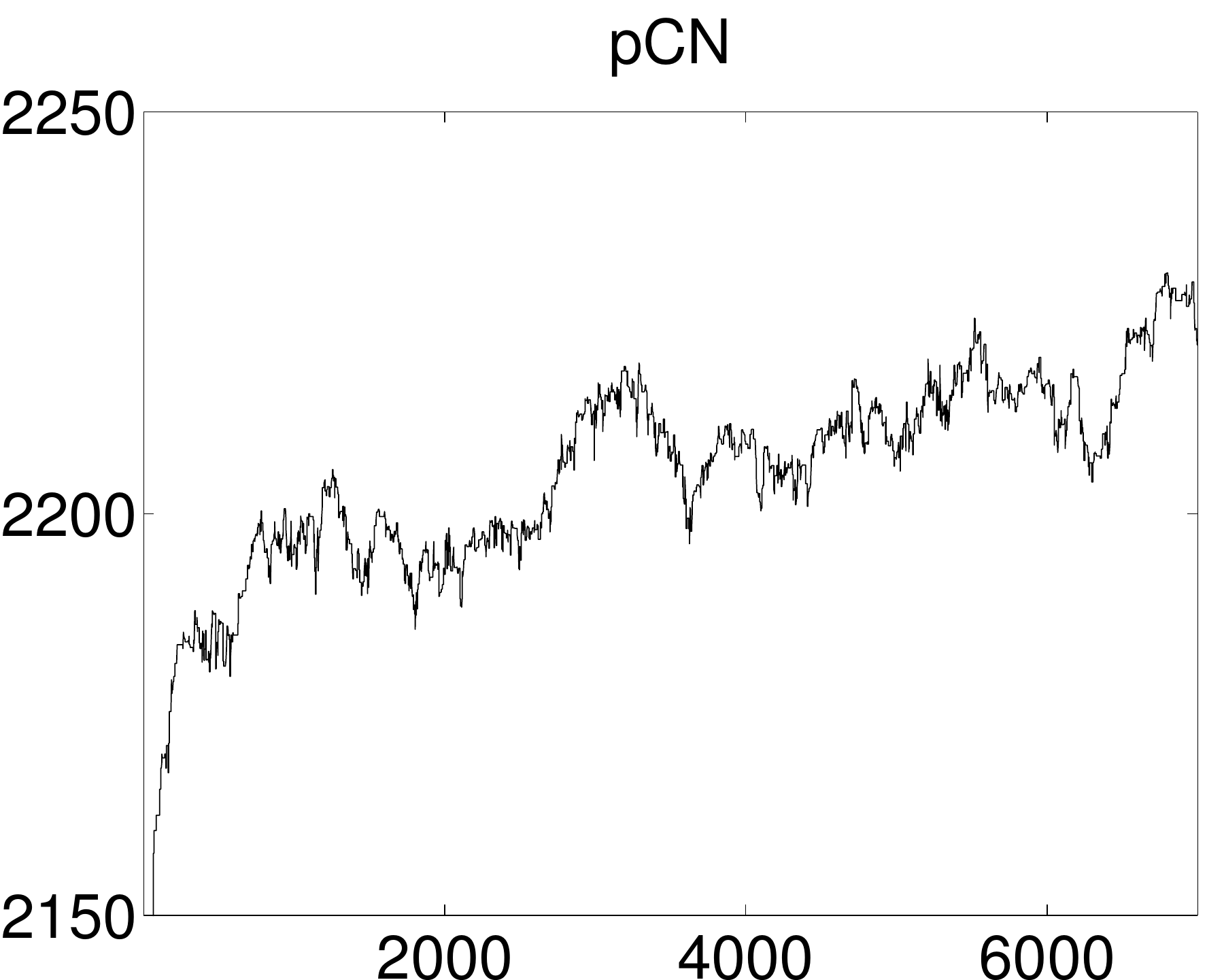} \\
\includegraphics[scale=0.2]
{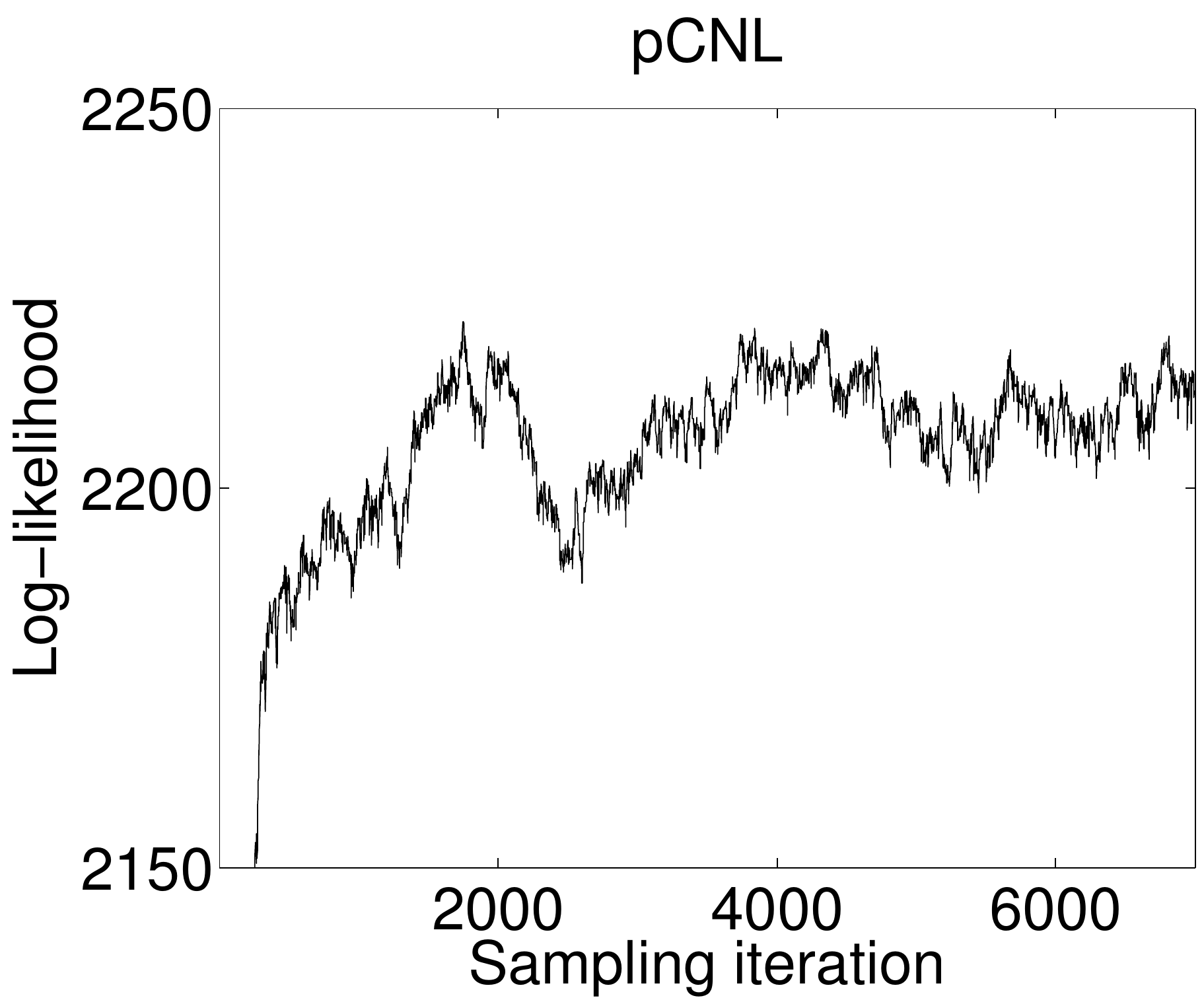} &
\includegraphics[scale=0.2]
{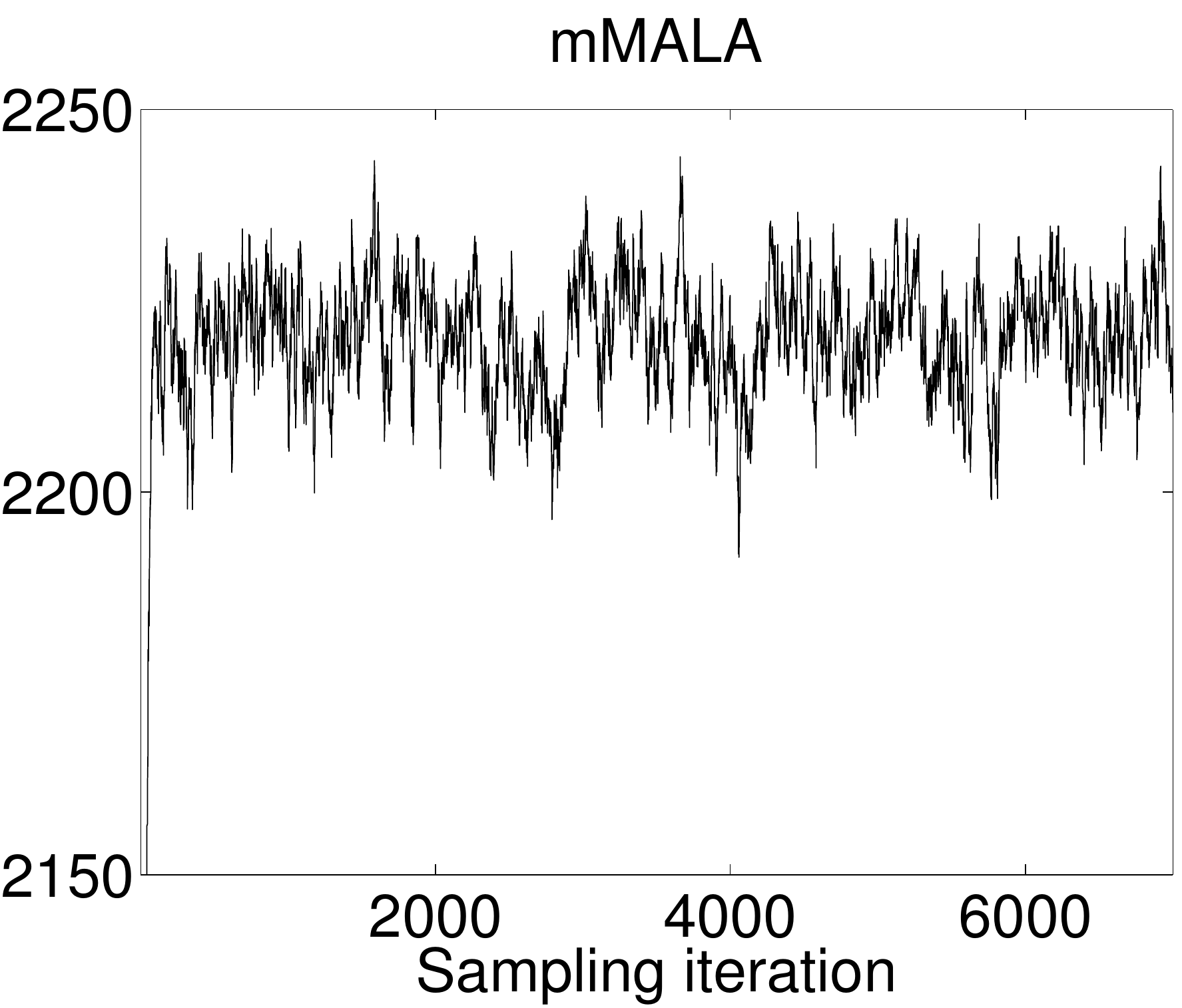} &
\includegraphics[scale=0.2]
{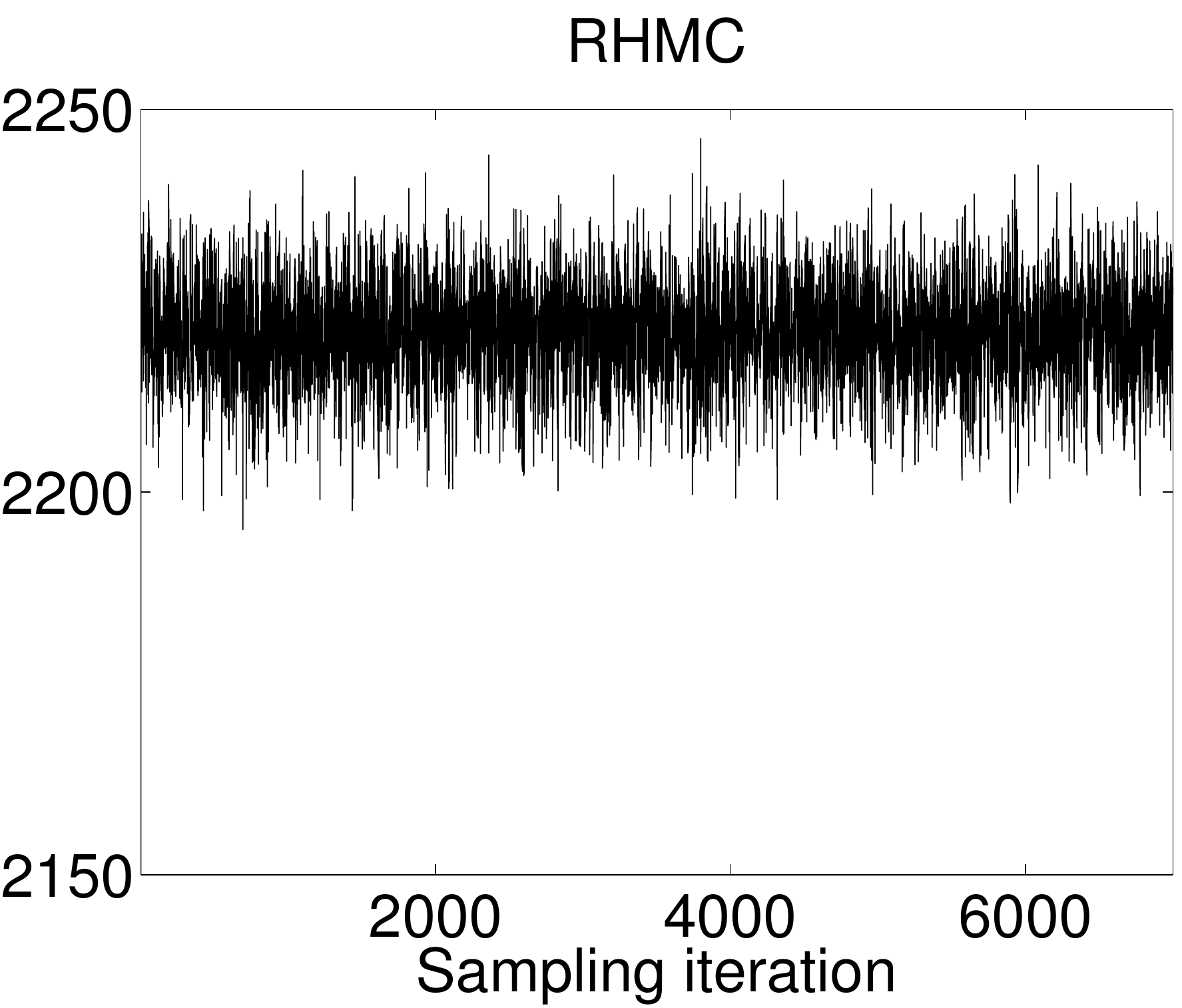} 
\end{tabular}
\caption{The evolution of the log-likelihood values across all iterations 
for all sampling schemes in the log-Gaussian Cox process example.
\label{fig:GaussianCoxLogL}}
\end{figure}

For the original dataset where $\bx$ has dimensionality $n=4096$, i.e.\ the dataset used in the main article, 
Figure \ref{fig:GaussianCoxLogL} plots the evolution of the log-likelihood $f(\bx)$ that illustrates
convergence and sampling efficiency for all compared algorithms. 

Furthermore, in order to investigate how the dimensionality of $\bx$ affects performance, we 
down-sampled the initial dataset by making the mesh grid sparser so that it became 
of size  $32 \times 32$. This was done by merging every four neighbouring
cells into one (so that the area of a new cell becomes $m = 1/1024$) while the observed counts 
assigned to the initial cells are summed up to form the observed count in the larger cell. 
This creates a down-sampled dataset where the latent field $\bx$ has dimensionality $n=1024$, i.e.\
four times smaller than the initial dimensionality.  
Notice that this procedure should keep the likelihood information essentially the same, but 
we should expect now each likelihood term to become (on average) four times more informative about the value of its latent 
value $x_{ij}$. We re-run all sampling methods in this down-sampled dataset and 
Table \ref{table:logGaussianCox2} reports the results.  We can observe that
the ESS scores remain at the same levels as for the finer grid
dataset, while  Min ESS/s gets obviously larger due to the 
smaller running times. A notable feature in Table \ref{table:logGaussianCox2} is that now the 
step sizes $\delta$ found by aGrad-z, aGrad-u and 
mGrad become approximately four times smaller than the ones for the
finer dataset. This 
shows that the step sizes are mostly determined by the noise in the likelihood terms, so that in the 
down-sampled dataset the noise is reduced due to the aggregation of data from the neighbouring cells. 

\begin{table}
\caption{\label{table:logGaussianCox2}Comparison of sampling methods in the log-Gaussian Cox model dataset in 
the down-sampled version where $n=1024$.}
\centering
\fbox{%
\begin{tabular}{*{5}{c}}
\em Method &\em Time(s) &\em Step $\delta$  &\em ESS (Min, Med, Max)  &\em Min ESS/s (s.d.) \\ 
\hline
aGrad-z  &   3.5  &  0.264  &  (41.0, 202.8, 512.4)  &  11.76 (1.95)\\ 
aGrad-u  &  4.5  &  0.691  &  (93.3, 444.2, 1038.7)  &  21.06 (5.11) \\ 
mGrad  &  3.8  &  1.410  &  (170.9, 771.1, 1582.7)  &  46.78 (16.42)\\ 
pMALA  &  12.3  &  0.006  &  (3.6, 12.6, 53.8)  &  0.29 (0.02)\\ 
Ellipt  &  3.6  &   &  (4.7, 17.4, 63.5)  &  1.32 (0.19)\\ 
pCN  &  2.0  &  0.011  &  (3.5, 11.7, 50.8)  &  1.78 (0.32)\\ 
pCNL  &  6.0  &  0.006  &  (3.6, 12.7, 51.3)  &  0.59 (0.03)\\ 
mMALA  &   17.7  &  0.070  &  (22.1, 90.5, 185.5)  &  1.25 (0.33) \\ 
RHMC  &  77.9  &  0.100  &  (1747.2, 4601.9, 5000.0)  &  22.52 (1.94) \\ 
\end{tabular}}
\end{table}

 \subsection*{Additional experiments with binary regression}
Figure \ref{fig:HeartLogL} shows the evolution 
of the log-likelihood values for the ``Heart'' dataset, detailed
results for which are shown in the main article.  For the rest four
datasets such plots look similar. The figure
illustrates convergence and sampling efficiency. 
Tables \ref{table:australian}, \ref{table:ripley}, \ref{table:german}, \ref{table:pima} 
show the performance of the different sampling schemes in the rest four binary classification datasets.

\begin{table}
\caption{\label{table:australian}Comparison of sampling methods in Australian Credit dataset. The size of the
latent vector $\bx$ is $n = 690$ and the input dimensionality is $D = 14$.}
\centering
\fbox{%
\begin{tabular}{*{5}{c}}
\em Method &\em Time(s) &\em Step $\delta$  &\em ESS (Min, Med, Max)  &\em Min ESS/s (s.d.) \\ 
\hline
aGrad-z  &   3.0  &  1.520  &  (47.0, 153.6, 424.7)  &  15.50 (4.64)\\ 
aGrad-u  &  3.5  &  2.649  &  (96.1, 263.2, 647.3)  &  27.60 (3.89) \\ 
mGrad  &  3.0  &  5.300  &  (198.2, 504.2, 1191.0)  &  66.07 (10.03)\\ 
pMALA  &  7.0  &  0.009  &  (5.1, 21.2, 89.5)  &  0.73 (0.09)\\ 
Ellipt  &  4.3  &   &  (6.3, 24.8, 94.7)  &  1.47 (0.19)\\ 
pCN  &  1.7  &  0.013  &  (4.2, 17.6, 75.4)  &  2.56 (0.23)\\ 
pCNL  &  2.6  &  0.009  &  (5.3, 22.1, 88.8)  &  2.03 (0.30)\\ 
\end{tabular}}
\end{table}

\begin{table}
\caption{\label{table:german}Comparison of sampling methods in German Credit dataset. The size of the
latent vector x is $n = 1000$ and the input dimensionality is $D = 24$.}
\centering
\fbox{%
\begin{tabular}{*{5}{c}}
\em Method &\em Time(s) &\em Step $\delta$  &\em ESS (Min, Med, Max)  &\em Min ESS/s (s.d.) \\ 
\hline
aGrad-z  &   4.4  &  0.681  &  (47.1, 161.3, 345.4)  &  10.74 (1.49)\\ 
aGrad-u  &  5.3  &  1.134  &  (78.0, 257.0, 518.2)  &  14.65 (2.66) \\ 
mGrad  &  4.5  &  2.689  &  (176.0, 528.0, 929.1)  &  38.73 (6.63)\\ 
pMALA  &  15.4  &  0.016  &  (6.4, 26.6, 91.6)  &  0.42 (0.06)\\ 
Ellipt  &  5.4  &   &  (4.6, 19.8, 66.7)  &  0.85 (0.07)\\ 
pCN  &  2.6  &  0.013  &  (3.8, 14.6, 60.5)  &  1.48 (0.14)\\ 
pCNL  &  7.7  &  0.016  &  (6.2, 26.4, 92.1)  &  0.81 (0.11)\\ 
\end{tabular}}
\end{table}

\begin{table}
\caption{\label{table:pima}Comparison of sampling methods in Pima Indian dataset. The size of the latent
vector $\bx$ is $n = 532$ and the input dimensionality is $D = 7$.}
\centering
\fbox{%
\begin{tabular}{*{5}{c}}
\em Method &\em Time(s) &\em Step $\delta$  &\em ESS (Min, Med, Max)  &\em Min ESS/s (s.d.) \\ 
\hline
aGrad-z  &   2.5  &  1.667  &  (96.6, 299.1, 735.5)  &  39.25 (9.70)\\ 
aGrad-u  &  2.9  &  2.943  &  (176.0, 529.8, 1072.6)  &  60.66 (6.49) \\ 
mGrad  &  2.4  &  6.134  &  (322.2, 947.8, 1594.5)  &  133.75 (13.25)\\ 
pMALA  &  3.6  &  0.053  &  (28.4, 93.0, 233.0)  &  7.93 (2.16)\\ 
Ellipt  &  3.2  &   &  (19.2, 66.6, 175.4)  &  6.06 (1.28)\\ 
pCN  &  1.4  &  0.055  &  (12.2, 43.1, 128.7)  &  8.73 (1.16)\\ 
pCNL  &  2.1  &  0.056  &  (27.8, 93.0, 235.9)  &  13.35 (3.46)\\ 
\end{tabular}}
\end{table}

\begin{table}
\caption{\label{table:ripley}Comparison of sampling methods in Ripley dataset. The size of the latent vector
$\bx$ is $n = 250$ and the input dimensionality is $D = 2$.}
\centering
\fbox{%
\begin{tabular}{*{5}{c}}
\em Method &\em Time(s) &\em Step  $\delta$  &\em ESS (Min, Med, Max)  &\em Min ESS/s (s.d.) \\ 
\hline
aGrad-z  &   1.3  &  3.589  &  (21.5, 125.5, 630.0)  &  16.58 (7.09)\\ 
aGrad-u  &  1.4  &  4.526  &  (28.8, 163.5, 805.3)  &  19.95 (8.43) \\ 
mGrad  &  1.2  &  9.014  &  (47.0, 289.4, 1730.3)  &  38.04 (11.87)\\ 
pMALA  &  1.5  &  0.017  &  (13.8, 58.2, 240.4)  &  8.94 (0.87)\\ 
Ellipt  &  1.9  &   &  (11.2, 51.2, 206.9)  &  5.76 (1.63)\\ 
pCN  &  0.7  &  0.019  &  (7.3, 30.9, 140.6)  &  10.70 (3.51)\\ 
pCNL  &  1.1  &  0.017  &  (14.3, 56.1, 222.1)  &  13.09 (3.25)\\ 
\end{tabular}}
\end{table}

\begin{figure}
\centering
\begin{tabular}{ccc}
{\includegraphics[scale=0.2]
{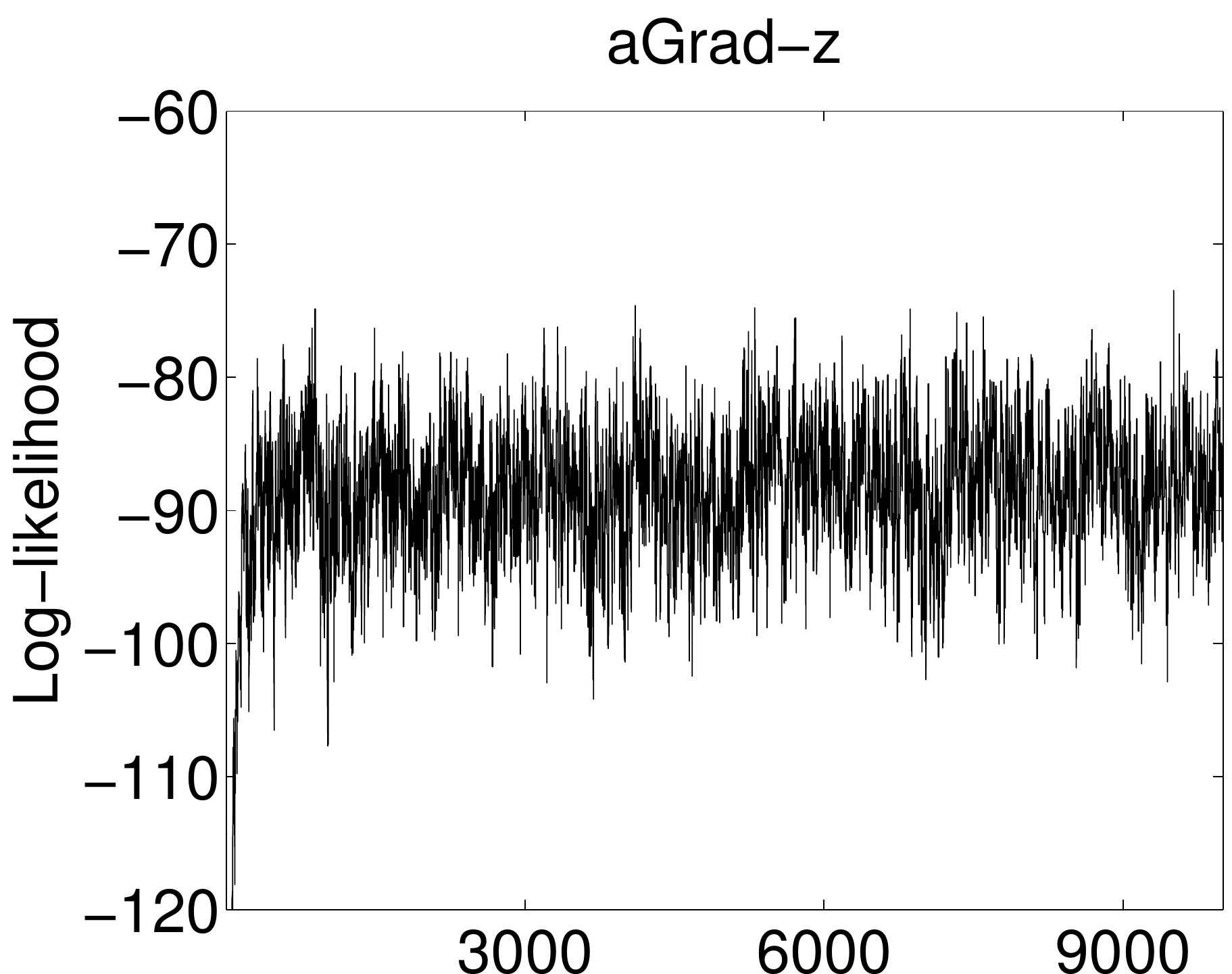}} &
{\includegraphics[scale=0.2]
{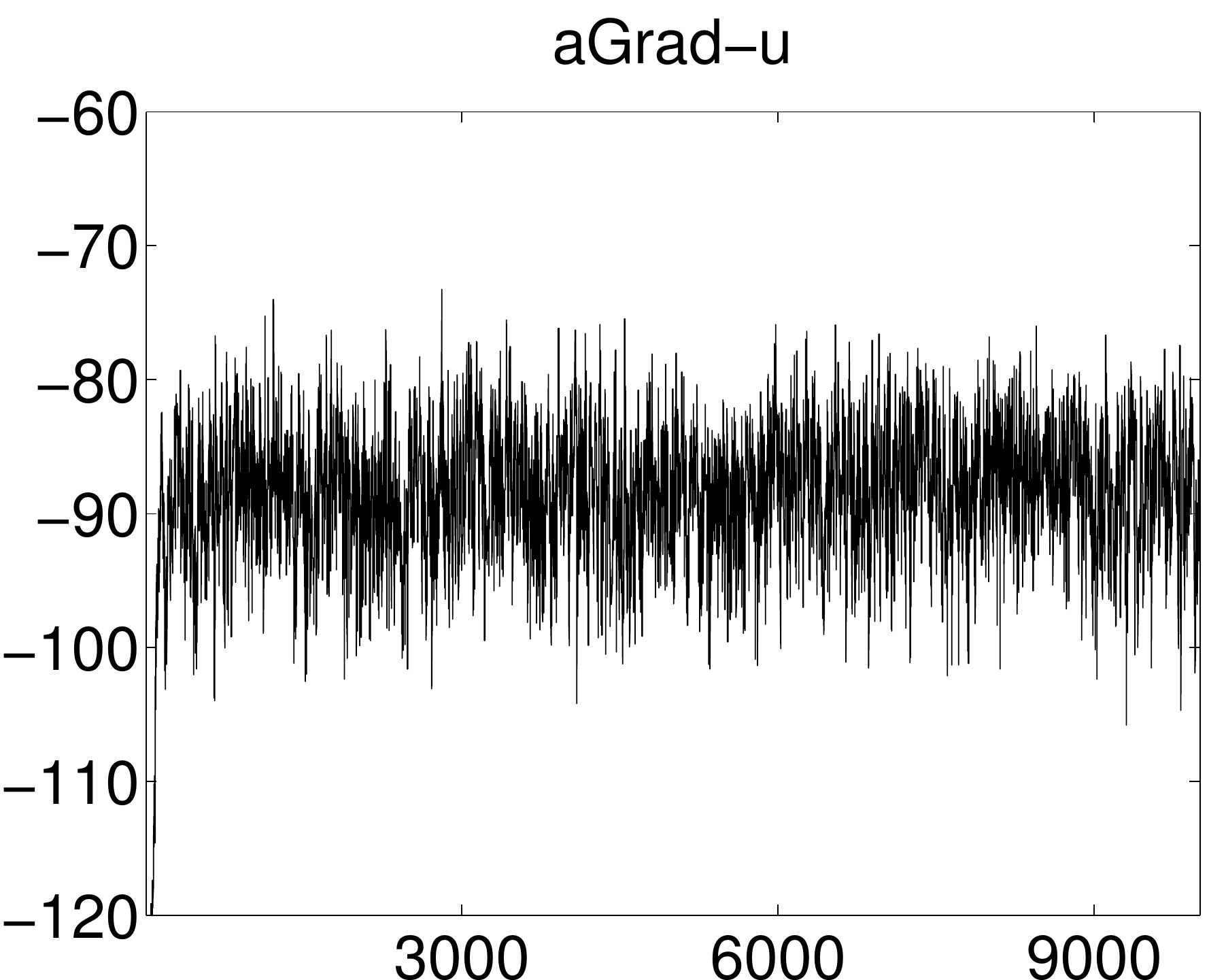}} &
{\includegraphics[scale=0.2]
{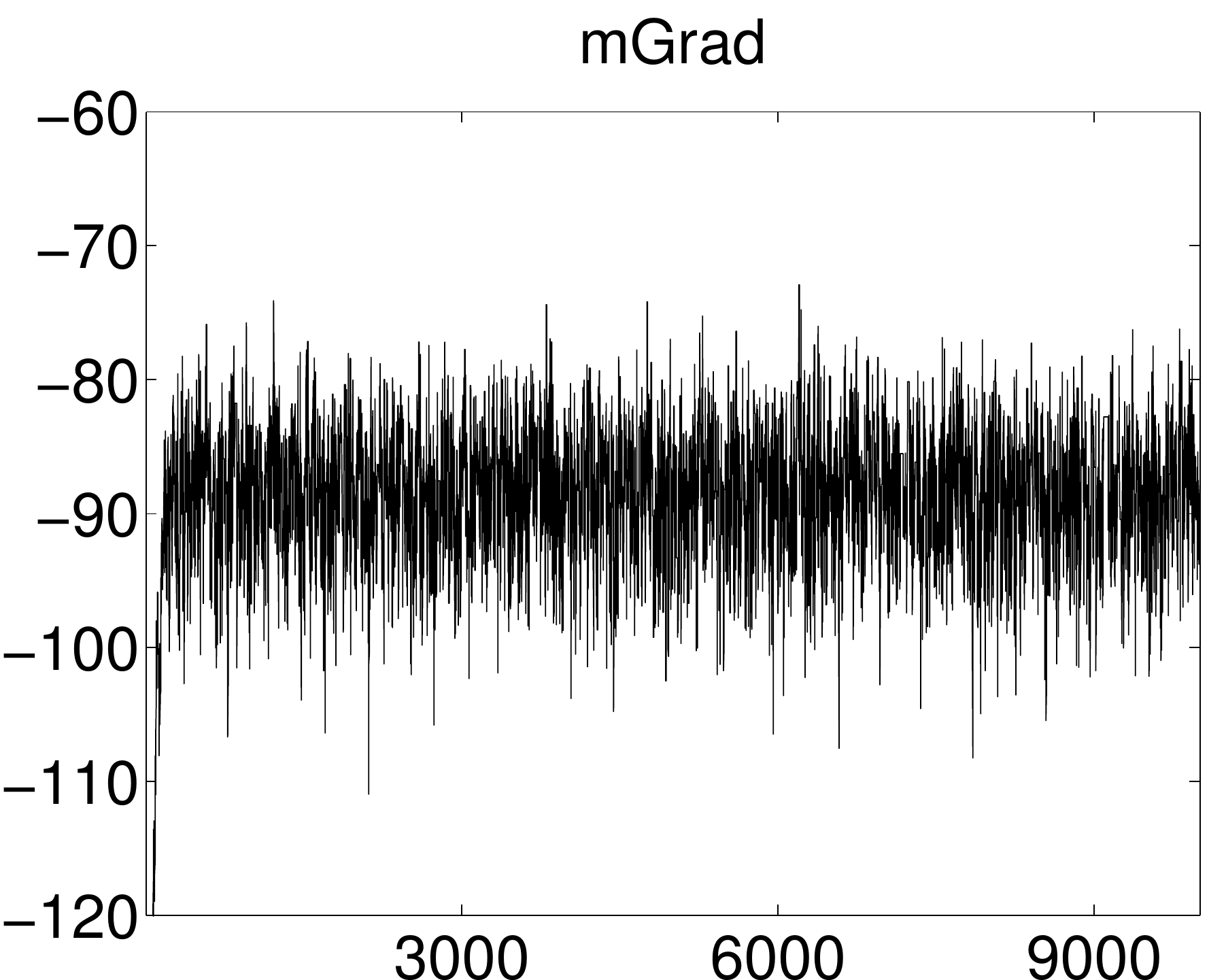}} \\
{\includegraphics[scale=0.2]
{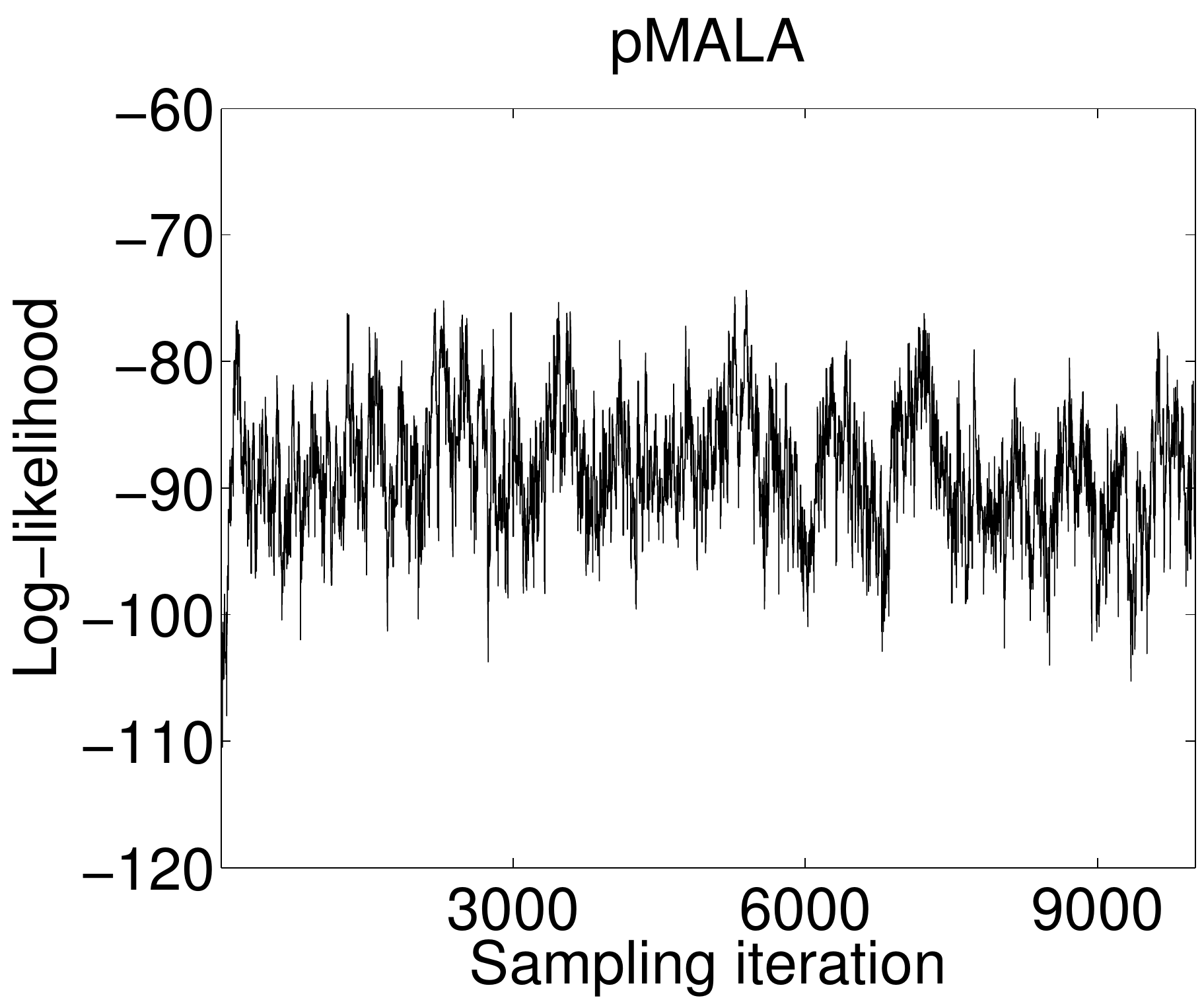}} &
{\includegraphics[scale=0.2]
{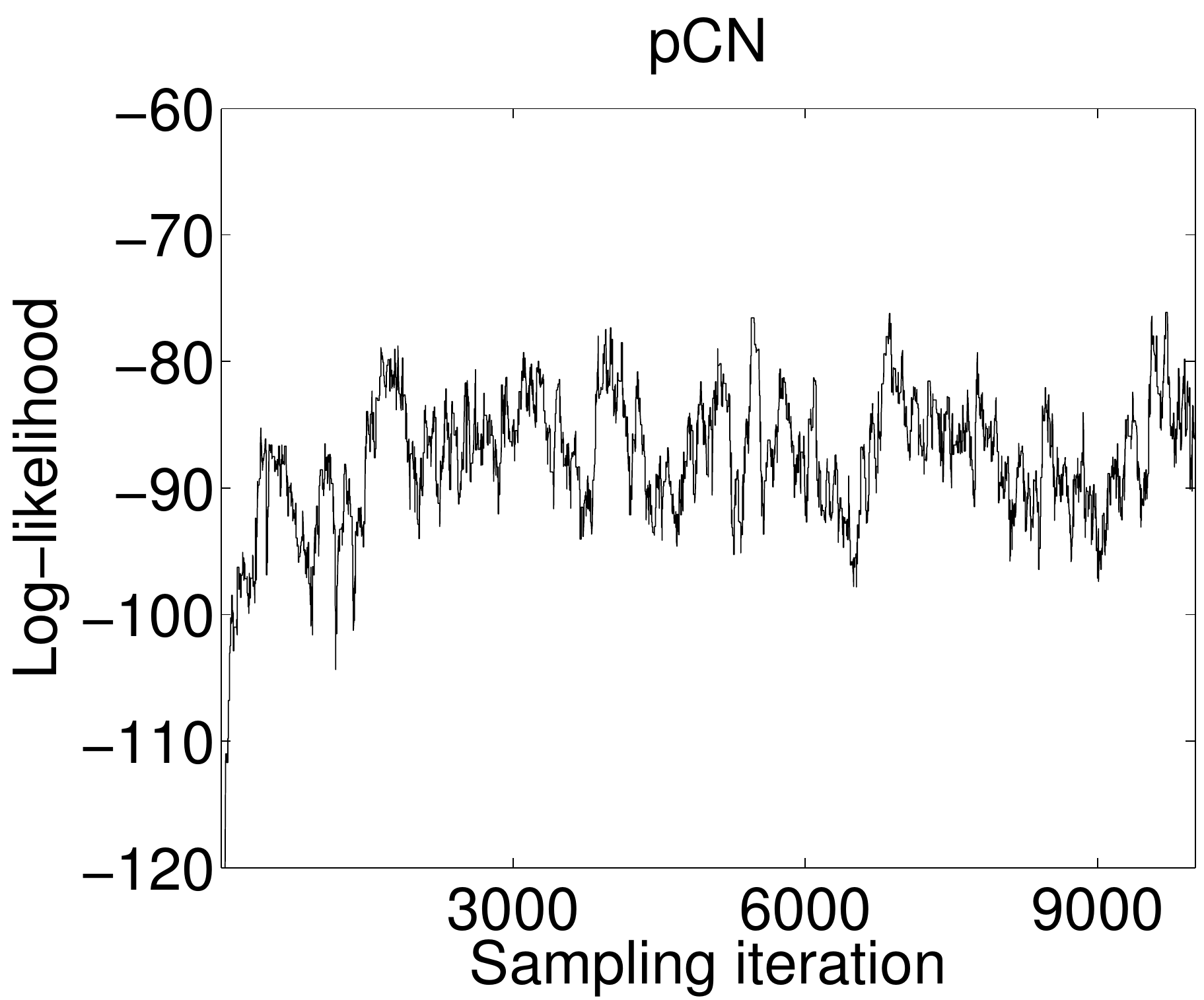}} &
{\includegraphics[scale=0.2]
{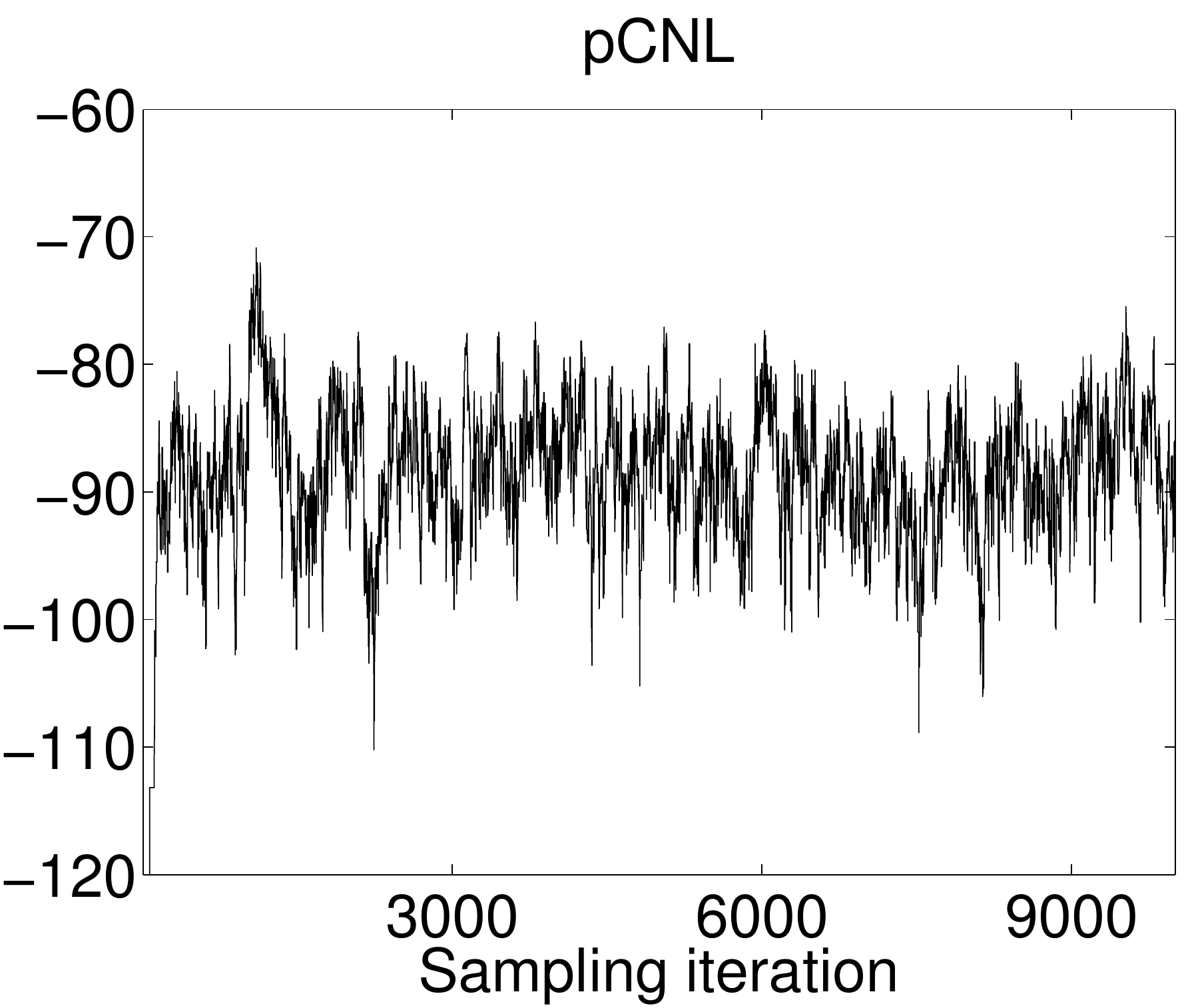}} \\
\end{tabular}
\caption{The evolution of the log-likelihood values across all iterations 
for the sampling schemes aGrad-z, aGrad-u, mGrad-z, pMALA, pCN and pCNL in the Heart dataset. 
The corresponding plot for Ellipt looks very similar to pCN and it is omitted.   
\label{fig:HeartLogL}}
\end{figure}

\subsection*{Additional experiments for hyperparameter learning}

In the first experiment we consider binary Gaussian process classification and we 
apply the alternative schemes to Pima Indian dataset. We run the algorithms for
$4 \times 10^4$ burn-in iteration and then we collect $5 \times 10^3$ samples.  
Figure \ref{fig:pimahypers} shows the log-likelihoods across all $4.5 \times 10^4$
sampling iterations. We can observe that aGrad-z-joint is the best mixing with aGrad-z-gibbs coming second. The corresponding plots for 
Ellipt-gibbs and pCNL-gibbs exhibit less variation, and stabilize in slightly smaller log-likelihood values, 
with pCNL being the most problematic. Figure \ref{fig:pimahypersBox} shows boxplots
for the inferred hyperparameters for the different methods. We observe that while  
aGrad-z-gibbs and aGrad-z-joint largely agree about the inferred values, 
Ellipt-gibbs and pCNL-gibbs have very narrow boxplots. Also pCNL-gibbs clearly has stuck to  
a posterior mode which is rather very different from the remaining algorithms. Finally, Table 
\ref{table:pimahyperstheta} shows performance scores for sampling the two hyperparameters 
which show the clear superiority of aGrad-z-joint. Notice that the learned step size $\kappa$ 
is much larger for aGrad-z-joint than for the remaining schemes.  
Notice, however, that the ESS scores for Ellipt-gibbs and pCNL-gibbs should not be trusted since 
Figure \ref{fig:pimahypers} and Figure \ref{fig:pimahypersBox} suggest that these methods do not mix  well.

\begin{figure}
\centering
\begin{tabular}{cc}
\includegraphics[scale=0.3]
{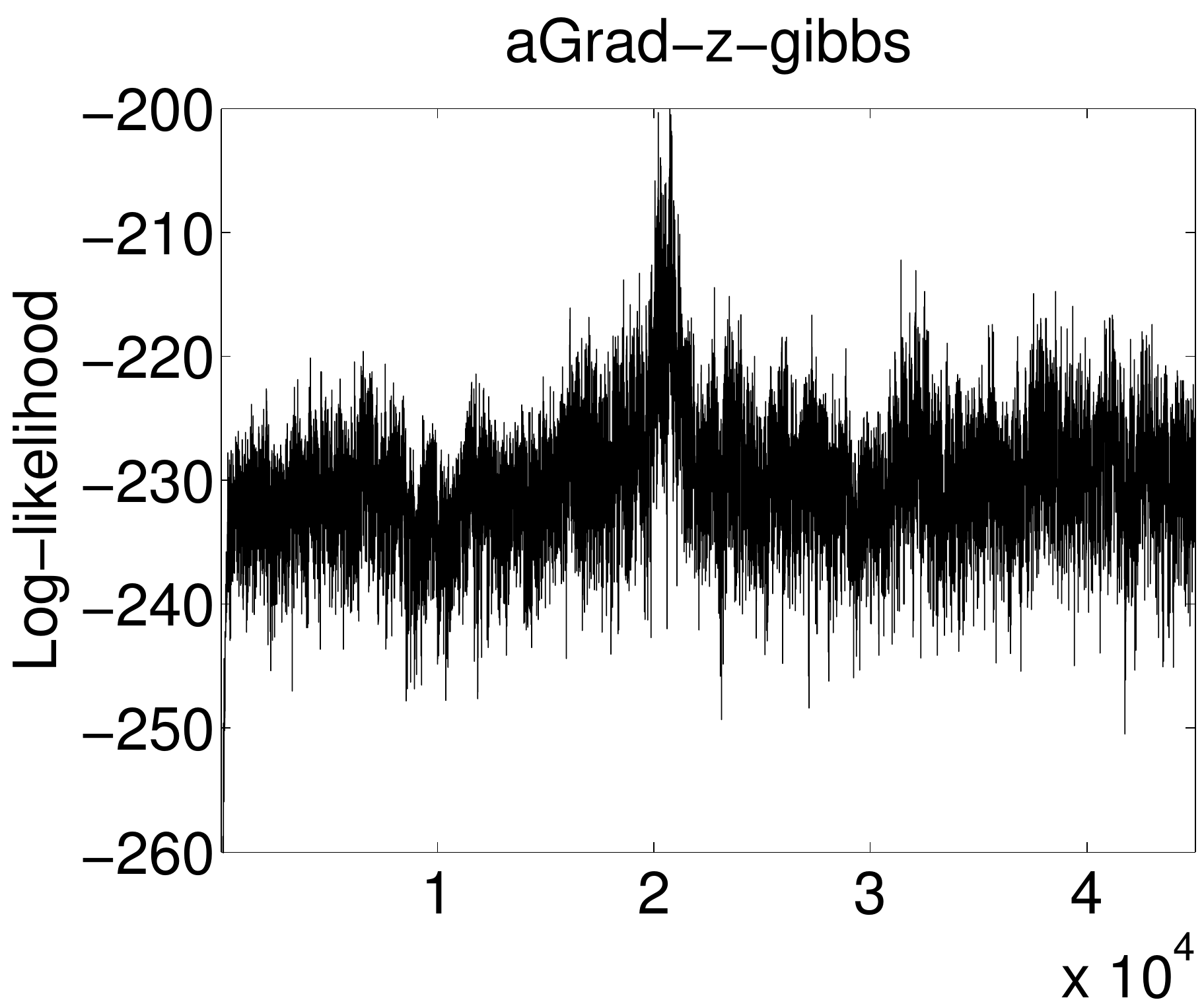} &
\includegraphics[scale=0.3]
{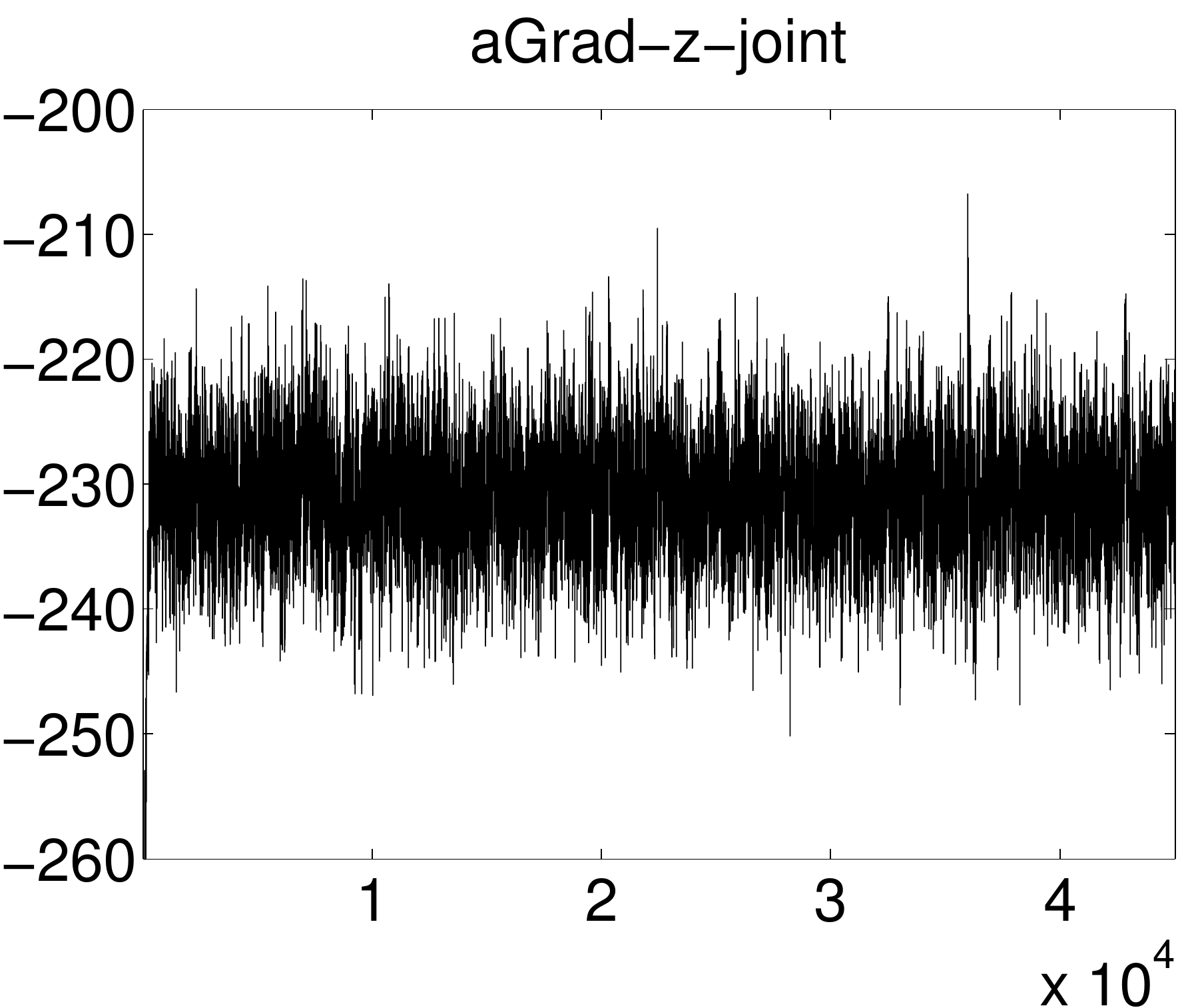} \\
\includegraphics[scale=0.3]
{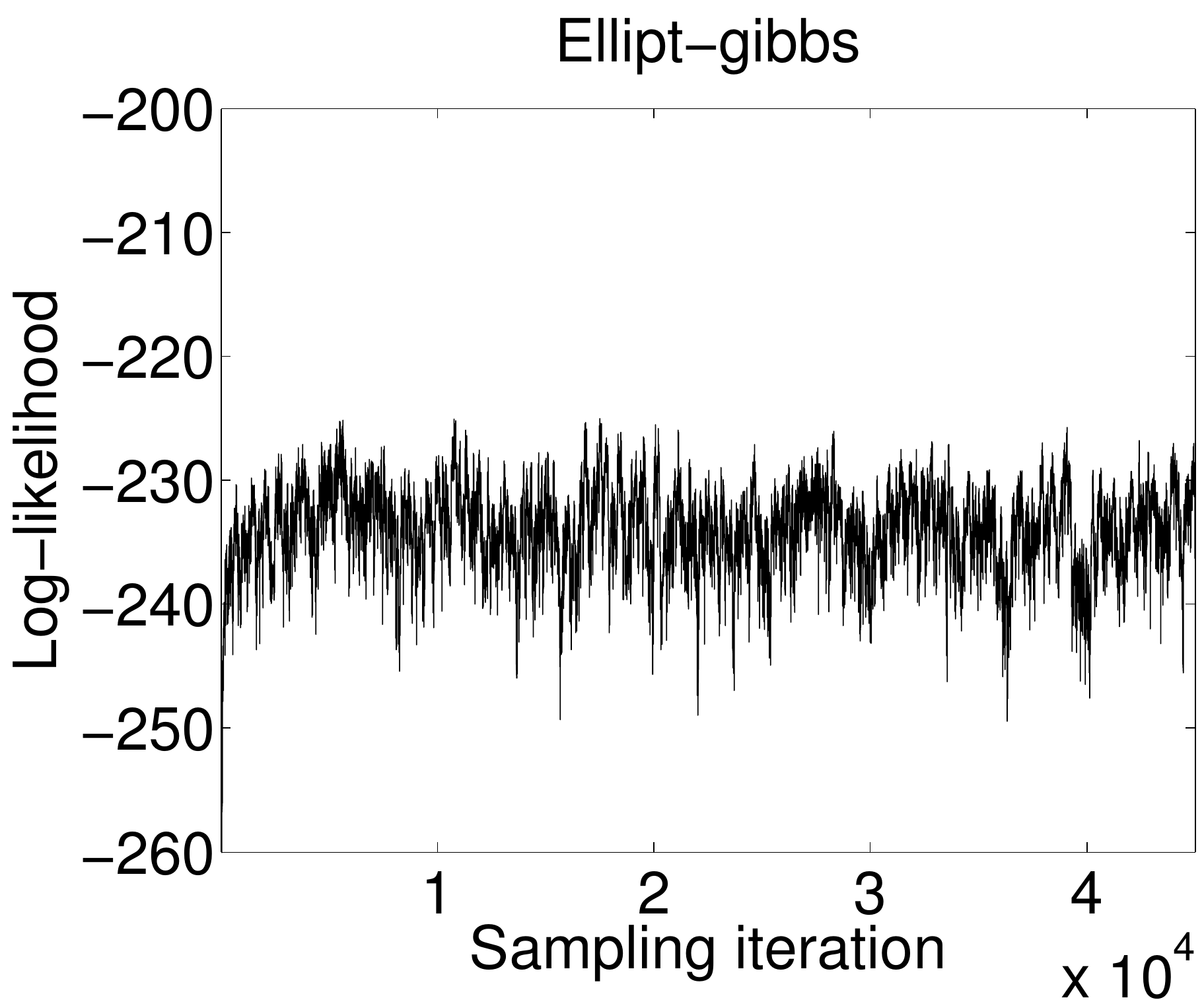} &
\includegraphics[scale=0.3]
{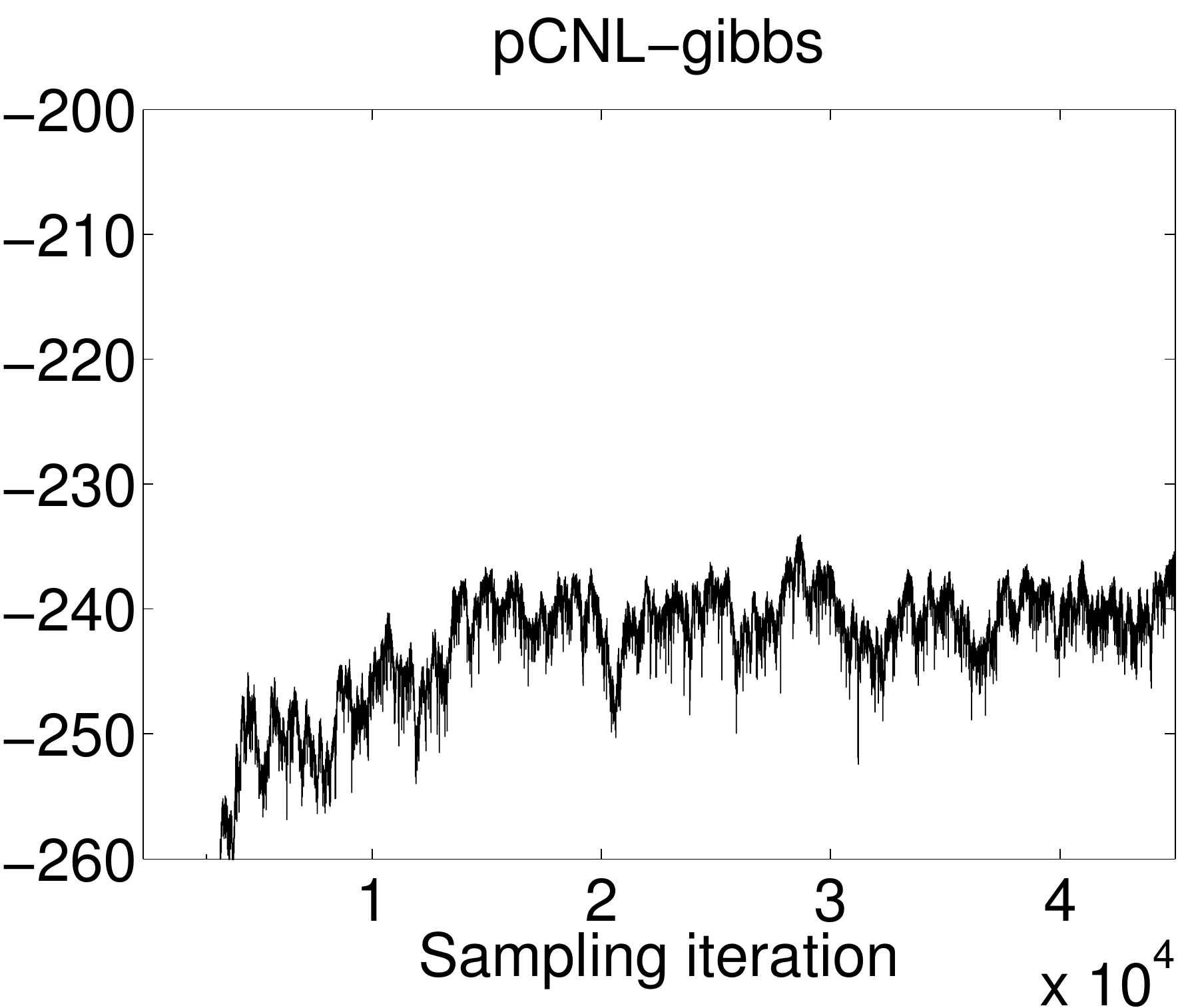} 
\end{tabular}
\caption{\label{fig:pimahypers}The evolution of the log-likelihood values across all iterations in Pima Indian.}
\end{figure}

\begin{figure}
\centering
\begin{tabular}{cc}
\includegraphics[scale=0.3]
{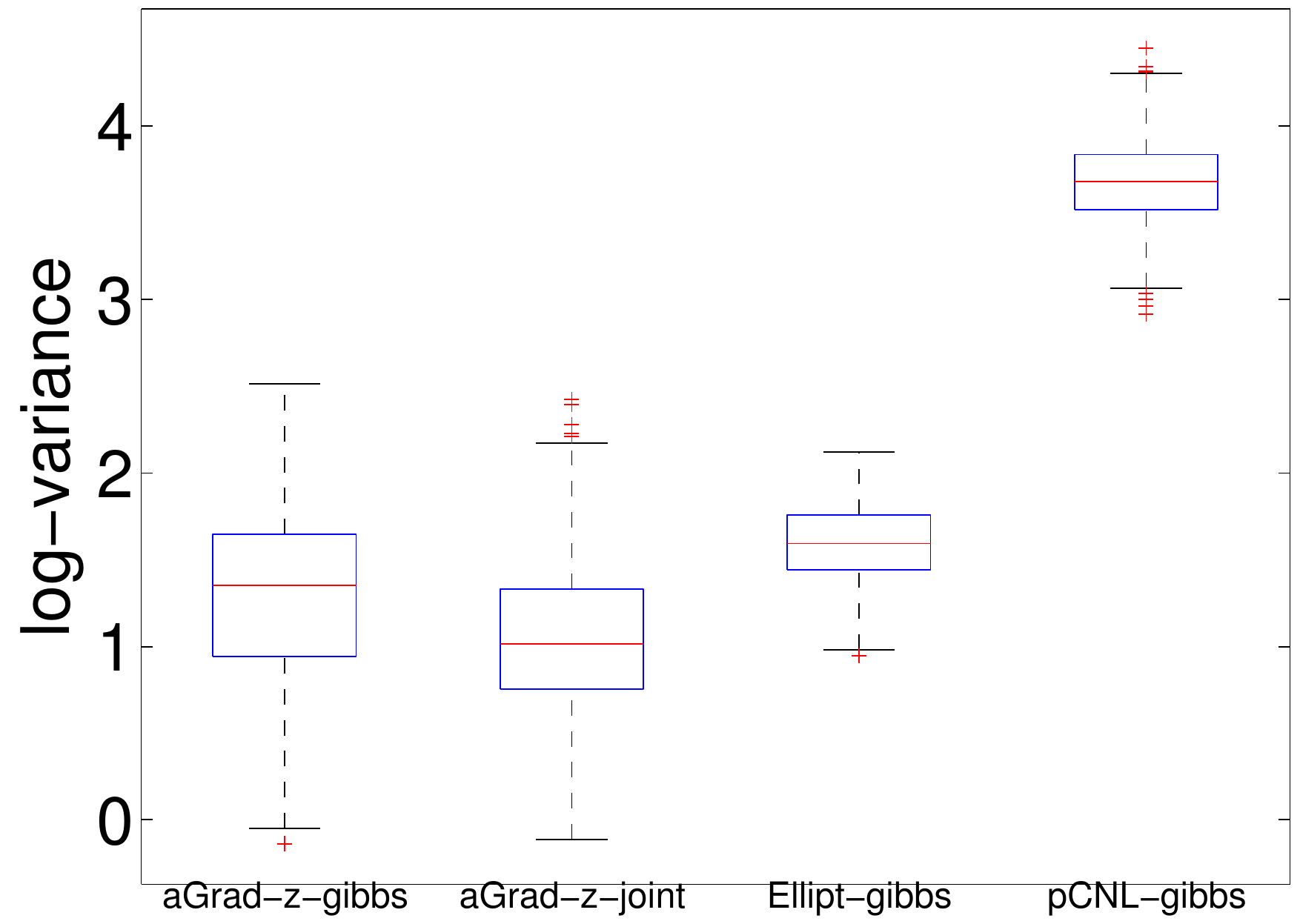} &
\includegraphics[scale=0.3]
{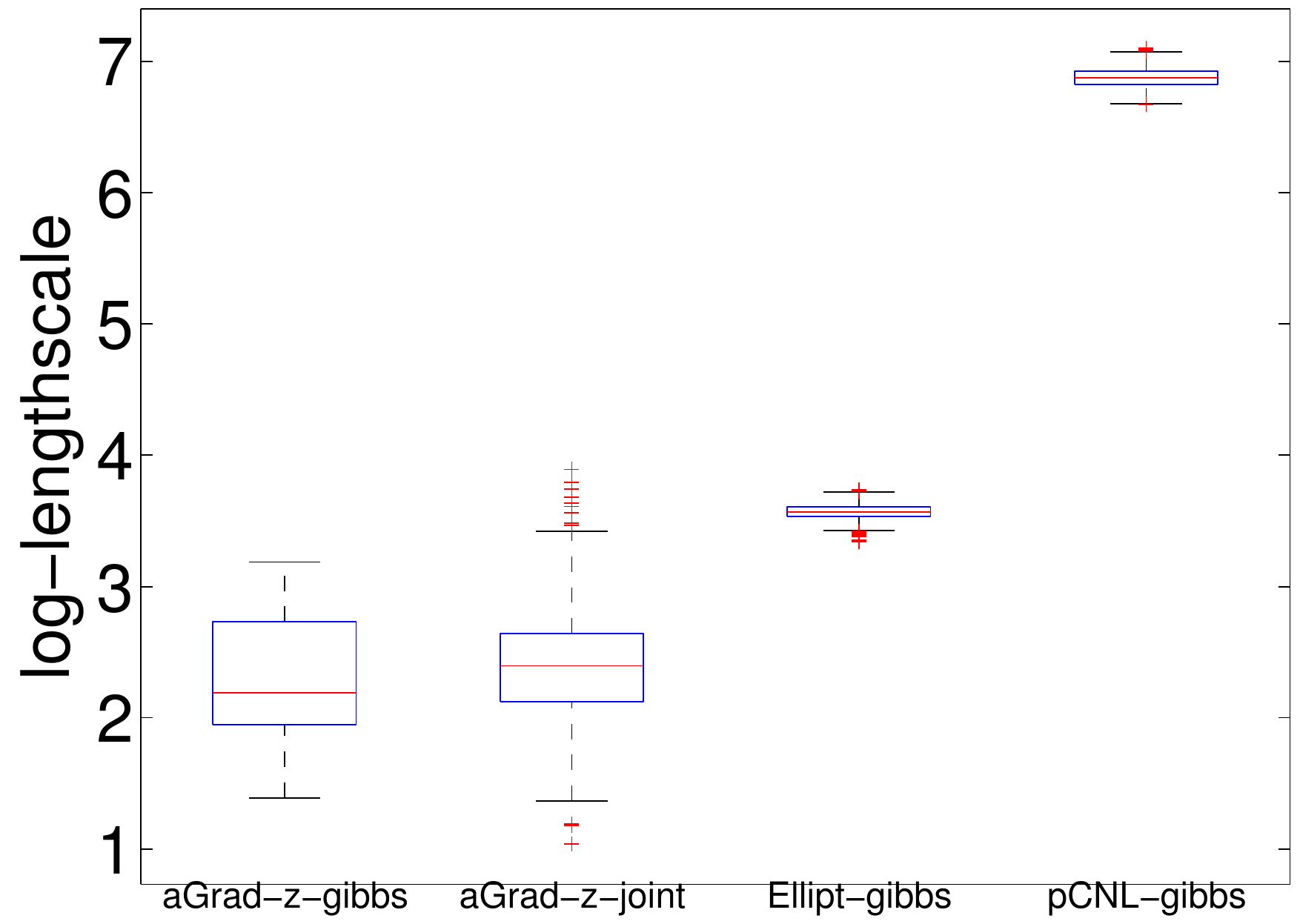}  
\end{tabular}
\caption{\label{fig:pimahypersBox}The inferred hyperparameters values in Pima Indian dataset.}
\end{figure}


\begin{table}
\caption{\label{table:pimahyperstheta}Performance scores of the sampling methods in Pima dataset 
for the kernel hyperparameters.}
\centering
\fbox{%
\begin{tabular}{*{5}{c}}
\em Method &\em Time(s) &\em Step $\kappa$  &\em ESS ($\sigma_x^2$, $\ell^2$)  &\em Min ESS/s \\ 
\hline
aGrad-z-gibbs  &   599.4  &  0.005  &  (5.1, 11.0)  &  0.01 \\ 
aGrad-z-joint  &   692.6  &  0.179  &  (173.4, 232.5)  &  0.25 \\ 
Ellipt-gibbs  &  426.8   &  0.006  &  (3.2, 5.7)  &  0.01 \\ 
pCNL-gibbs  &  350.7  &  0.020  &  (4.4, 15.3)  &  0.01 \\ 
\end{tabular}}
\end{table}

\section*{Efficient implementation of the algorithms}

\subsection*{Pilot tuning of step size}
All algorithms discussed in the main article require the choice of a single step size
parameter $\delta$ in order to achieve an acceptance rate that leads
to good mixing. In our experiments all algorithms that use gradient
information are tuned to achieve an acceptance rate around 50\%;
asymptotic theory, e.g. in \cite{scaleSS},  supports this tuning for pMALA, this is the rate
that pCNL is tuned to in experiments in \cite{Cotter2013MCMC} and rate
that our experiments suggests to be advantageous for the new auxiliary
and marginal samplers. On the other hand, pCN is tuned to achieve a
rate around 30\% following the recommendation in
\cite{Cotter2013MCMC}.  A simple observation is that no matrix
factorisations are needed during the tuning as we adapt $\delta$,
since all matrices involved in the algorithms, either when $\bS=\bI$
or $\bS=\bC$, share the same eigenspace with $\bC$ regardless of the value of
$\delta$, hence an offline spectral decomposition of the latter
(e.g. using a QR-algorithm) can be
easily updated at $\Op(n)$ cost to produce all matrices necessary to
carry out the algorithms. Additional details regarding the implementation 
of the algorithms is given in the next section.

\subsection*{Efficient generation of proposal, computation of
  acceptance ratio and pseudocode for the proposed algorithms}

Here, we discuss how we can efficiently implement all algorithms based on 
a eigenvalue decomposition of the prior covariance matrix $\bC$. 
We provide full details for the proposed marginal sampler, while for all remaining 
schemes the implementation follows similar steps. 

The marginal scheme has proposal 
\begin{align} 
q(\by|\bx) = \Gau \left(\by |  { 2 \over \delta}\bA  \left ( \bx   +  {\delta \over 2} \nabla
f(\bx) \right),   {2 \over \delta }\bA^2  + \bA \right)\,
\end{align}
where $\bA^{-1} =  \bC^{-1} + {2 \over \delta} \bI$ and 
the corresponding Metropolis-Hastings ratio is
\begin{align}
& \exp\{f(\by) - f(\bx) + h(\bx,\by) - h(\by,\bx)\}, \label{eq:margina-ar} \\    
& h(\bx,\by)=  \left ( \bx - \frac{2}{\delta}\bA  ( \by +  {\delta \over 4} 
\nabla f(\by) ) \right)^T  \left( {2 \over \delta} \bA + \bI \right)^{-1} \nabla f(\by)\, \nonumber 
\end{align}
To provide an efficient implementation we will be based on a precomputed 
decomposition of $\bC = \bU \bLambda \bU^T$. Based on this 
$\bA$ can be written as 
$$
\bA = \bU \bLambda_1 \bU^T 
$$
where each eigenvalue in the diagonal matrix $\bLambda_1$ has the form  
$\frac{\gamma \delta}{\delta + 2 \gamma}$ where $\gamma$ is an eigenvalue of $\bC$. Similarly 
$$
 {2 \over \delta }\bA^2  + \bA = \bU \bLambda_2 \bU^T
$$
where each eigenvalue of $\bLambda_2$ has the form  
$\frac{\gamma \delta}{\delta + 2 \gamma} \times \frac{\delta + 4 \gamma}{\delta + 2 \gamma}$. 
Also it would be useful to know the decomposition of $\left( {2 \over \delta} \bA + \bI \right)^{-1}$
which is 
$$
\left( {2 \over \delta} \bA + \bI \right)^{-1} = \bU \bLambda_3 \bU^T
$$
where each eigenvalue  of $\bLambda_3$ has the form  
$\frac{\delta + 2 \gamma}{\delta + 4 \gamma}$. 
A sample from $q(\by|\bx)$ can be implemented as follows 
$$
\by = \bU \left[   \bLambda_1 \left( {2 \over \delta}  \bU^T \bx   +   \bU^T \nabla f(\bx)  \right)   + \bLambda_2^{\frac{1}{2}} \bfeta \right], \ \ \text{where} \ \ \bfeta \sim \Gau(\bzero,\bI). 
$$
To compute this term we consider that the vectors $\bU^T \bx$,  $\bU^T \nabla f(\bx)$ and the whole 
$\bLambda_1 \left( {2 \over \delta}  \bU^T \bx   +   \bU^T \nabla f(\bx)  \right)$ have already been pre-computed 
from a previous MCMC iteration. Therefore, the above sampling operation involves the computation of  $\bLambda_1^{\frac{1}{2}} \bfeta$, which costs $\Op(n)$ since $\bLambda_1^{\frac{1}{2}}$ is diagonal, an addition of two vectors
 and a single matrix-vector multiplication involving the most-left $\bU$ and the remaining term (which is just an already 
 computed vector). Thus, overall the proposal of $\by$ can be implemented with a single 
 matrix-vector multiplication,  which costs $\Op(n^2)$.
 
The computation of the Metropolis-Hastings ratio can be carried out as follows. 
To compute $h(\bx,\by)$ we are based on the expression 
\begin{align}
h(\bx,\by) & =  \left ( \bx - \bU \bLambda_1 ( \frac{2}{\delta} \bU^T  \by +  {1 \over 2} 
\bU^T \nabla f(\by) ) \right)^T  \bU \bLambda_3 \bU^T \nabla f(\by) \nonumber \\
 & = \left[ \bU^T \bx  - \left(  \bLambda_1 (\frac{2}{\delta} \bU^T  \by +  {1 \over 2} 
\bU^T \nabla f(\by) ) \right) \right]^T \bLambda_3 \bU^T \nabla f(\by) \nonumber
\end{align}
Then, we apply two matrix-vector multiplications, each having cost $\Op(n^2)$,  in order to 
compute the vectors $\bU^T \by$ and  $\bU^T \nabla f(\by)$. With these two vectors pre-computed and stored 
all remaining operations cost $\Op(n)$. For example, the computation of $\bLambda_3 \bU^T \nabla f(\by)$ is 
$\Op(n)$ since $\bLambda_3$ is diagonal. Also the computation of the symmetric term $h(\by,\bx)$ required in  the 
ratio is just $\Op(n)$ since the required vectors $\bU^T \bx$ and $\bU^T \nabla f(\bx)$ are known form the 
previous iteration. Notice that if $\by$ is accepted, then the two vectors $\bU^T \by$ and $\bU^T \nabla f(\by)$ 
will become the already pre-computed vectors required in the next iteration. Also when $\by$ 
is accepted we also compute $\bLambda_1 (\frac{2}{\delta} \bU^T  \by + \bU^T \nabla f(\by) )$
needed for the upcoming proposal step in the next iteration. 
So overall, the marginal scheme  requires three matrix-vector multiplications per iteration. 

Working similarly as above, the auxiliary scheme based on $\bu$ can be implemented with three matrix-vector
multiplications per iteration, the auxiliary scheme based on $\bz$ as well as pCNL require two such operations, while 
the pCN scheme requires just a single matrix-vector multiplication per iteration.  
\rev{ Full pseudocode for all three proposed algorithms aGrad-z, aGrad-u and mGrad  is given by  Algorithm  \ref{algorz}, 
Algorithm  \ref{algoru}  and Algorithm  \ref{algorm} respectively. The appearance of $*$s  indicate the number and the  
precise lines inside the code we need to  perform a matrix-vector multiplication.    
   
  \begin{algorithm}[tb]
   \caption{ aGrad-z \label{algorz}}
   \label{alg1}
\begin{algorithmic}
   \STATE {\bfseries Input:} $\bU$ and $\bG$ such that $\bU \bG  \bU^T  =   \bC$,  $f(\bx)$, $\nabla f(\bx)$,  number of burn-in and collection iterations  $\text{BURN}$ and $\text{T}$.    
   \STATE Initialize: $\bx = {\bf 0}$,  $f\bx = f(\bx)$, $\text{grad}f\bx = \nabla f(\bx)$,  the diagonal $\bLambda_1$ 
                such that $[\bLambda_1]_{ii}  = \frac{\gamma_i \delta}{\delta + 2 \gamma_i}$.   
    \FOR{$i=1$ {\bfseries to} $\text{BURN} + \text{T}$}
   \STATE Draw $\bz = \bx  + (\delta/2) \text{grad}f\bx  + \sqrt{\delta/2} \bfeta,   \  \  \ \bfeta \sim \mathcal{N}({\bf 0}, \bI)$.
    \STATE Propose  $\by =   \bU \left(  \bLambda_1^{\frac{1}{2} }  \left(   \bLambda_1^{\frac{1}{2} }   (  \bU^T ( \frac{2}{\delta} \bz ) )       +  \bfeta \right)   \right), \ \ \ \bfeta \sim \mathcal{N}({\bf 0}, \bI)$.  \ \ \ \ \ \ \ \ \ **
    \STATE M-H step:  $f\by = f(\by)$, $\text{grad}f\by  = \nabla f(\by)$, \\
              \ \ \ \ \  \  $g\bz\by =  ( \bz  - \by  - \frac{\delta}{4} \text{grad}f\by  )^T \text{grad}f\by$, \\
              \ \ \ \ \ \   $g\bz\bx =   ( \bz  - \bx  - \frac{\delta}{4} \text{grad}f\bx )^T \text{grad}f\bx$. \\
              \ \ \ \ \ \   {\bf if}  $rand < min(1,  \exp(f\by - f\bx + g\bz\by  - g\bz\bx )  )$  {\bf then } \\  
              \ \ \ \ \ \ \ \ \  $\bx = \by$,  $f\bx = f\by$,  $\text{grad}f\bx = \text{grad}f\by$.  \\
              \ \ \ \ \ \   {\bf end if} 
      \STATE Adapt $\delta$ or collect sample:  \\
              \ \ \ \ \ \  {\bf if } $i \leq $ BURN  \\ 
              \ \ \ \ \ \  \ \ \ \ adapt $\delta$  to achieve acceptance rate in $50\%$ to $60 \%$.\\ 
              \ \ \ \ \ \ \ \ \ \  update $\bLambda_1$:  $[\bLambda_1]_{ii}  = \frac{\gamma_i \delta}{\delta + 2 \gamma_i}$.  \\
               \ \ \ \ \ \  {\bf else} \\
              \ \ \ \ \ \ \ \ \ \  collect sample $\bx$.  \\
              \ \ \ \ \ \  {\bf end if} 
      \ENDFOR          
\end{algorithmic}
\end{algorithm}

}

\rev{ 

  \begin{algorithm}[tb]
   \caption{ aGrad-u \label{algoru}}
   \label{alg1}
\begin{algorithmic}
   \STATE {\bfseries Input:} $\bU$ and $\bG$ such that $ \bU \bG  \bU^T  =   \bC$, $f(\bx)$,  $\nabla f(\bx)$,  number of burn-in and collection iterations  $\text{BURN}$ and $\text{T}$.    
   \STATE Initialize: $\bx = {\bf 0}$,  $f\bx = f(\bx)$, $\text{grad}f\bx = \nabla f(\bx)$,  $\text{Ugrad}f\bx = \bU^T \text{grad}f\bx$,  the diagonal $\bLambda_1$ 
                such that $[\bLambda_1]_{ii}  = \frac{\gamma_i \delta}{\delta + 2 \gamma_i}$.    
    \FOR{$i=1$ {\bfseries to} $\text{BURN} + \text{T}$}
   \STATE Draw $\bu = \bx  + \sqrt{\delta/2} \bfeta,   \  \  \ \bfeta \sim \mathcal{N}({\bf 0}, \bI)$.
   \STATE   $\text{Udelta}\bu  = \bU^T ( \frac{2}{\delta} \bu)$.     \ \ \ \ \ \ \ \ \ \ \ \ \ \ \ \ \ \ \ \ \ \ \ \ \ \ \ \ \ \ \ \ \ \ \ \ \ \  \ \ \ \ \ \ \ \ \ \ \ \ \ \ \ \ \ \ \ \ \ \ *
    \STATE Propose  $\by =   \bU  \left(  \bLambda_1 \left(  \text{Udelta}\bu   +  \text{Ugrad}f\bx   \right)   +  \bLambda_1^{\frac{1}{2}} \bfeta \right), \  \bfeta \sim \mathcal{N}({\bf 0}, \bI)$.  \ *
    \STATE M-H step:  $f\by = f(\by)$, $\text{grad}f\by  = \nabla f(\by)$,    $\text{Ugrad}f\by = \bU^T \text{grad}f\by$,  \ \   *  \\  
              \ \ \ \ \  \  $j\bx\by\bu  = \bx^T \text{grad}f\by  -  \left(  \bLambda_1 (  \text{Udelta}\bu   + \frac{1}{2}  \text{Ugrad}f\by  )     \right)^T \text{Ugrad}f\by$,  \\
                \ \ \ \ \  \  $j\by\bx\bu  = \by^T \text{grad}f\bx  -  \left(  \bLambda_1 (  \text{Udelta}\bu   + \frac{1}{2}  \text{Ugrad}f\bx  )     \right)^T \text{Ugrad}f\bx$.  \\
              \ \ \ \ \ \   {\bf if}  $rand < min(1,  \exp(f\by - f\bx + j\bx\by\bu  - j\by\bx\bu )  )$   {\bf then } \\  
              \ \ \ \ \ \ \ \ \  $\bx = \by$,  $f\bx = f\by$,  $\text{grad}f\bx = \text{grad}\by$,  $\text{Ugrad}f\bx  = \text{Ugrad}f\by$. \\
              \ \ \ \ \ \   {\bf end if} 
      \STATE Adapt $\delta$ or collect sample:  \\
              \ \ \ \ \ \  {\bf if } $i \leq $ BURN  \\ 
              \ \ \ \ \ \  \ \ \ \ adapt $\delta$  to achieve acceptance rate in $50\%$ to $60 \%$.\\ 
              \ \ \ \ \ \ \ \ \ \  update $\bLambda_1$:  $[\bLambda_1]_{ii}  = \frac{\gamma_i \delta}{\delta + 2 \gamma_i}$.  \\
               \ \ \ \ \ \  {\bf else} \\
              \ \ \ \ \ \ \ \ \ \  collect sample $\bx$.  \\
              \ \ \ \ \ \  {\bf end if} 
      \ENDFOR          
\end{algorithmic}
\end{algorithm}

}

\rev{

  \begin{algorithm}[tb]
   \caption{ mGrad \label{algorm}}
   \label{alg1}
\begin{algorithmic}
   \STATE {\bfseries Input:} $\bU$ and $\bG$ such that $\bU \bG  \bU^T  =   \bC$,  $f(\bx)$, 
               $\nabla f(\bx)$,  number of burn-in and collection iterations  $\text{BURN}$ and $\text{T}$.    
   \STATE Initialize: $\bx = {\bf 0}$,  $f\bx = f(\bx)$, $\text{grad}f\bx = \nabla f(\bx)$,   $\text{U}\bx = \bU^T \bx$, 
                $\text{Ugrad}f\bx = \bU^T \text{grad}f\bx $, 
               $\bLambda_1$ 
                such that $[\bLambda_1]_{ii}  = \frac{\gamma_i \delta}{\delta + 2 \gamma_i}$, $\bLambda_2$ 
                 such that   $[\bLambda_2]_{ii}  = \frac{\gamma_i \delta}{\delta + 2 \gamma_i} \times 
                 \frac{\delta + 4 \gamma_i}{ \delta + 2 \gamma_i}$, 
                  $\bLambda_3$ 
                 such that   $[\bLambda_3]_{ii}  = \frac{\delta + 2 \gamma_i}{ \delta + 4 \gamma_i}$,
                  $\text{tmpSample}\bx =  \bLambda_1 \left(  \frac{2}{\delta}  \text{U}\bx   +   \text{Ugrad}f\bx \right)$,
                 $\text{tmpMH}\bx =  \bLambda_1 \left(  \frac{2}{\delta}  \text{U}\bx   +   \frac{1}{2} \text{Ugrad}f\bx \right)$.
    \FOR{$i=1$ {\bfseries to} $\text{BURN} + \text{T}$}
    \STATE Propose  $\by =   \bU  \left(  \text{tmpSample}\bx  +  \bLambda_2^{\frac{1}{2} }  \bfeta \right), \ \ \ \bfeta \sim \mathcal{N}({\bf 0}, \bI)$.  \ \ \ \ \ \ \ \ \ \ \ \ \ \ \ \  *
    \STATE M-H step:  $f\by = f(\by)$, $\text{grad}f\by  = \nabla f(\by)$, \\
                 \ \ \ \ \ \ $\text{U}\by = \bU^T \by$,  $\text{Ugrad}f\by = \bU^T \text{grad}f\by$,  \ \ \ \ \ \ \ \ \ \ \ \ \ \ \ \ \ \ \  \ \ \ \ \ \ \ \ \ \ \ \ \ ** \\
                 \ \ \ \ \ \ $\text{tmpSample}\by =  \bLambda_1 \left(  \frac{2}{\delta}  \text{U}\by   +   \text{Ugrad}f\by \right)$, \\
                 \ \ \ \ \ \ $\text{tmpMH}\by =  \bLambda_1 \left(  \frac{2}{\delta}  \text{U}\by   +   \frac{1}{2} \text{Ugrad}f\by \right)$, \\
                 \ \ \ \ \ \ $h\bx\by =  ( \text{U}\bx  -  \text{tmpMH}\by )^T   ( \bLambda_3  \text{Ugrad}f\by$), \\
                 \ \ \ \ \ \  $h\by\bx =   ( \text{U}\by  -  \text{tmpMH}\bx  )^T  ( \bLambda_3  \text{Ugrad}f\bx$),  \\
                 \ \ \ \ \ \   {\bf if}  $rand < min(1,  \exp(f\by - f\bx + h\bx\by  - h\by\bx )  )$  {\bf then } \\  
                \ \ \ \ \ \ \ \ \ \ $\bx = \by$,  $f\bx = f\by$,  $\text{grad}f\bx = \text{grad}f\by$,   \\ 
                \ \ \ \ \ \ \ \ \ \ $\text{U}\bx = \text{U}\by$, $\text{Ugrad}f\bx = \text{Ugrad}f\by$,  \\
                \ \ \ \ \ \ \ \ \  \ $\text{tmpSample}\bx = \text{tmpSample}\by$,    $\text{tmpMH}\bx = \text{tmpMH}\by$. \\
                \ \ \ \ \ \   {\bf end if} 
      \STATE Adapt $\delta$ or collect sample:  \\
              \ \ \ \ \ \  {\bf if } $i \leq $ BURN  \\ 
              \ \ \ \ \ \  \ \ \ \ adapt $\delta$  to achieve acceptance rate in $50\%$ to $60 \%$.\\ 
              \ \ \ \ \ \ \ \ \ \  update $\bLambda_1$:  $[\bLambda_1]_{ii}  = \frac{\gamma_i \delta}{\delta + 2 \gamma_i}$.  \\
              \ \ \ \ \ \ \ \ \ \  update $\bLambda_2$:  $[\bLambda_2]_{ii}  = \frac{\gamma_i \delta}{\delta + 2 \gamma_i} \times  \frac{\delta + 4 \gamma_i}{ \delta + 2 \gamma_i}$, \\
              \ \ \ \ \ \ \ \ \ \  update $\bLambda_3$:  $[\bLambda_3]_{ii}  = \frac{\delta + 2 \gamma_i}{ \delta + 4 \gamma_i}$, \\
              \ \ \ \ \ \ \ \ \ \  update  $\text{tmpSample}\bx =  \bLambda_1 \left(  \frac{2}{\delta}  \text{U}\bx   +   \text{Ugrad}f\bx \right)$,  \\
              \ \ \ \ \ \ \ \ \ \  update $\text{tmpMH}\bx =  \bLambda_1 \left(  \frac{2}{\delta}  \text{U}\bx   +   \frac{1}{2} \text{Ugrad}f\bx \right)$. \\
               \ \ \ \ \ \  {\bf else} \\
              \ \ \ \ \ \ \ \ \ \ collect sample $\bx$.  \\
              \ \ \ \ \ \  {\bf end if} 
      \ENDFOR          
\end{algorithmic}
\end{algorithm}

}

\end{document}